\documentclass[twoside,11pt]{article}
\usepackage[preprint]{jmlr2e}
\usepackage{microtype}
\hypersetup{hidelinks,breaklinks}
\usepackage{amsmath}
\usepackage{bm}
\usepackage{mathtools,mathabx}
\usepackage[shortlabels]{enumitem}
\usepackage[dvipsnames]{xcolor}
\usepackage{tikz}
\usetikzlibrary{decorations.pathmorphing}
\usepackage[breakable]{tcolorbox}
\usetikzlibrary{cd}
\usepackage{wrapfig}
\usepackage{framed}

\newtheorem{examples}[theorem]{Example}
\newtheorem{notation}[theorem]{Notation}
\newtheorem{assumption}[theorem]{Assumption}
\newtheorem{question}[theorem]{Question}

\def\semicolon{;}
\def\applytolist#1{
    \expandafter\def\csname multi#1\endcsname##1{
        \def\multiack{##1}\ifx\multiack\semicolon
            \def\next{\relax}
        \else
            \csname #1\endcsname{##1}
            \def\next{\csname multi#1\endcsname}
        \fi
        \next}
    \csname multi#1\endcsname}
\def\calc#1{\expandafter\def\csname c#1\endcsname{{\mathcal #1}}}
\applytolist{calc}QWERTYUIOPLKJHGFDSAZXCVBNM;
\def\bbc#1{\expandafter\def\csname bb#1\endcsname{{\mathbb #1}}}
\applytolist{bbc}QWERTYUIOPLKJHGFDSAZXCVBNM;
\def\bfc#1{\expandafter\def\csname bf#1\endcsname{{\mathbf #1}}}
\applytolist{bfc}QWERTYUIOPLKJHGFDSAZXCVBNM;
\def\sfc#1{\expandafter\def\csname s#1\endcsname{{\sf #1}}}
\applytolist{sfc}QWERTYUIOPLKJHGFDSAZXCVBNM;
\def\fc#1{\expandafter\def\csname f#1\endcsname{{\mathfrak #1}}}
\applytolist{fc}QWERTYUIOPLKJHGFDSAZXCVBNM;

\newcommand{\op}[1]{\operatorname{#1}}
\DeclareMathOperator{\prob}{Prob}
\DeclareMathOperator{\id}{Id}
\DeclareMathOperator{\diam}{diam}

\newcommand{\dexp}[1][]{d_{\mathrm{R}#1}}
\newcommand{\dexpcon}[1][]{d_{\operatorname{RC}#1}}

\DeclareMathOperator{\reeb}{Reeb}
\DeclareMathOperator{\dis}{dis}

\newcommand{\set}[2]{\left\{ #1\middle| #2\right\}}
\newcommand*{\Scale}[2][4]{\scalebox{#1}{\ensuremath{#2}}}

\newcommand{\supp}{\mathrm{supp}}
\newcommand{\dWp}[1]{d_{\mathrm{W},#1}}
\newcommand{\dW}{\dWp 1}
\newcommand{\dGWp}[1]{d_{\mathrm{GW},#1}}
\newcommand{\dH}{d_\mathrm{H}}
\newcommand{\dWkernp}[1]{d_{\mathrm{W},#1}}
\newcommand{\dWkern}{\dWkernp{1}}
\newcommand{\TV}{d_{\mathrm{TV}}}
\newcommand{\con}{\cC_{\mathrm{con}}}%Inverse-connected correspondences between spaces
\newcommand{\pow}{\mathrm{Pow}}%Power set
\newcommand{\rad}{\mathrm{Rad}}%Rademacher distribution

 \newcommand{\journal}[1]{}
 
\renewcommand{\epsilon}{\varepsilon}

\newcommand{\brantley}[1]                {{\textcolor{blue} {#1}}}

\newcommand{\bob}[1]                {{\textcolor{Green} {#1}}}

\title{Geometry and Stability of Supervised Learning Problems}
\date{}

\usepackage{lastpage}
\jmlrheading{24}{2024}{1-\pageref{LastPage}}{3/24}{x}{x}{Facundo M\'emoli and Brantley Vose and Robert C. Williamson}
\ShortHeadings{Geometry and Stability of Supervised Learning Problems}{M\'emoli, Vose, and Williamson}
\firstpageno{1}

\editor{Pierre Alquier}

\author{%
   \name Facundo M\'emoli \email{facundo.memoli@gmail.com}\\
   \addr{Department of Mathematics\\ Rutgers University, Piscataway, NJ 08854}
   \AND
   \name Brantley Vose \email{vose.5@osu.edu}\\
   \addr{Department of Mathematics\\ The Ohio State University, Columbus, OH 43210}
   \AND 
   \name Robert C. Williamson \email{Bob.Williamson@uni-tuebingen.de}\\
    \addr{University of T\"{u}bingen and T\"{u}bingen AI Center, 72076 T\"{u}bingen, Germany }
}
\date{}

\usepackage{lastpage}
\jmlrheading{26}{2025}{1-\pageref{LastPage}}{3/24; Revised
8/25}{8/25}{24-0322}{Facundo M\'emoli, Brantley Vose, Robert C. Williamson}
\ShortHeadings{Geometry and Stability of Supervised Learning Problems}{M\'emoli, Vose, Williamson}

\begin{document}

\maketitle
\begin{abstract}%
    We introduce a notion of distance between supervised learning problems, which we call the \emph{Risk distance}. This distance, inspired by optimal transport, facilitates stability results; one can quantify how seriously issues like sampling bias, noise, limited data, and approximations might change a given problem by bounding how much these modifications can move the problem under the Risk distance. With the distance established, we explore the geometry of the resulting space of supervised learning problems, providing explicit geodesics and proving that the set of classification problems is dense in a larger class of problems. We also provide two variants of the Risk distance: one that incorporates specified weights on a problem's predictors, and one that is more sensitive to the contours of a problem's risk landscape.
\end{abstract}

\begin{keywords}
  supervised learning, stability, metric geometry, optimal transport, risk landscape
\end{keywords}

\section{Introduction}\label{sec:introduction}

In machine learning, even before beginning work on a problem, we are often forced to accept discrepancies between the problem we want to solve and the problem we actually get to work with. We may put up with noise or bias in the data collection process which distorts our view of the true distribution of observations. We may replace a loss function with a surrogate loss whose computational cost or optimization properties are preferable. We may have access to only a small sample of observations. Such compromises are a necessary reality.

This brings us to our primary motivating questions.
\begin{itemize}
    \item \textbf{Question 1}: How much can such a compromise change our problem and its descriptive features?
    \item \textbf{Question 2}: How much effect can multiple compromises have in conjunction? Can we guarantee that a sequence of small changes will not have drastic effects on the problem to be solved?
\end{itemize}

\subsection{Overview of our approach.}
In this paper, we provide a comprehensive framework with which to answer such questions. While many frameworks exist to answer Question~1, these methods are concerned with quantifying changes to one or two aspects of a problem at a time, limiting their ability to answer Question~2. In real world problems, multiple simultaneous corruptions and substitutions are unavoidable.
Our framework is broad enough to handle simultaneous changes to many aspects of a problem. To begin, we give a precise definition of ``supervised learning problem'' and define a notion of distance, dubbed the \emph{Risk distance} and denoted $\dexp$, by which to compare pairs of problems.
This gives rise to the \emph{(pseudo)metric space of supervised learning problems}. The Risk distance lets us make geometric sense of Questions~1 and 2; we can measure how much a compromise affects a problem by seeing how far the problem moves under the Risk distance.

To actually define the Risk distance $\dexp$, we draw on the wisdom of metric geometry.
In 1975, Edwards constructed a metric on the set of isomorphism classes of compact metric spaces which came to be known as the Gromov-Hausdorff distance \citep{edwards_structure_1975, gromov_structures_1981}.
Similarly, the Gromov-Wasserstein distance, introduced by \citet{memoli_GH_Distance,memoli_gromovwasserstein_2011}, provides an optimal-transport-based metric on the collection of isomorphism classes of metric spaces equipped with probability measures. The Gromov-Hausdorff and Gromov-Wasserstein distances have become integral to the theory of metric geometry by facilitating a geometric understanding of spaces in which the points themselves are spaces.

Inspired by this tradition, we craft the Risk distance in the image of the Gromov-Wasserstein distance. This sets up the following analogy between supervised learning and metric geometry.
\begin{align*}
    \textbf{Supervised Learning} & \qquad\ \textbf{Metric Geometry}\\
    \text{Problem}&\longleftrightarrow\text{Metric measure space} \\
    \text{Loss function}&\longleftrightarrow\text{Metric} \\
    \text{Risk distance}&\longleftrightarrow\text{Gromov-Wasserstein distance}
\end{align*}

\subsection{Previous approaches.}
Question~1 has previously been explored for various specific kinds of changes to a problem. In our work, a \emph{supervised learning problem} (or just \emph{problem},   for short), is modeled as a 5-tuple $P=(X,Y,\eta,\ell,H)$ where $X$ and $Y$ are the \emph{input space} and \emph{response space} respectively, $\eta$ is \emph{the joint law}: a probability measure on $X\times Y$, $\ell$ is the \emph{loss function} and $H$ is the \emph{predictor set}: a collection of functions $h:X\to Y$.

\begin{itemize}
    \item \textbf{Joint law $\eta$.} A common assumption in machine learning is that the data is drawn according to an unknown underlying probability measure. Concepts such as noise and bias in data collection or data shift phenomena such as covariate shift, label shift, or concept drift \citep[Ch 8]{huyen_designing}, can be described as changes to this underlying measure. The effects of various kinds of noise \citep{zhu_class_2004, natarajan_learning_2013, menon_learning_2015, menon_learning_2018, rooyen_theory_2018, iacovissi_general_2023} and data shifts \citep{shimodaira_improving_2000}
    on supervised learning problems is a longstanding area of research. A geometric framework for understanding changes in probability measures exists as well; information geometry seeks to understand spaces of probability measures from the viewpoint of Riemannian geometry \citep{amari_information_2016, ay_information_2017}.
    \item \textbf{Loss function $\ell$.} A theoretically attractive loss may have poor optimization properties, prompting one to replace it with a so-called surrogate loss. Alternatively, an attractive loss could be expensive to exactly compute, making it a candidate for approximation, like replacing a Wasserstein-based loss with a Sinkhorn-based loss in probability estimation \citep{cuturi_sinkhorn_2013}. Replacing the loss represents a tradeoff between theoretical properties and computational efficiency, and the quantitative details of this tradeoff have been explored in many contexts \citep{lin_note_2004,zhang_statistical_2004,bartlett_convexity_2006,steinwart_how_2007,awasthi_h-consistency_2022,mao_cross-entropy_2023}.
    \item \textbf{Predictor set $H$.} Universal approximation theorems, popular in deep learning, establish that certain classes of models can approximate large classes of predictors arbitrarily well \citep{pinkus_approximation_1999}. These can be seen as theorems which compare large, intractable predictor sets to those that can be produced by a given model, aiming to show that there is effectively no difference between selecting a predictor from either set. Approximation theory more broadly is similarly concerned with the approximation power of function classes or transforms \citep{trefethen_approximation_2019}.
    \item \textbf{Input and response spaces $X$ and $Y$.} Modifications of the input and response spaces are implicit in many of the modifications listed above. Additionally, the process of feature engineering \citep[Ch 5]{huyen_designing}, for which there is little established theory, can be seen as a modification of the input space. Examples of output space modification include the common relaxation of classification to class probability estimation or, less common, the discretization of the response space into a finite set of labels \citep{langford_probing_2005} which served as a motivation for the exact results by \citet{Reid:2011aa}.
    \item \textbf{Combinations.} 
    Our notion of a learning problem comprises five separate components. There are 
    some existing results that explore how changes to some components affect
    the others. Some of this work goes under the name of ``Machine Learning
    Reductions'' \citep{Langford:2009aa}. There are results on how changes to 
    the distribution of examples $\eta$ (due to ``noise'') has the effect 
    of changing  the  loss function $\ell$ (with label noise) 
    \citep{rooyen_theory_2018} or the  model class $H$ 
    (with attribute noise) \citep{Bishop:1995aa,Dhifallah:2021aa,
    Rothfuss:2019aa,Smilkov:2017aa, williamson_information_2024}.

\end{itemize}

\begin{figure}
    \centering
    \begin{tikzpicture}
        \shade[ball color = lightgray, opacity = 0.5] (0,0,0) ellipse (5cm and 3cm);
    
        \coordinate (P1) at (-3,1.5);
        \coordinate (l1) at (-2.5,0.6);
        \coordinate (P2) at (-1.3,0.7);
        \coordinate (l2) at (-1.2,-0.5);
        \coordinate (P3) at (0.3,-1.5);
        \coordinate (l3) at (2.5,-1.3);
        \coordinate (P4) at (2.5,-0.1);
        \coordinate (l4) at (3.3,0.9);
        \coordinate (P5) at (2,1.7);

        \fill (P1) circle (2pt) node[above] {P};
        \fill (P2) circle (2pt);
        \fill (P3) circle (2pt);
        \fill (P4) circle (2pt);
        \fill (P5) circle (2pt) node[above] {P'};
    
        \draw[red, -{Latex[scale=1.5]}] (P1) -- (P2);
        \draw[red, -{Latex[scale=1.5]}] (P2) -- (P3);
        \draw[red, -{Latex[scale=1.5]}] (P3) -- (P4);
        \draw[red, -{Latex[scale=1.5]}] (P4) -- (P5);
        \draw[black, dashed] (P1) -- (P5) node[midway, above] {$\dexp(P,P')$};
            
        \node[align=left] at (l1) {Add\\ noise};
        \node[align=left] at (l2) {Add\\ bias};
        \node[align=left] at (l3) {Approximate\\ loss};
        \node[align=left] at (l4) {Restrict\\ predictors};

    \end{tikzpicture}
    \caption{A simple depiction of how a sequence of changes can push a particular problem through the space of problems. An idealized problem $P$ is modified by a sequence of four changes, each change moving the problem some distance in the space of problems. Once the sequence of corruptions has been applied, we call the resulting problem $P'$. The original and corrupted problems are connected by a dashed line representing the risk distance between them.}
    \label{fig:small-changes}
\end{figure}
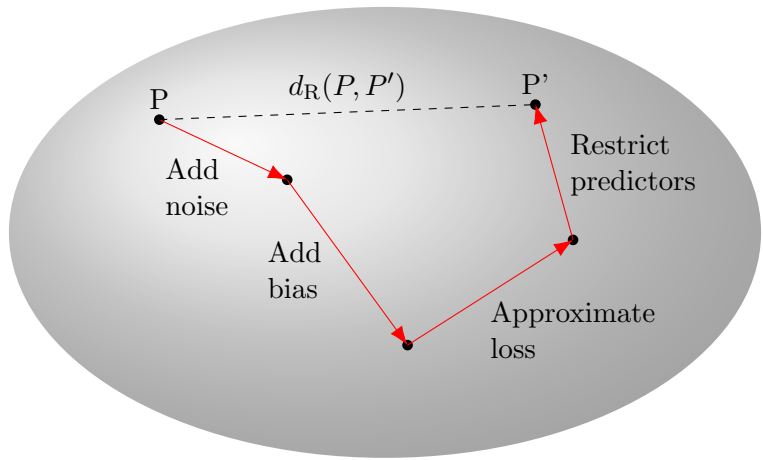
There is precedent for geometric frameworks regarding statistical learning. We have mentioned the rise of information geometry as a way to understand statistical manifolds. Also, much as we aim to do with supervised learning problems, Le Cam developed a notion of distance \citep{leCam_sufficiency, mariucci2016le} between the ``statistical experiments'' of Blackwell \citep{blackwell_comparison_1951}. A modified discrepancy incorporating a loss function was introduced by \citet[Complement 37]{torgersen_comparison_1991}. We will discuss the relationship between supervised learning problems and statistical experiments in Section~\ref{subsec:statistical-experiments}. %\facundo{(mention Blackwell)} \brantley{Added}

While much research has sought to quantify and control the effects of various specific corruptions on supervised learning problems, no framework so far has been comprehensive enough to consider changes to all aspects of a problem. That is, while many have given answers to Question~1, to the best of our knowledge no attempt has been made at a unified answer to Question~2.

\subsection{Our contributions.}
We provide a geometric framework with which to answer both Questions~1 and 2 in the form of the Risk distance.

Our primary tool for answering Question~1 with the Risk distance is the central topic of Section~\ref{sec:stability}: stability results. A stability result is an upper bound on how much a certain change to a problem can move that problem under the Risk distance. That is, if we apply a certain change to a problem $P$ to obtain a new problem $P'$, such as a change to the loss function, the data collection process, or the set of potential predictors, a stability result is an upper bound on $\dexp(P,P')$. Such a result acts as a quantitative guarantee that a small change to some aspect of the problem will not drastically change the problem as a whole.

Comparing problems via a metric also presents advantages over more general kinds of dissimilarity scores: the triangle inequality lets one string together upper bounds on distances.
To wit, let $x_0$ be a point in a metric space $(X,d_X)$. Suppose we create a sequence of points $x_0, x_1, \dots, x_k$ by pushing each point $x_{i-1}$ a distance of at most $\epsilon$ to get the next point $x_i$ in the sequence. Then the triangle inequality assures us that
\[d_X(x_0,x_k)\leq \sum_{i=1}^k d_X(x_{i-1},x_i) \leq k\epsilon.\] 
In other words, in a metric space, a short sequence of small pushes cannot add up to a drastic change. In the context of the Risk distance, this feature lets us handle multiple simultaneous changes to a problem by applying the triangle inequality and stringing together stability results. This gives an answer to Question~2. See Figure~\ref{fig:small-changes} for an illustration. By design, the Risk distance will be a pseudometric (that is, a metric for which distinct points can be distance zero apart), but the above virtues of metric spaces apply to pseudometric spaces as well.

Along with stability of the Risk distance under various problem modifications, Section~\ref{sec:stability} will also demonstrate stability of certain \emph{descriptors of learning problems} with respect to the Risk distance. In particular, we prove in Section~\ref{subsec:loss-control} that a descriptor of a problem called its \emph{constrained Bayes risk} is stable under the Risk distance. In other words, if the Risk distance between two problems $P$ and $P'$ is less than $\epsilon$, then the optimal expected loss across all predictors for $P$ is within $\epsilon$ of that of $P'$; see Figure \ref{fig:overall} for an illustration.  We obtain this stability result by first proving stability under the Risk distance of a stronger descriptor, called the \emph{loss profile set} of a problem, which describes the possible distributions of losses incurred by a random observation. We also demonstrate stability under the Risk distance of a descriptor from statistical learning theory known as the \emph{Rademacher complexity} for arbitrarily large sample sizes.

\begin{figure}
\centering
\includegraphics[width = \textwidth]{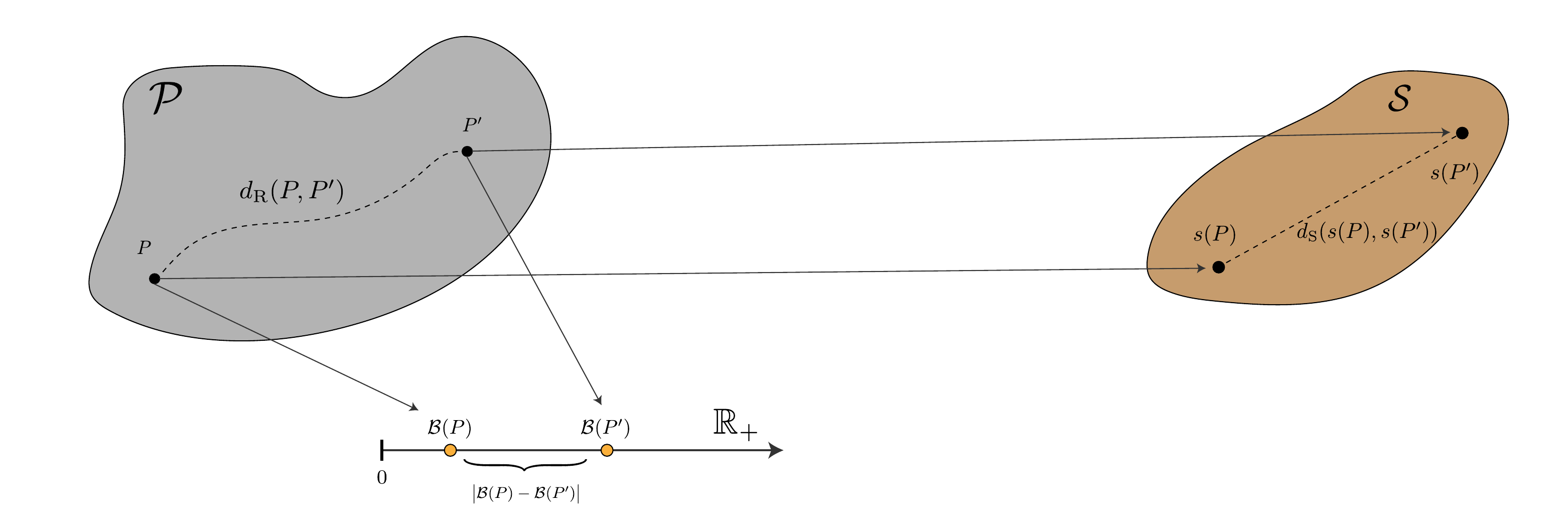}

\caption{The space $\mathcal{P}$ of all supervised learning problems and a descriptor. A descriptor is modeled as a map $s:\mathcal{P}\to \mathcal{S}$ whose codomain is some metric space $(\mathcal{S},d_{\mathcal{S}})$. A descriptor $s$ is regarded as stable  if there is some constant $C>0$ such that $C\,\dexp(P,P)\geq d_{\mathcal{S}}(s(P),s(P'))$ for all $P,P'\in\mathcal{P}$. The figure also depicts the case when $\mathcal{S}$ is the set $\mathbb{R}_{+}$ of non-negative reals and $s(P)$ is the constrained Bayes risk $\cB(P)$  of $P$. In that case, one of our stability results (Theorem \ref{thm:loss-control}) establishes that $\dexp(P,P')\geq \big|\cB(P)-\cB(P')\big|$ for all $P,P'\in\mathcal{P}.$}\label{fig:overall}
\end{figure}
In addition to facilitating stability results, the Risk distance gives us a notion of convergence for supervised learning problems. Just as it is fruitful to discuss the convergence of sequences of functions, random variables, or probability distributions, so too can it be useful to discuss the convergence of a sequence of problems. In Section~\ref{subsec:empirical-convergence}, we will demonstrate an analog of the Glivenko–Cantelli theorem by proving that the ``empirical problem,'' a problem constructed using a finite random sample of observations instead of the true distribution from which it is sampled, converges to the ``true problem'' under the Risk distance with probability~1.

Lastly, the Risk distance lets us consider the geometry of the space of problems and consequences thereof. In Sections~\ref{subsec:geodesics} and \ref{subsec:optimal-couplings}, we connect the geometric concept of a \emph{geodesic} between two problems with the probabilistic notion of couplings of distributions associated with those problems. This connection lets us write down some explicit geodesics in the space of problems. In Section~\ref{subsec:density-of-classification}, we show that classification problems, meaning those problems with a finite output space, are dense in a larger space of problems under the Risk distance. This density result proves useful when establishing the convergence of empirical problems.

In the final two sections, we also study two variants of the Risk distance. In Section~\ref{sec:probabilistic}, we define a parameterized family of distances, each called the \emph{$L^p$-Risk distance} for $p\in [1,\infty]$. The $L^p$-Risk distance incorporates weights in the form of a specified probability measure on the set of possible predictors. The use of this additional data results in improved stability when compared to the Risk distance. In Section~\ref{sec:topological}, we also define a stronger version of the Risk distance, called the \emph{Connected Risk distance}, which is designed to be more sensitive to changes in the risk landscape of a problem. In particular, we show that a topological summary of the problem called its \emph{Reeb graph} is stable under the Connected Risk distance.

\subsection{Outlook}

We regard the Risk distance as a tool to produce theoretical guarantees. To prove that problems are stable under a certain kind of change with respect to the Risk distance, or that certain descriptors are stable under the Risk distance, we need only bound the Risk distance above or below. The utility of the Risk distance lies not in explicit algorithms and associated computations, but in the theoretical landscape that it facilitates. The current landscape of machine learning is rapidly expanding, particularly in the space of problem formulations. Given the staggering growth of machine learning research (e.g., 27,000 submissions to NeurIPS 2025~\citep{neurips2025submissions}), there is clear value in providing a conceptual framework to help organize and interpret this proliferation of methods. This paper takes an initial step toward an organizing ‘map’ of this space—namely, we endow the class of supervised learning problems with a metric that induces a geometric organization, so existing and emerging methods can be situated and compared within a common landscape. While this paper focuses on supervised learning, many generative models, such as GANs, are implicitly tied to supervised learning problems via divergence measures—e.g., $f$-divergences—which are in one-to-one correspondence with certain supervised objectives~\citep{nowozin2016fgan,williamson_information_2024}.

\subsection{Other previous work.}
Kleinberg's seminal paper \citep{kleinberg} initiated a thread of research into the formal study of clustering problems \citep{carlsson2010characterization,carlsson_classifying_2013, ackerman_towards_2012, CohenAddad2018ClusteringRT}. In a spirit similar to this thread, the present paper aims to formally examine supervised learning problems.

The Risk distance introduced in this paper participates in a growing thread of research into optimal-transport-based distances, a thread which begins with the metrics discussed in Section~\ref{subsec:MG-and-OT} and runs through more recent variants including optimal transport distances for graph-based data \citep{chowdhury_gromovwasserstein_2019, chen2022weisfeiler}, matrices \citep{peyre_gromov-wasserstein_2016}, partitioned data sets \citep{chowdhury_quantized_2021}, labeled data sets \citep{alvarez-melis2020geometric}, optimal transport of both data and features jointly \citep{titouan_co-optimal_2020}, or of data with quite general structures \citep{Vayer2020FusedGD,patterson2020}, as well as approximations \citep{scetbon2022linear,sejourne_unbalanced_2021,VincentCuaz2021SemirelaxedGW} and regularizations \citep{peyre_gromov-wasserstein_2016} thereof.

\subsection{Statement of contribution}
The genesis of this project dates back to 2010-2011 when F.M. visited R.W. at NICTA (Canberra, Australia). The formulation of the results, finding the proofs, and the vast majority of the
writing of the paper was done by F.M. and B.V. R.W. contributed the initial formulation
of the question to be addressed, provided some context from the ML literature,
and reviewed and slightly revised the final document.

\tableofcontents

\section{Background}

We begin with some necessary concepts and notation from measure theory, metric geometry, and optimal transport. We refer the reader to the standard reference of \citet{villani_topics_2003} for more on measure-theoretic and optimal transport concepts, and that of \citet*{BBI} for metric geometry concepts. 

\subsection{Preliminary Notation}

We use $\pi_i$ to represent a projection map $X_1\times X_2 \times \dots \times X_n \to X_i$ onto the $i$th component of a product. More generally, $\pi_{i_1,\dots,i_k} X_1\times \dots \times X_n \to X_{i_1}\times \dots \times X_{i_k}$ represents the projection onto coordinates $i_1, \dots, i_k$.

When integrating a measurable function $f$ against a measure $\mu$, we will favor the notation
\[\int_{X\times Y} f(x,y) \, \mu(dx\times dy)\]
if $f$ is a function of two variables, and similarly for more variables. We often suppress the domain of integration under the integral symbol since it is implicit in the measure $\mu$.
Rarely, when many variables are necessary and horizontal space is at a premium, we may instead use the notation
\[\int f(x,y) \, d\mu(x,y).\]

Given a measurable space $(Z,\Sigma_Z)$, we let $\prob(Z)$ denote the space of probability measures on $(Z,\Sigma_Z)$. 
For any $z\in Z$, we let $\delta_z \in \prob(Z)$ denote the \emph{Dirac probability measure at $z$} defined, for $A\in \Sigma_Z$, as $\delta_z(A)=1$ whenever $z\in A$ and as $0$ otherwise.
Given a probability measure $\alpha \in \prob(\mathbb{R})$, we let $\mathrm{mean}(\alpha) := \int_\mathbb{R} t\,\alpha(dt)$ denote its mean whenever it exists.

\subsection{Pushforward Measures, Correspondences, and Couplings}

We recall some basic ideas from measure theory and establish notation. Given a measure space $(X,\Sigma_X,\mu)$,  a measurable space $(Y,\Sigma_Y)$, and a measurable map $f:X\to Y$, the \emph{pushforward of $\mu$ through $f$} is the measure on $(Y,\Sigma_Y)$ defined by
\[f_\sharp \mu (A):= \mu(f^{-1}(A))\]
for all $A\in \Sigma_Y$. If $\mu$ is a probability measure, $f_\sharp \mu$ will be a probability measure as well. Hence $f_\sharp$ is a function $\prob(X) \to \prob(Y)$. A pushforward has a probabilistic interpretation: if $\mathsf X$ is a random variable in a measurable space $X$ with law $\mu$, and $f:X\to Y$ is measurable, then the random variable $f(\mathsf X)$ has law $f_\sharp \mu$.

Given two sets $X$ and $Y$, a subset $R\subseteq X\times Y$ is called a \emph{correspondence} if

\begin{itemize}
  \item 
  for all $x\in X$ there is some $y\in Y$ with $(x, y)\in R$, and
  \item for all $y\in Y$ there is some $x\in X$ with $(x, y)\in R$.
\end{itemize}
We use $\cC(X,Y)$ to denote the set of all correspondences between $X$ and $Y$. 

A \emph{coupling} between two probability measure spaces $(X,\Sigma_X,\mu_X)$ and $(Y,\Sigma_Y,\mu_Y)$ is a probability measure on $X\times Y$ with marginals $\mu_X$ and $\mu_Y$. That is, a probability measure $\gamma$ on $X\times Y$ is a coupling if the pushforward measures satisfy $(\pi_1)_\sharp \gamma = \mu_X$, $(\pi_2)_\sharp \gamma = \mu_Y$. We use $\Pi(\mu_X,\mu_Y)$ to denote the set of all such measures. Couplings can be seen as a probabilistic counterpart to correspondences.

The following are standard methods for constructing correspondences between two sets $X$ and $Y$.
\begin{itemize}
    \item The entire product $X\times Y$ is always a correspondence. In particular, this shows that $\cC(X,Y)$ cannot be empty.

\item If $X=Y$, one can construct the \emph{diagonal correspondence} $\set{(x,x)}{x\in X}$.

\item If $Y=\{\bullet\}$ is the one point set, then there is a unique correspondence $X\times \{\bullet\}$.
    \item The graph of any surjection $f:X\to Y$ is a correspondence between $X$ and $Y$.
\end{itemize}

Analogous constructions exist for couplings as well. Let $(X,\mu_X)$ and $(Y,\mu_Y)$ be probability spaces.
\begin{itemize}
    \item The product measure $\mu_X\otimes \mu_Y$ is always a coupling, sometimes called the \emph{independent coupling}. This shows that $\Pi(\mu_X,\mu_Y)$ cannot be empty.
    \item Suppose $(X,\mu_X) = (Y,\mu_Y)$. If $\Delta_X:X\to X\times X$ denotes the diagonal map $\Delta_X(x) = (x,x)$, then the \emph{diagonal coupling} is the pushforward measure $(\Delta_X)_\sharp \mu_X$.
    
    \item If $Y=\{\bullet\}$ is the one point set, then forcibly $\mu_Y = \delta_\bullet$ and there is a unique coupling $\mu_X\otimes \delta_\bullet$ in this case.
    
    \item Suppose $f:X\to Y$ is a measure-preserving map, meaning $f_\sharp \mu_X = \mu_Y$. Let $i_f:X\to X\times Y$ be the map sending $X$ to the graph of $f$, meaning $i_f(x) := (x,f(x))$. Then the pushforward $(i_f)_\sharp \mu_X$ is a coupling between $\mu_X$ and $\mu_Y$.
\end{itemize}

When we discuss Markov kernels in Section~\ref{subsec:markov-kernels}, we will mention one more method for generating couplings by combining a probability measure $\alpha \in \prob(X)$ and a Markov Kernel $\beta:X\to\prob(Y)$ using reverse disintegration.

We will make use of the following standard coupling result, which can be found in any standard optimal transport text, such as that of \citet[Lemma 7.6]{villani_topics_2003}. Note that a topological space is called \emph{Polish} if it is separable and metrizable by a complete metric. Most familiar spaces are Polish, including $\mathbb R^n$ and open subsets thereof, manifolds with or without boundary, and finite discrete spaces. A Borel measurable space arising from a Polish topological space is called Polish as well.
\begin{proposition}[The Gluing Lemma]\label{lem:gluing}
    Let $\mu_1$, $\mu_2$ and $\mu_3$ be probability measures on Polish spaces $X_1$, $X_2$, and $X_3$ respectively. Let $\gamma_{1,2} \in \Pi(\mu_1, \mu_2)$ and $\gamma_{2,3} \in \Pi(\mu_2, \mu_3)$. Then there exists a probability measure $\gamma_{1,2,3}$ on $X_1\times X_2 \times X_3$ with marginals $\gamma_{1,2}$ on $X_1\times X_2$ and $\gamma_{2,3}$ on $X_2\times X_3$.
\end{proposition}
We call such a $\gamma_{1,2,3}$ a \emph{gluing} of $\gamma_{1,2}$ and $\gamma_{2,3}$.
An analog of the  gluing lemma for correspondences \citep[see][Exercise 7.3.26]{BBI} instead of couplings is included below in order to highlight the parallels between couplings and correspondences.

\begin{proposition}[The Gluing Lemma for Correspondences]\label{lem:gluing-correspondences}
    Let $A_1$, $A_2$ and $A_3$ be sets, and let $R_{1,2} \in \cC(A_1, A_2)$ and $R_{2,3} \in \cC(A_2, A_3)$. Then there exists a subset $R_{1,2,3}\subseteq A_1\times A_2 \times A_3$ which projects to $R_{1,2}$ on $A_1\times A_2$ and $R_{2,3}$ on $A_2\times A_3$.
\end{proposition}
We call such an $R_{1,2,3}$ a \emph{gluing} of $R_{1,2}$ and $R_{2,3}$.

\begin{proof}
    Selecting $R_{1,2,3} = \set{(x,y,z) \in A_1{\times}A_2{\times}A_3}{(x,y)\in R_{1,2}, (y,z)\in R_{2,3}}$ provides one such subset.
\end{proof}

While some measurable spaces can be quite poorly behaved (for instance, those that fail to exhibit a gluing lemma), the measurable spaces in this paper will all arise from relatively well-behaved topological spaces.

\begin{assumption}\label{assume:polish} We assume that any measurable space that appears will arise from a Polish topological space by equipping it with the Borel $\sigma$-algebra. Therefore, when convenient, we will suppress the $\sigma$-algebra in our notation for measurable spaces. That is, instead of writing a measurable space as a pair $(X,\Sigma_X)$ with $X$ a Polish space and $\Sigma_X$ its Borel $\sigma$-algebra, we will simply write $X$ with the understanding that the Borel $\sigma$-algebra is being used. We will similarly write measure spaces $(X, \Sigma_X, \mu)$ as pairs $(X,\mu)$.
    
\end{assumption}

Bearing Assumption~\ref{assume:polish} in mind, we define the \emph{support} of the measure $\mu$ as the closed set $$\supp[\mu]:=\{x\in X\,|\,\mu(U)>0\,\forall\,U\,\mbox{open s.t.}\, x\in U\}.$$

\subsection{Metric Geometry and Optimal Transport}\label{subsec:MG-and-OT}

Here we present four metrics from the field of metric geometry. The following diagram \citep[inspired by][Fig. 2]{memoli_gromovwasserstein_2011} illustrates the relationships between the four metrics. 
\[
\begin{tikzcd}
   & \text{Hausdorff} \ar[dl,"\substack{\text{Remove}\\\text{ambient space}}",swap] \ar[dr,"\substack{\text{Introduce}\\\text{probability measure}}"]\\
   \text{Gromov-Hausdorff}\ar[dr,"\substack{\text{Introduce}\\\text{probability measure}}",swap] & & \text{Wasserstein} \ar[dl,"\substack{\text{Remove}\\\text{ambient space}}"]\\
   & \text{Gromov-Wasserstein}
\end{tikzcd}  
\]
We do not explicitly use the Gromov-Hausdorff or Gromov-Wasserstein distances in this paper. However, these distances provide context by illustrating common threads that run through all four metrics and is shared by the Risk distance. In all four of the metrics, objects are compared by optimally ``aligning'' the two objects, whether via correspondences or couplings, and numerically scoring the distortion incurred by an optimal alignment. The Risk distance (Definition~\ref{def:metric}) follows this same template. To emphasize the continuation of this thread, we will take a moment to compare the Risk distance to the Gromov-Wasserstein distance in Section~\ref{subsec:OT-connection}.

Note that we will often suppress the underlying metric in our notation, writing a metric space $(X,d_X)$ as simply $X$.

\subsubsection{Hausdorff Distance}
The first metric is the Hausdorff distance. Introduced by \citet{hausdorff_grundzuge_1914}, it provides a notion of distance between compact subsets of a fixed metric space. Given two compact subsets $X,Y$ of a metric space $(Z,d_Z)$, the Hausdorff distance between them is typically defined to be
\begin{align*}
    \dH^Z(X,Y) := \max\left\{\sup_{x\in X} \inf_{y\in Y} d_Z(x,y), \sup_{y\in Y} \inf_{x\in X} d_Z(x,y)\right\}.
\end{align*}
To emphasize the commonalities among the four metrics of this section, we will use an alternative equivalent definition provided by \citet[Prop. 2.1]{memoli_gromovwasserstein_2011} in terms of correspondences.

\begin{proposition}
  Given two compact subsets $X$ and $Y$ of a metric space $(Z,d_Z)$, the Hausdorff distance between $X$ and $Y$ is equal to
  \begin{align*}
    \dH^Z(X,Y) = \inf_{R\in \cC(X,Y)} \sup_{(x,y) \in R} d_Z(x,y).
  \end{align*}
\end{proposition}

\subsubsection{Gromov-Hausdorff Distance}
The second metric of this section is the Gromov-Hausdorff distance, an extension of the Hausdorff distance that removes the need for a common ambient metric space. While the Hausdorff distance compares compact subsets of a common metric space $Z$, the Gromov-Hausdorff distance compares pairs of compact metric spaces.

\begin{definition}\label{def:GH-distance}
  Given two compact metric spaces $(X,d_X)$ and $(Y,d_Y)$, the \emph{Gromov-Hausdorff distance} between them is given by
  \begin{align*}
    d_{\op{GH}}(X,Y) := \inf_{R\in \cC(X,Y)} \sup_{(x,y), (x',y') \in R} \frac{1}{2} \big|d_X(x,x') - d_Y(y,y')\big|.
  \end{align*}
\end{definition}

The function $d_{\op{GH}}$ was first discovered by \citet[Def. III.3]{edwards_structure_1975} and later rediscovered by its namesake \citet{gromov_structures_1981}, both with alternative (but equivalent) definitions different from the one given in Definition~\ref{def:GH-distance}. The formulation above in terms of correspondences is due to \citet{Kalton_Mikhail_1997}.
This formula does indeed define a metric on the set of compact metric spaces up to isometry, meaning that $d_{\op{GH}}(X,Y) = 0$ if and only if $X$ and $Y$ are isometric \citep[Thm. 7.3.30]{BBI}.

\subsubsection{Wasserstein Distance}

The third metric of this section is the Wasserstein distance.
\begin{definition}[Kantorovich, 1942]
  Let $p\in [1,\infty)$. Given two probability measures $\mu$, $\nu$ on a metric space $(Z,d_Z)$, the \emph{$p$-Wasserstein distance} between $\mu$ and $\nu$ is
  \begin{align*}
    \dWp{p}^Z(\mu,\nu) := \inf_{\gamma\in \Pi(\mu,\nu)} \left(\int_{X\times X} d_Z^p(x,y) \, \gamma(dx\times dy) \right)^{1/p}.
  \end{align*}
\end{definition}

Notice that to construct the Wasserstein distance, we can begin with the Hausdorff metric and replace the correspondences with their probabilistic counterparts: couplings. We then replace the supremum with an $L^p$-style integral, and our job is done. Moving forward, since we primarily make use of the 1-Wasserstein distance, we will refer to $\dW$ simply as the \emph{Wasserstein distance}.

The Wasserstein distance is also called the \emph{earth mover's distance} due to the following interpretation. Suppose $(X,d_X)$ represents the ground of a garden equipped with the 2-d Euclidean metric, and $\mu$ describes the distribution of a pile of soil on $X$. Suppose that $\nu$ is a desired distribution of soil. Then a coupling $\gamma \in \Pi(\mu,\nu)$ represents a plan for moving the soil from distribution $\mu$ to distribution $\nu$, and $\dW^X(\mu,\nu)$ is the amount of work that an optimal plan will require. Indeed, the Wasserstein metric seems to have first appeared in the work of Soviet mathematician and economist L. V. Kantorovich, who was studying optimal transportation of resources. For an English translation of Kantorovich's 1942 paper, see \citet{kantorovich_translocation_2006}, and for a historical survey of this often-rediscovered metric, see \citet{vershik_kantorovich_2006}.

\subsubsection{Gromov-Wasserstein Distance}
The fourth metric of this section is the Gromov-Wasserstein distance. Just as the Gromov-Hausdorff metric is an ambient-space-free variant of the Hausdorff distance, the Gromov-Wasserstein distance is an ambient-space-free variant of the Wasserstein distance. It can be used to compare probability measures regardless of the underlying metric space.
\begin{definition}
    A \emph{metric measure space} is a triple $(X,d_X,\mu_X)$ such that $(X,d_X)$ is a metric space, and $\mu_X$ is a Borel probability measure on $X$ with full support.
\end{definition}
For more on metric measure spaces, we direct the reader to Mikhael Gromov's foundational book \citep[Chapter 3$\frac12$]{gromov_convergence_2007}.

\begin{definition}[\citet{memoli_GH_Distance,memoli_gromovwasserstein_2011}]\label{def:gw}
  Let $p \in [1,\infty)$. The \emph{$p$-Gromov-Wasserstein distance} between two metric measure spaces $(X,d_X,\mu_X)$ and $(Y,d_Y,\mu_Y)$ is given by
  \begin{align*}
    \dGWp{p}(X,Y) := \inf_{\gamma \in \Pi(\mu_X,\mu_Y)} \frac{1}{2}
    \left(\int_{X\times Y \times X \times Y}\big|d_X(x,x')-d_Y(y,y')\big|^p \, \gamma(dx\times dy)\gamma(dx'\times dy')\right)^{1/p}.
  \end{align*}
\end{definition}
As with the Wasserstein distance, we will typically take $p=1$ and refer to $\dGWp{1}$ simply as the \emph{Gromov-Wasserstein distance}.

While preceded by Gromov's box distance and observable distance \citep[Ch. 3{\small $\frac12$}]{gromov_structures_1981}, the first to formulate a ``Wasserstein-type'' adaptation of the Gromov-Hausdorff distance to the setting of metric measure spaces was \citet{sturm_geometry_2006}. The definition above was introduced and developed by \citet{memoli_GH_Distance, memoli_gromovwasserstein_2011} in order to apply methods from optimal transport to the computational problem of data comparison and matching. \citet{sturm_space_2012} extensively explored the geometry of the resulting ``space of spaces''.
The Gromov-Wasserstein distance has spawned variants in many domains (some of which are referenced in Section~\ref{sec:introduction}) and applications to various problems in Machine Learning and Data Science such as unsupervised natural language translation \citep{alvarez-melis_gromov-wasserstein_2018}, imitation learning \citep{fickinger_cross-domain_2022},  generative modeling \citep{bunne_learning_2019}, and to the study of Graph Neural Networks in the form of the Weisfeiler-Lehman distance and a corresponding version of the Gromov-Wasserstein distance for comparing labelled Markov processes \citep{chen2022weisfeiler}.  See also \cite{peyre2019computational,Vayer2020FusedGD,titouan2019sliced,peyre_gromov-wasserstein_2016,scetbon2022linear,chowdhury2020gromov,sejourne_unbalanced_2021,delon2022gromov,beier2022linear,le2022entropic,dumont2024existence,arya2023gromov} for other work related to the Gromov-Wasserstein distance.

\subsection{Total Variation Distance}\label{subsec:total-variation}
There is one additional metric that we will use. The simplest method of comparing two probability measures on a common measurable space is the \emph{total variation distance}. See \citet[Sec 4.1]{markov_2017} for an introduction. 

\begin{definition}\label{def:tv-distance} 
    Let $\mu,\nu$ be probability measures on the measurable space $(\Omega, \mathcal S)$. The \emph{total variation distance between $\mu$ and $\nu$} is given by
    \[\TV(\mu,\nu) := \sup_{B \in \mathcal S}\big| \mu(B) - \nu(B) \big|.\]
\end{definition}

It is not actually necessary to take an absolute value in the definition of $\TV$, since the value $| \mu(B) - \nu(B)|$ is achieved by either $\mu(B) - \nu(B)$ or $\mu(\Omega\setminus B) - \nu(\Omega \setminus B)$. Hence,
\begin{align*}
    \TV(\mu,\nu) = \sup_{B\in \mathcal S} \big(\mu(B) - \nu(B)\big).
\end{align*}

Additionally, the total variation distance can be seen as the Wasserstein distance when $\Omega$ is equipped with the discrete metric. In other words,
\begin{align*}
    \TV(\mu,\nu) = \inf_{\gamma \in \Pi(\mu,\nu)} \int 1_{x\neq y} \, \gamma(dx{\times} dy)
    =
    \inf_{\gamma \in \Pi(\mu,\nu)} \gamma\set{(x,y)}{x\neq y}.
\end{align*}
Note that the above equalities only make sense if the diagonal $\Delta:=\set{(x,x)}{x\in \Omega}$ is measurable. This is true under Assumption~\ref{assume:polish}; if $\Omega$ is a Polish space, then $\Delta$ is closed in $\Omega \times \Omega$.

\subsection{Markov Kernels}\label{subsec:markov-kernels}
Here we establish definitions and notation surrounding Markov kernels. See \citet[Ch 14.2]{klenke_probability_2014} for a comprehensive introduction.

Let $(X, \Sigma_X)$ and $(Y,\Sigma_Y)$ be measurable spaces. A map $M:X\times \Sigma_Y \to \mathbb R$ is called a \emph{Markov kernel if $M(x,\cdot)\in \prob(Y)$ for all $x\in X$, and $M(\cdot,A)$ is measurable for all $A\in \Sigma_Y$. One often writes $M(x)(A)$ instead of $M(x,A)$ and thinks of $M$ as a map $X\to \prob(Y)$.} One can use Markov kernels to represent functions where the output is random; given an input $x$, the random output is described by the probability measure $M(x)$. Markov kernels also go by the name ``communication channel'' in information theory since they can be used to model a channel in which random noise affects the output. We will use Markov kernels to model noise in Sections~\ref{subsec:stability-noise} and \ref{subsec:improved-stability}.

There are several standard methods for constructing Markov kernels. A probability measure can be seen as a simple Markov kernel. Given $\mu \in \prob(X)$, we can construct a Markov kernel from a singleton  $\{\bullet\}\to \prob(X)$ mapping $\bullet \mapsto \mu$. Measurable functions can be seen as Markov kernels as well. Given a measurable $f:X\to Y$, we can construct the \emph{deterministic Markov kernel induced by $f$}, denoted $\delta_f$, by setting $\delta_f(x)$ to be the point mass $\delta_{f(x)}$.

We can construct more interesting Markov kernels by the process of \emph{disintegration}. Let $\mu \in \prob(X\times Y)$ with $X$-marginal $\alpha$. Given that $X$ and $Y$ are Polish spaces, there exists a Markov kernel $\beta:X\to \prob(Y)$ satisfying
\begin{align}
    \mu(A\times B) = \int_A \beta(x)(B) \, \alpha(dx).\label{eq:disintegration}
\end{align}
This characterizes $\beta$ up to an $\alpha$-null set. We call such a $\beta$ the \emph{disintegration of $\mu$ along $X$}. The existence of such a $\beta$ is shown in, for instance,  \citep[Theorem~10.4.8]{bogachev_measure_2007}.

Disintegration can be seen as a measure-theoretic formulation of conditional probability. Indeed, if $\mathsf X$ and $\mathsf Y$ are random variables in $X$ and $Y$ respectively with joint law $\mu$, then the disintegration $\beta$ of $\mu$ along $X$ gives the conditional law of $\mathsf Y$ with respect to $\mathsf X$. That is, $\beta(x)(B) = \bbP(\mathsf Y \in B|\mathsf X =x)$ for all $x\in X$ and measurable subsets $B\subseteq Y$.

One can also reverse the disintegration process. Just as a joint law of $(\mathsf X,\mathsf Y)$ can be recovered from the law of $\mathsf X$ and the conditional law of $\mathsf Y$ on $\mathsf X$, so too can we recover a coupling in a similar manner. Indeed, Equation~\eqref{eq:disintegration} shows that the coupling $\mu$ can also be recovered from $\beta$ and $\alpha$. This is another common method for producing couplings; one selects a probability measure $\alpha \in \prob(X)$ and a Markov kernel $\beta:X\to \prob(Y)$ and defines $\mu$ as in Equation~\eqref{eq:disintegration}.

Given two Markov kernels $M:X\to \prob(Y)$ and $N:Y\to \prob(Z)$, their \emph{Markov kernel composition} $M\cdot N:X\to \prob(Z)$ is given by
\[M\cdot N(x)(C) := \int_Y N(y)(C)\, M(x)(dy)\]
for all measurable $C\subseteq Z$.
Given Markov kernels $M:X\to \prob(Y)$ and $N:W\to \prob(Z)$, the \emph{independent product of Markov kernels} $M\otimes N:X\times W \to \prob(Y\times Z)$ is given by
\[M\otimes N(x,w) := M(x) \otimes N(w)\]
where the $\otimes$ on the right hand side denotes the product of measures.

\subsection{Empirical Measures and Glivenko–Cantelli Classes}\label{subsec:GC}

If a data set is sampled from an underlying probability measure $\mu$, the sample itself can be modeled as a random measure called an \emph{empirical measure}.
\begin{definition}
    Given a probability space $(Z,\mu)$, the \emph{$n$th empirical measure of $\mu$} is the random measure given by
    \[\mu_n:= \frac{1}{n}\sum_{i=1}^n \delta_{Z_i}\]
    where $Z_1,\dots, Z_n$ is an i.i.d. sample with law $\mu$.
\end{definition}
A practitioner, having access only to a random sample and not the true underlying law $\mu$ from which it is drawn, can nevertheless approximate $\mu$ with $\mu_n$.
One hopes that for large $n$, $\mu_n$ is similar enough to $\mu$ for practical purposes. For instance, one would hope that integrating against $\mu$ and $\mu_n$ yield similar results when $n$ is large. One formalization of this hope taken from statistical learning theory is the notion of a \emph{Glivenko-Cantelli class}.
\begin{definition}
    Let $(X,\mu)$ be a probability space, and let $\mathcal F$ be a collection of integrable functions $X\to \mathbb R$. Let $\mu_n$ denote the empirical measure of $\mu$. We say that $\mathcal F$ is a \emph{Glivenko-Cantelli class with respect to $\mu$} if
    \[
    \sup_{f\in \mathcal F} \left| \int f(x) \mu(dx) - \int  f(x) \mu_n(dx) \right| \xrightarrow{\text{a.s.}} 0.
    \]

    We say that $\mathcal F$ is a \emph{Universal Glivenko-Cantelli class} if it is a Glivenko-Cantelli class with respect to every $\mu \in \prob(X)$.
\end{definition}
See~\citet{van_der_vaart_weak_1996} for an overview of Glivenko-Cantelli classes and their relationships to other conditions in the theory of empirical processes.

We will employ Glivenko-Cantelli classes in our analysis of the convergence of empirical problems (Section~\ref{subsec:empirical-convergence}), and universal Glivenko-Cantelli classes in our discussion of Rademacher complexity (Section~\ref{subsec:rademacher}).

The Glivenko-Cantelli condition may be illuminated by the following sufficient condition taken from the theory of empirical processes involving the popular Vapnik-Chervonenkis (VC) dimension. Given a family of subsets $\cA$ of some set $Z$, the \emph{VC dimension of $\cA$} is a number which measures the ``flexibility'' of the family $\cA$. Briefly, the VC dimension of $\cA$ is the cardinality of the largest finite set $B\subseteq Z$ such that $\set{A\cap B}{A\in \cA} = \pow(B)$, where $\pow$ denotes the power set. If there is no largest such $n$, we say that $\cA$ has \emph{infinite VC dimension}. For an introduction to VC dimension and its place in the theory of statistical learning, see \citet[Ch 12]{devroye_probabilistic_1996}.

\begin{definition}\label{def:subgraph}
    Let $f:X\to \mathbb R$ be any function. The \emph{subgraph of $f$} is given by
    \[\mathrm{sg}(f) := \set{(x,t)}{t<f(x)}\subseteq X \times \bbR.\]
\end{definition}
\begin{proposition}\label{prop:VC-sufficient}
    Let $\mathcal F$ be a collection of measurable functions $X\to \bbR$ that is uniformly bounded below and above. Suppose $\mathcal F$ is endowed with a Polish topology, and subsequently the Borel $\sigma$-algebra, such that the map $(f,x)\mapsto f(x)$ is jointly measurable in $f$ and $x$. If the collection of subgraphs $\set{\mathrm{sg}(f)}{f\in \mathcal F}$ has finite VC dimension, then $\mathcal F$ is a universal Glivenko-Cantelli class.
\end{proposition}

    Proposition~\ref{prop:VC-sufficient} follows from standard results in the theory of empirical processes which can be found in, for instance, the reference by \citet{van_der_vaart_weak_1996}. Specifically, we use a combination of Theorems~2.6.7 and 2.4.3 along with Example~2.3.5 from that source.

\section{The Risk Distance}\label{sec:distance}

In this section, we introduce the points (problems) and distance function (the Risk distance) which will together define our pseudometric space of supervised learning problems.

\subsection{Problems}
Before we proceed, we must solidify the informal notion of ``supervised learning problem'' into a formal definition. Supervised learning problems come in many forms, including regression, classification, and ranking problems. We would like a definition that is general enough to encompass all of these examples. In all of these problem types, the goal is to examine the relationship between a random variable in an input space $X$ and another in an response space $Y$. One then aims to select an appropriate function $h:X\to Y$ that predicts the output value given an input. The effectiveness of $h$ is measured using a chosen loss function. (As an aside, even ``unsupervised'' problems, such as generating samples from a distribution,  can have a supervised problem ``under-the-hood''; for example, $f$-GANS rely upon judging suitability of a distribution by use of the $f$-divergence \citep{Nock:2017aa}, which is intimately related to a Bayes risk \citep{Reid:2011aa}).

\begin{definition}\label{def:problem}
    A \emph{supervised learning problem}, or simply \emph{problem}, is a 5-tuple
    \[P = (X,Y,\eta,\ell,H)\]
    where
    \begin{itemize}
        \item $X$ and $Y$ are Polish topological spaces called the \emph{input space} and \emph{response space} respectively,\footnote{The Polish condition is put in place to avoid pathologies. In all of the examples we have in mind, $X$ and $Y$ are either open, closed, or convex subsets of $\bbR^n$, or else discrete sets. Such spaces are Polish. This should also help assuage the reader with set-theoretic concerns. We will speak freely of the ``collection of all problems''. While this collection is technically too large to be a set, restricting ourselves to the above examples resolves this issue.}
        \item $\eta$ is a Borel probability measure on $X\times Y$ called the \emph{joint law},
        \item $\ell$ is a measurable function $Y\times Y \to \mathbb{R}_{\geq 0}$ called the \emph{loss function},
        \item and $H$ is a family of measurable functions $X\to Y$, called the \emph{predictors}.\footnote{For greater generality, one could alternatively define a supervised learning problem such that the predictors $h\in H$ are functions into some decision space $D$ which is not necessarily the same as $Y$. That is, one could define a problem to be a 6-tuple $(X,Y,\eta,\ell,H,D)$, where $H$ is a collection of functions $h:X\to D$ and $\ell$ is a function $\ell:D\times Y \to \mathbb R_{\geq 0}$. Definition~\ref{def:problem} would then cover the special case $D = Y$. The more general definition would, for instance, encompass class probability estimation problems by taking $D = \prob(Y)$. Many of the results in this paper would generalize to this more general setting. We leave this investigation to future work.}
       
        We require each $h\in H$ to exhibit finite expected loss, meaning
        \[\int_{X\times Y}\ell(h(x),y)\, \eta(dx{\times} dy) < \infty.\]
    \end{itemize}
\end{definition}
Notice that our definition does not explicitly mention the random variables in the input and output spaces. Only the most crucial information about the variables, the joint law $\eta$ that couples them, is recorded. This strengthens the analogy with metric geometry; a supervised learning problem is analogous to a metric measure space, where the metric corresponds to the loss function and the measure corresponds to the joint law.

The ``expected loss'' mentioned in Definition~\ref{def:problem} is more commonly known as the \emph{risk} of a predictor.

\begin{definition}
    Given a problem $P = (X,Y,\eta,\ell, H)$, the \emph{risk} of a predictor $h\in H$ is defined to be
    \[\cR_P(h) := \int \ell(h(x),y)\, \eta(dx\times dy).\]
    The \emph{constrained Bayes risk of $P$} is defined to be
    \[\cB(P) := \inf_{h\in H} \cR_P(h) = \inf_{h\in H} \int \ell(h(x),y)\, \eta(dx\times dy).\]
\end{definition}
The risk of a predictor is its expected loss on a randomly selected observation. If we measure the performance of a predictor by its risk, then the constrained Bayes risk of a problem represents the optimal possible performance when selecting a predictor from $H$. The constrained Bayes risk is so called by \citet{williamson_information_2024} since it is a modification of the classical notion of the Bayes risk of a predictor. Note that if the predictor set for a problem $P$ is the set $H_{X,Y}$ of all measurable functions $f:X\to Y$, then $\cB(P)$ agrees with the classical Bayes risk, which we will henceforth denote as 
\[\cB^*(P) := \inf_{h\in H_{X,Y}} \cR_P(h).\]

We emphasize that a problem does not include a method for \emph{choosing} a predictor. We think of a problem as something that a practitioner must \emph{solve} by selecting, through some algorithmic procedure, a function from the predictor set that performs well according to some useful functional built on the loss function, such as the risk $\cR_P:H\to \mathbb{R}_{\geq 0}$. In this regard, in Section \ref{sec:topological} we consider a certain compressed representation of $\cR_P:H\to \mathbb{R}_{\geq 0}$ which permits gaining insight into the complexity associated to this selection task.

\begin{notation}
To avoid tediously re-writing these 5-tuples, we establish some persistent notation. We assume, unless otherwise stated, that a problem with the name $P$ has components $(X,Y,\eta,\ell,H)$, a problem named $P'$ has components $(X',Y',\eta',\ell',H')$, and a problem named $P''$ has components $(X'',Y'',\eta'',\ell'',H'')$. Similarly, for any integer $i\geq 0$, the problem $P_i$ has components $(X_i,Y_i,\eta_i,\ell_i,H_i)$.
\end{notation}
\begin{notation}
    Given any problem $P = (X,Y,\eta,\ell,H)$, we use the shorthand $\ell_h$ for the function $\ell_h:X\times Y \to \bbR_{\geq 0}$ given by $\ell_h(x,y):= \ell(h(x),y)$.
\end{notation}

\begin{examples}\label{ex:linear-classification}
Consider the problem of predicting whether a college student will successfully graduate given their high school GPA and ACT scores. We can model this problem as follows. We represent each student with a point in $\bbR^2 \times \{0,1\}$, where the coordinates in $\bbR^2$ are determined by the student's GPA and ACT score, and their graduation status is represented with a $0$ for failure and $1$ for success. We think of students as being sampled from an unknown probability measure $\eta$ on $\bbR^2 \times \{0,1\}$. Our aim is then to select a predictor $\bbR^2\to \{0,1\}$. We might restrict ourselves to predictors with a linear decision boundary:
\begin{align*}
    h_{a,b,c}(x_1,x_2) := \begin{cases}
    1 & ax_1 + bx_2 \geq c \\
    0 & ax_1 + bx_2 < c
    \end{cases}.
\end{align*}
We may also decide that all incorrect predictions carry the same penalty. Any incorrect guess incurs a loss of 1, while a correct guess has loss 0.

We have now constructed a problem
\[P := (\bbR^2,\{0,1\},\eta,\ell,H),\]
where $\ell(y,y') = 1_{y\neq y'}$, and $H = \{h_{a,b,c}\big| a,b,c\in \bbR\}$.

The risk of an arbitrary predictor $h_{a,b,c}$ is then
\begin{align*}
    \cR_P(h_{a,b,c}) & = \int \ell(h_{a,b,c}(x),y) \,\eta(dx{\times}dy)
    \\ & =
    \int_{\bbR^2 \times \{0\}} h_{a,b,c}(x) \, \eta(dx{\times}dy)
    +
    \int_{\bbR^2 \times \{1\}} (1-h_{a,b,c}(x)) \, \eta(dx{\times}dy)
    \\ & =
    \eta\big(\set{x\in \mathbb R^2}{h_{a,b,c}(x)=1} \times \{0\}\big)
    +\eta\big(\set{x\in \mathbb R^2}{h_{a,b,c}(x)=0} \times \{1\}\big).
\end{align*}
That is, the risk of $h_{a,b,c}$ is its misclassification probability. The constrained Bayes risk $\cB(P)$ is then the infimal misclassification probability across all predictors with linear decision boundaries. Under additional normality conditions, the predictor achieving the risk $\cB(P)$ can be found with the Fisher linear discriminant \citep{fisher_use_1936}.

\end{examples}

\begin{examples}\label{ex:linear-regression}
Suppose we are faced with a linear regression problem. That is, we have a random vector $\mathsf X$ in $\mathbb R^n$ and a random number $\mathsf Y$ in $\mathbb R$ and are tasked with finding an affine linear functional $\bbR^n \to \bbR$ that most closely resembles their joint law. We can formalize this problem as follows. We let $\eta \in \prob(
\bbR^n \times \bbR)$ be the unknown joint law. We may select the squared difference $\ell(y,y') := (y-y')^2$ as our loss function.
The resulting problem is then
\begin{align*}
    P := (\bbR^n,\bbR,\eta,\ell,H)
\end{align*}
where $H$ is the collection of affine linear functions $\bbR^n \to \bbR$.

The risk of an arbitrary predictor $h(x)=a^Tx+b$ is the mean squared error of $h$:
\begin{align*}
    \cR_P(h) = \int \ell(h(x),y)\,\eta(dx{\times}dy) & = \int (a^Tx+b-y)^2\, \eta(dx{\times} dy).
\end{align*}
The constrained Bayes risk $\cB(P)$ is then the infimal mean squared error across all linear functionals $h\in H$. 

\end{examples}

\begin{examples}\label{ex:one-point-problem}
    By making trivial choices for each component, one can define the \emph{one-point problems} to be
    \[P_\bullet(c) := (\{\bullet\},\{\bullet\},\delta_{(\bullet,\bullet)}, c, \{\id_{\bullet}\})\]
    for any $c \geq 0$. We will most often set $c=0$ and use the notation $P_\bullet := P_\bullet(0)$.
    The constrained Bayes risk $\cB(P_\bullet(c))$ is easily seen to be $c$, since the loss function of $P_\bullet(c)$ is identically $c$.
\end{examples}

\subsection{Comparison with Blackwell's Statistical Experiments}\label{subsec:statistical-experiments}
    Our definition of a supervised learning problem may seem superficially similar to the established notion of a statistical experiment which is prevalent in mathematical statistics. Since there already exist established quantitative methods for comparing statistical experiments such as Le Cam's deficiency distance \citep{leCam_sufficiency, mariucci2016le}, let us take a moment to contrast statistical experiments and supervised learning problems. Originally defined by Blackwell \citep{blackwell_comparison_1951}, a statistical experiment consists of
    \begin{itemize}
    \item a measurable space $X$,
    \item a set of probability measures $\{P_\theta,\theta\in \Theta\}$ indexed by some set $\Theta$, and
    \item a distinguished $\theta\in \Theta$, thought to be the ``true'' parameter, or the parameter ``chosen by nature''.
    \end{itemize}
    A practitioner is then tasked with estimating the true value of $\theta$ and hence the underlying probability measure.

    Supervised learning can be placed into Blackwell's framework \citep[c.f.][]{williamson_information_2024} by taking $X$ to be the input space of the problem, taking $\Theta$ to be the response space $Y$, and letting the map $y \mapsto P_y$ be the disintegration of the joint law $\eta \in \prob(X\times Y)$ along $Y$.\footnote{In the context of classification, the measure $P_y$ is called the \emph{class-conditional distribution}.} When predicting the label for a particular observation $x$, the ``true'' value $y_0$ represents the unknown true response, and we assume the observed input $x$ was drawn from the probability measure $P_{y_0}$. One's goal, of course, is to then estimate $y_0$ using an observed $x\in X$. A process for doing so defines a predictor $X\to Y$. Hence the superivised learning problem is encoded as a statistical experiment $E := (X,\set{P_y}{y\in Y}, y_0)$.

    In contrast, our supervised learning problems have three additional pieces of information: 
    \begin{itemize}
    \item[(1)] the $Y$-marginal of $\eta$, 
    \item[(2)] a loss function $\ell$ on the response space $Y$, and 
    \item[(3)] an explicit family $H$ of functions $X\to Y$. 
    \end{itemize}
    Indeed, using the $Y$-marginal of $\eta$ and the disintegration $y\mapsto P_{y}$, we can recover the joint probability measure $\eta \in \prob(X\times Y)$ via reverse-disintegration (see Section~\ref{subsec:markov-kernels} for details). Hence from the statistical experiment and the listed additional information, one can recover all five components of the supervised learning problem $(X,Y,\eta,\ell,H)$.

\iffalse
    \bob{I discussed this with my postdoc Nan Lu and she observed that in supervised learning problems the label $Y$ plays the role of $\theta$ in the comparison of experiments. This is indeed the case if you look at my information processing equalities paper. But there is a clear difference between how the comparison of experiments literature sets things up compare to what you do here, and that is you are taking a particular loss function, whereas the deficiency distance and results such as the Blackwell-Sherman-Stein theorem reply on taking a sup over all loss functions. I did not want to just silently edit the text above, but I think it should be tweaked to concentrate primarily on the fact that we work with a specific loss function. It would be nice to add "Thanks to Nan Lu for comments on a draft." in the acknowledgements too.} \brantley{Thanks for the comments. I added a thank you to Dr. Lu in the acknowledgements. I gave some thought to the differences between our supervised learning problems and the encoding as a statistical experiment that you described.  I think that in addition to the loss function $\ell$, our supervised learning problems also have the extra data of a prior on $Y$ and an explicitly defined class $H$ of functions $X\to Y$. I've redone the above two paragraphs to emphasize this.}
\fi

\subsection{Motivating Properties}\label{subsec:motivation}
Our goal is to quantitatively compare problems by developing a notion of distance on the collection of all problems.
This notion will take the form of a pseudometric $\dexp$ which we will call the \emph{Risk distance}.
To  design such a distance, we might begin by deciding on some desirable properties.
When we introduce $\dexp$, we will see that it is characterized as the largest pseudometric satisfying these conditions.

The first property that one needs to decide when designing a pseudometric is what its zero set should be. In other words one needs to postulate a notion of \emph{isomorphism} that we expect the notion of distance to be compatible with. We start by introducing one such notion.
\begin{definition}\label{def:strong-iso} Two problems $P$ and $P'$ are said to be \emph{strongly isomorphic}  if there exist measurable bijections
    \[f_1:X'\to X \qquad f_2:Y'\to Y\]
    and a bijection between predictor sets
    \[\varphi: H'\to H\]
    such that the following conditions are satisfied:
    \begin{enumerate}
        \item $(f_1\times f_2)_\sharp \eta' = \eta$, 
        \item For every $h'\in H'$, the diagram
        \[
        \begin{tikzcd}
            X'\times Y'\ar[dd,"f_1\times f_2",swap] \ar[dr,"\ell'_{h'}"]
            \\
            & \bbR
            \\
            X\times Y \ar[ur,"\ell_{\varphi(h')}", swap]
        \end{tikzcd}    
        \] 
        commutes. That is, $\ell'_{h'}= \ell_{\varphi(h')} \circ (f_1\times f_2)$. 
    \end{enumerate}
 \end{definition} 
Here the notation $f_1\times f_2$ denotes the product map: $(f_1\times f_2)(x',y') := (f_1(x'), f_2(y'))$. (We choose the notation $f_1\times f_2$ over the alternative $(f_1,f_2)$ because the former emphasizes that each component of the output depends only on the corresponding component of the input.) In other words, a strong isomorhism consists of bijections between the components of a problem that preserve the structure of the problem, and in particular preserve the performance of every predictor; if $f_1(x')=x$, $f_2(y')=y$ and $\phi(h')=h$, then $\ell'_{h'}(x',y') = \ell_h(x,y)$. Despite being natural, the notion of strong isomorphism is rather rigid.  For instance, Condition~2 requires that an equality involving the predictors and loss hold for every point in $X'\times Y'$, even those outside of the support of $\eta'$. Such points could reasonably be considered irrelevant to the supervised learning problem, since an observation will appear there with probability 0. In order to motivate a more flexible notion of isomorphism which characterizes the pairs of problems that should be distance zero apart, we collect potential modifications to problems that one might reasonably consider immaterial.

\begin{examples}\label{ex:linear-regression-formulations}
    Consider the linear regression problem $P = (\mathbb R^n, \mathbb R, \eta,\ell, H)$ from Example~\ref{ex:linear-regression}. We chose $H$ to be the set of all affine linear functionals $\mathbb R^n \to \mathbb R$. As is often done, we could alternatively formulate the linear regression problem as follows. We place our data in $\mathbb R^{n+1}$ by setting the first coordinate of every observation to a constant 1. Instead of an affine linear map, our task is now to select a (non-affine) linear map $\mathbb R^{n+1} \to \mathbb R$. Formally, we construct a problem
    \[P' := (\mathbb R^{n+1}, \mathbb R, i_\sharp \eta, \ell, H')\]
    where
    \begin{itemize}
        \item $i:\mathbb R^n \to \mathbb R^{n+1}$ is given by $i(x_1,\dots, x_n) := (1,x_1,\dots, x_n)$, and
        \item $H'$ is the set of all (non-affine) linear functionals $\mathbb R^{n+1} \to \mathbb R$.
    \end{itemize}
    The problems $P$ and $P'$ represent two common formulations of linear regression, the latter stated in terms of linear functionals, the former in terms of affine linear functionals. While $P$ and $P'$ are not strongly isomorphic, the formulations they represent are considered equivalent. Whatever our notion of distance between problems, it is desirable that $P$ and $P'$ be distance zero apart in order to reflect this equivalence.
\end{examples}

\begin{examples}\label{ex:extra-label}
    Consider the problem
    \[P := (\mathbb R, \{a,b\}, \eta, \ell, H),\]
    where
    \begin{itemize}
        \item $\eta$ is any probability measure on $\mathbb R \times \{a,b\}$.
        \item $\ell$ is the 0-1 loss (also called the discrete metric) on $\{a,b\}$. Namely, $\ell(a,a)=\ell(b,b)=0$ and $\ell(a,b)=\ell(b,a)=1$.
        \item $H$ consists of functions $\mathbb R \to \{a,b\}$ with finitely many discontinuities.
    \end{itemize}
    This $P$ represents a classification problem; a predictor $h\in H$ assigns to each input either the label $a$ or the label $b$.
    Now construct a second problem $P'$ by adding to $P$ a third label, say, $c$. We design the remaining components of $P'$ so as to make the labels $b$ and $c$ functionally identical. More precisely, we define
    \[P' := (\mathbb R, \{a,b,c\}, \eta', \ell', H')\]
    where
    \begin{itemize}
        \item $\eta'$ is any measure such that $\eta$ is the pushforward of $\eta'$ under the map $\{a,b,c\}\to \{a,b\}$ sending $a\mapsto a$ and $b,c\mapsto b$.
        \item $\ell'$ extends $\ell$ by setting
        \[\ell'(b,c) = \ell'(c,b) = \ell'(c,c) = 0\]
        \[\ell'(a,c)=\ell'(c,a) = 1.\]
        That is, $\ell'$ considers $b$ and $c$ to be indistinguishable.
        \item $H'$ consists of all functions $\mathbb R \to \{a,b,c\}$ with finitely many discontinuities.
    \end{itemize}
    
    The problems $P$ and $P'$ are both classification problems. By design, $P'$ is $P$ with an extraneous label. If we solve $P$ by choosing a function $h\in H$, the same $h$ will provide a solution to $P'$ with identical performance. Conversely, if we solve $P'$ by selecting some $h'\in H'$, we can produce an identically-performing solution for $P$ by composing $h'$ with the map sending $a\mapsto a$ and $b,c\mapsto b$. While $P$ and $P'$ are not strongly isomorphic, they differ only superficially, and a meaningful distance should set $P$ and $P'$ to be distance zero apart.
\end{examples}
    
Examples~\ref{ex:linear-regression-formulations} and \ref{ex:extra-label} demonstrate that any of the five components of a problem can change without modifying the problem's essential nature. This motivates the following definition which arises as a relaxation of Definition \ref{def:strong-iso}.

\begin{definition}
    Given two problems $P$ and $P'$, we say that $P'$ is a \emph{simulation} of $P$ (or that $P'$ \emph{simulates} $P$), and denote it by $P'\rightsquigarrow P$, if there exist measurable functions
    \[f_1:X'\to X \qquad f_2:Y'\to Y\]
    such that the following are satisfied.
    \begin{enumerate}
        \item $(f_1\times f_2)_\sharp \eta' = \eta$,       
        \item For every $h\in H$, there exists an $h'\in H'$ such that the diagram
        \[
        \begin{tikzcd}
            X'\times Y'\ar[dd,"f_1\times f_2",swap] \ar[dr,"\ell'_{h'}"]
            \\
            & \bbR
            \\
            X\times Y \ar[ur,"\ell_h", swap]
        \end{tikzcd}    
        \]
        commutes (that is, $\ell'_{h'}= \ell_h \circ (f_1\times f_2)$) $\eta'$-almost everywhere.
        Similarly, for every $h' \in H'$, there exists an $h\in H$ such that the above diagram commutes $\eta'$-almost everywhere.
    \end{enumerate}
\end{definition}

Note that if $P$ and $P'$ are strongly isomorphic then $P$ and $P'$ are simulations of each other. It is, however, clear that the notion of simulation is weaker than strong isomorphism. Indeed, the former does not require the maps $f_1$ and $f_2$ to be bijections, nor does it require   Condition~2 to be satisfied through a bijection between $H$ and $H'$. Furthermore, Condition~2 encodes an additional, more subtle, relaxation of the second condition of strong isomorphism: whereas the latter requires the diagram to commute everywhere on $X'\times Y'$, the former only requires this to take place $\eta'$-almost everywhere within $X'\times Y'$. In Example~\ref{ex:linear-regression-formulations}, the problem $P'$ is a simulation of $P$ via the projection map $\mathbb R^{n+1} \to \mathbb R^n$ that drops the first coordinate, and the identity map on the output space $\mathbb R$. Likewise, in Example~\ref{ex:extra-label}, $P'$ is a simulation of $P$ via the identity map on the input space $\mathbb R$ and the map on the output space sending $a\mapsto a$ and $b,c\mapsto b$.

The term ``simulation'' is meant to suggest that when a problem $P'$ is a simulation of $P$, then $P'$ is, in a sense,  richer than $P$. The terminology is inspired by similar concepts appearing in the study of automata and Markov chains \citep{larsen1989bisimulation, milner1980calculus,milner1989communication}.

Motivated by Examples~\ref{ex:linear-regression-formulations} and \ref{ex:extra-label}, we posit that if $P'$ is a simulation of $P$, then it is reasonable to consider $P$ and $P'$ to be functionally identical. As these examples demonstrate, if $P'$ simulates $P$, then $P'$ can be solved ``in terms of $P$'' using the maps $f_1$ and $f_2$ which witness the simulation. That is, one could take an i.i.d. sample $\{(x'_i,y'_i)\}_{i=1}^n$ drawn from $\eta'$ and apply the maps $f_1$ and $f_2$ to the $x_i$'s and $y_i$'s respectively to get a sample $\{(x_i,y_i)\}_{i=1}^n$ in $X\times Y$. The pushforward condition in the definition of a simulation guarantees that the transformed sample will be an i.i.d. sample drawn from $\eta$. One could then solve $P$ by using this data to select an appropriate predictor $h\in H$, then solve $P'$ with a corresponding $h' \in H'$. Point~2 in the simulation definition implies that $\ell'(h'(x'_i),y'_i) = \ell(h(x_i), y_i)$. In other words, we are guaranteed that $h$ and $h'$ have identical performance on corresponding observations.
For instance, both problems have the same constrained Bayes risk. Indeed, if $P'$ is a simulation of $P$ via the maps $f_1:X\to X'$ and $f_2:Y \to Y'$, then
\begin{align*}
    \cB(P') &= \inf_{h'\in H'} \int \ell'_{h'}(x',y')\, \eta'(dx'{\times}dy')
    \\ & =
    \inf_{h\in H} \int \ell_h(f_1(x'),f_2(y'))\, \eta'(dx'{\times}dy')
    \\ & =
    \inf_{h\in H} \int \ell_h(x, y)\, ((f_1{\times} f_2)_\sharp\eta')(dx{\times}dy)
    \\ & =
    \inf_{h\in H} \int \ell_h(x, y)\, \eta(dx{\times}dy) = \cB(P).
\end{align*}
This sort of relationship between problems that allows one to solve one problem in terms of another is reminiscent of John Langford's ``Machine Learning Reductions'' research program \citep{Langford:2009aa}.

If $P'$ is a simulation of $P$, we would like our distance $\dexp$ to consider them identical.
\begin{tcolorbox}[colback=white]
    \textbf{Condition 1:} If $P'$ is a simulation of $P$, then $\dexp(P,P')=0$.
\end{tcolorbox}

Symmetrizing the simulation relationship produces a more general relationship which we call \emph{weak isomorphism}, modeled after \emph{weak isomorphism of networks}, defined by \citet{chowdhury_gromovwasserstein_2019,chowdhury2023distances}.

\begin{definition}\label{def:weak-iso}
    Define two problems $P$ and $P'$ to be \emph{weakly isomorphic} if there exists a third problem $\widetilde P$ that is a simulation of both $P$ and $P'$. 
\end{definition} 
In other words, weak isomorphism between problems $P$ and $P'$ arises from the existence of a simultaneous simulation of both problems, as illustrated  via the diagram\footnote{
Simulations, weak isomorphisms and their connections to the Risk distance suggest the existence of a ``category of problems'' in which simulations are morphisms, weak isomorphisms are spans, and from which the Risk distance can be derived. We leave the exploration of such a categorical perspective to future work.}

    \[\begin{tikzcd}
    & \widetilde P \ar[dl,rightsquigarrow] \ar[dr,rightsquigarrow]\\
    P  & & P'.\\
    \end{tikzcd} 
    \]

It is clear that strong isomorphism implies weak isomorphism. As we will see below (Theorem~\ref{thm:dexp-pseudometric} and Proposition~\ref{prop:drisk-zero-compact}), under fairly general conditions, the Risk distance $\dexp(P,P')$ will vanish if and only if $P$ and $P'$ are weakly isomorphic.

The second desirable condition for $\dexp$ is simpler to state. Let $P$ be a problem, and suppose that $P'$ is a problem with all the same components, except possibly for the loss function $\ell'$. If there exists $\alpha\geq 0$ such that, for all $h\in H$, we have
\[\int_{X\times Y} \big|\ell(h(x),y) - \ell'(h(x),y)\Big| \, \eta(dx\times dy) \leq \alpha,\]
then we demand $\dexp(P,P')\leq \alpha$.
\begin{tcolorbox}[colback=white]
    \textbf{Condition 2:} If
    \[P = (X,Y,\eta,\ell,H)\]
    \[P' = (X,Y,\eta,\ell',H),\]
    then
    \[\dexp(P,P') \leq \sup_{h\in H} \int_{X\times Y}\Big|\ell(h(x),y) - \ell'(h(x),y)\Big| \,\eta(dx\times dy).\]
\end{tcolorbox}

\subsection{The Risk Distance}

The following is a pseudometric that satisfies Conditions 1 and 2.
\begin{definition}\label{def:metric}
    Let $P$ and $P'$ be problems. For any coupling $\gamma \in \Pi(\eta,\eta')$ and correspondence $R\in \cC(H,H')$, define the \emph{risk distortion} of $\gamma$ and $R$ to be
    \[\dis_{P,P'}(R,\gamma):= \sup_{(h,h')\in R}
    \int \big|\ell(h(x),y) - \ell'(h'(x'),y')\big|
    \, \gamma(dx{\times} dy {\times} dx' {\times} dy').\]
    The \emph{Risk distance} between $P$ and $P'$ is then given by
    \[\dexp (P, P'):=\inf_{\substack{\gamma\in \Pi(\eta,\eta')\\R\in \cC(H,H')}} \dis_{P,P'}(R,\gamma).\]
\end{definition}

The risk distortion (and subsequently the Risk distance) is so named because of its similarity to the risk of a predictor.
If one believes that the most salient numerical descriptor of a predictor is the loss that it incurs, then it is natural to describe a predictor by its loss on an average observation. It is similarly natural to compare predictors by comparing their respective losses for an average observation.

The idea behind the Risk distance is that we should consider two problems $P$ and $P'$ to be similar if we can find a correspondence between the predictor sets and a coupling between the underlying spaces such that pairs of corresponding predictors \emph{incur similar losses} on corresponding observations.

Notably, the Risk distortion and Risk distance are not invariant to scaling and shifting of the loss functions. While we often think of loss substitutions such as $\ell(y,y')\mapsto a\ell(y,y') + b$ for $a,b\in \mathbb R$, $a>0$ as immaterial modifications in supervised learning, such changes do, in some sense, affect the problem. For instance, such a substitution will change the Bayes risk of the problem, so any notion of distance which controls the Bayes Risk (see Theorem~\ref{thm:loss-control} below) cannot be invariant to such modifications. In order to use the Risk distance on specific problems, one must select loss functions which are meaningful in their specific context. For instance, one may elect to ``normalize'' their loss functions by subtracting the Bayes risk $\cB(P)$ or the minimal possible loss $\min_{y,y'\in Y}\ell(y,y')$, or by scaling the loss to have meaningful units such as bits or nats. Our formalism supports such decisions.

By design, the Risk distance is not a proper metric, since Condition 1 implies that distinct problems could be distance zero apart. This is desirable behavior. Reformulating a problem in an equivalent way as in Example~\ref{ex:linear-regression-formulations} or adding extraneous labels as in Example~\ref{ex:extra-label} should not move a problem at all under the Risk distance since such modifications do not meaningfully change the problem.

\begin{theorem}\label{thm:dexp-pseudometric}
    The function $\dexp$ is a pseudometric on the collection $\mathcal{P}$ of all problems. Furthermore, $\dexp(P,P')=0$ whenever $P$ and $P'$ are weakly isomorphic.
\end{theorem}

That $\dexp$ is pseudometric entails that $\dexp$ vanishes on the diagonal, that is $\dexp(P,P)=0$ for all problems $P$, satisfies symmetry, that is $\dexp(P,P')=\dexp(P',P)$ for all problems $P$ and $P'$, and satisfies the triangle inequality: $\dexp(P,P'')\leq \dexp(P,P')+\dexp(P',P'')$ for all problems $P,P'$ and $P''$. 

That $\dexp$ satisfies the triangle inequality is nontrivial to prove, and the proof is relegated to Appendix~\ref{appendix:proofs}. Not only does $\dexp$ satisfy Conditions 1 and 2, it is characterized by them in the sense that it is \emph{uniquely determined by these conditions}.
\begin{theorem}\label{thm:characterization}
    The Risk distance satisfies Conditions~1 and 2. That is,
    \begin{enumerate}
        \item If $P'$ is a simulation of $P$, then $\dexp(P,P')=0$.
        \item For any problems $P=(X,Y,\eta,\ell,H)$ and $P'=(X,Y,\eta,\ell',H)$ which differ only in their loss functions, we have \[\dexp(P,P') \leq \sup_{h\in H} \int_{X\times Y}\Big|\ell(h(x),y) - \ell'(h(x),y)\Big| \,\eta(dx\times dy).\]
    \end{enumerate}
    Furthermore, the Risk distance is the largest pseudometric simultaneously satisfying Condition~1 and Condition~2.
\end{theorem}
Note that the claim from Theorem \ref{thm:dexp-pseudometric} that $\dexp(P,P')=0$ whenever $P$
 and $P'$ are weakly isomorphic can now be obtained from Theorem \ref{thm:characterization} as follows. By Definition \ref{def:weak-iso}, there exists a problem  $P''$ which is is joint simulation of both $P$ and $P'$. The claim can now be obtained through  the facts that $\dexp$ satisfies both the triangle inequality and  Condition 1 so that
$$\dexp(P,P')
 \leq \dexp(P,P'')+\dexp(P',P'') = 0. $$

The statement of Theorem~\ref{thm:characterization} implicitly asserts uniqueness of the largest such pseudometric. Indeed, if $d_1$ and $d_2$ were two different pseudometrics satisfying Conditions~1 and 2, their pointwise maximum $d':=\max(d_1,d_2)$ would also satisfy Conditions 1 and 2, would be a pseudometric, but would be larger than $d_1$ and $d_2$.\footnote{This would follow from the assumption that there exist $P,P'\in\mathcal{P}$ such that $d_1(P,P')\neq d_2(P,P').$} The proof of this theorem can also be found in Appendix~\ref{appendix:proofs}. 

While Theorem~\ref{thm:dexp-pseudometric} states that weakly isomorphic problems are distance zero apart, we can say more; weak isomorphism is equivalent to the existence of a coupling and correspondence between the problems with zero risk distortion.
\begin{proposition}\label{prop:weak-iso-optimal}
    Let $P$ and $P'$ be problems. Then $P\cong P'$ if and only if there exist a correspondence $R\in \cC(H,H')$ and a coupling $\gamma \in \Pi(\eta,\eta')$ such that \[\dis_{P,P'}(R,\gamma)=0.\]
\end{proposition}

The proof is in Appendix~\ref{appendix:proofs}. Note that if $\dis_{P,P'}(R,\gamma)=0$ then $\dexp(P,P')=0$ by definition of the Risk distance, but the converse does not necessarily hold; it is possible that the infimal risk distortion is not achieved by any particular coupling and correspondence. Hence weak isomorphism between problems $P$ and $P'$ is generally stronger than saying $\dexp(P,P')=0$. We will discuss the existence of such ``optimal'' couplings and correspondences in Section~\ref{subsec:optimal-couplings}, where we will prove that, under certain conditions, weak isomorphism is equivalent to identification under the Risk distance. While the statement (Corollary~\ref{cor:compact-weak-isomorphism-equivalence}) is in terms of definitions we have not yet established, it is equivalent to the following:
\begin{proposition}\label{prop:drisk-zero-compact}
    Let $P = (X,Y,\eta,\ell,H)$ and $P' = (X',Y',\eta',\ell',H')$ be problems. Suppose that for all $h\in H$, $h'\in H'$, the functions $\ell_h$ and $\ell'_h$ are continuous almost everywhere (with respect to $\eta$ and $\eta'$ respectively), and the sets of functions
    \[\set{\ell_h}{h\in H}\subseteq L^1(\eta)\quad \text{and} \quad \set{\ell'_{h'}}{h'\in H'}\subseteq L^1(\eta')\] are compact under their respective $L^1$ norms. Then $\dexp(P,P')=0$ if and only if $P \cong P'$.
\end{proposition}

The continuity condition is satisfied by a typical regression problem, such as Example~\ref{ex:linear-regression}, since the predictors and loss for such a problem tend to be continuous. A classification problem can satisfy the continuity condition as well under some assumptions on $\eta$. In particular, we need only assume that a random observation falls on a decision boundary with probability zero. We will discuss the continuity condition in more detail in Section~\ref{subsec:optimal-couplings}.

The compactness condition is implied if one assumes that a problem's predictor set $H$ is parameterized by some compact parameter space $\Theta \to H$ and that the composition $\Theta \to H\to L^1(\eta)$ is continuous.
Consider, for example, the problem $P = (\bbR, \{0,1\}, \eta, \ell, H)$ where $\ell$ is the 0-1 loss and $H$ is the set of indicators $1_{(a,\infty)}$ for $a$ in the compact interval $[-\infty,\infty]$. That is, $P$ is a binary classification problem for which the predictors are given by hard cutoff values. Under modest assumptions on $\eta$ (namely that its first marginal is absolutely continuous with respect to the Lebesgue measure), the composition $[-\infty,\infty] \to H\to L^1(\eta)$ is continuous, and hence $P$ satisfies the compactness condition. Under the same assumptions, we also have that each $\ell_h$ is $\eta$-almost everywhere continuous. Hence $P$ satisfies both conditions of Proposition~\ref{prop:drisk-zero-compact}.

While classification problems are more likely to satisfy the compactness condition of Proposition~\ref{prop:drisk-zero-compact}, regression problems can be made to do so as well. Consider the linear regression problem in Example~\ref{ex:linear-regression}. The predictor set for that problem is parameterized by the non-compact space $\bbR^{n+1}$. However, it can be made compact with regularization. Indeed, regularized regression methods like ridge regression and lasso can be seen as restrictions of the parameter space to a compact region $A\subset \bbR^{n+1}$ \citep[Sec 6.2]{james_introduction_2021}. Linear regression problems with such restrictions will satisfy both the continuity and compactness conditions.

Adding a constant $\alpha>0$ to the loss function of a problem moves it a distance of $\alpha$ under the Risk distance. That is, if $P$ is a problem with loss function $\ell$, and $P'$ is the same problem but with loss function $\ell+\alpha$, then $\dexp(P,P')=\alpha$. This may appear to be undesirable behavior, since the losses $\ell$ and $\ell+\alpha$ are practically identical from an optimization point of view. However, adding the constant $\alpha$ to the loss function also increases the constrained Bayes risk of the problem by $\alpha$. From this perspective, it is reasonable to expect the distance between the two problems to be $\alpha$.

\begin{examples}\label{ex:bayes-one-point}
    Let $P$ be an arbitrary problem and $P_\bullet = P_\bullet(0)$ be the one-point problem of Example~\ref{ex:one-point-problem}. There is only one correspondence in $\cC(H,\{\id_\bullet\})$ and one coupling in $\Pi(\eta,\delta_{(\bullet,\bullet)})$, so one can compute
    \begin{align*}
        \dexp(P,P_\bullet) & =
        \dis_{P,P_\bullet}\left(H \times \{\id_{\{\bullet\}}\}, \eta \otimes \delta_{(\bullet,\bullet)}\right)
        \\ & =
        \sup_{h\in H}\int \ell(h(x),y) \, \eta(dx\times dy).
    \end{align*}
    Note the similarity with the constrained Bayes risk $\mathcal B(P)$. While $\mathcal B(P)$ is equal to the risk of an \emph{optimal} predictor, the distance $\dexp(P,P_\bullet)$ is equal to the risk of a \emph{pessimal} predictor.

    If we additionally assume that the loss $\ell$ is bounded above by a constant $\ell_{\max}$, one can similarly compute the distance from $P$ to the one-point problem $P_\bullet(\ell_{\max})$ to be
    \begin{align*}
        \dexp(P,P_\bullet(\ell_{\max})) & =
        \sup_{h\in H}\int \big(\ell_{\max} - \ell(h(x),y) \big) \, \eta(dx\times dy)
        \\ & =
        \ell_{\max} - \inf_{h\in H}\int \ell(h(x),y) \, \eta(dx\times dy)
        \\ & = \ell_{\max} - \cB(P).
    \end{align*}
    Hence $\cB(P) = \ell_{\max} - \dexp(P,P(\ell_{\max}))$, showing that the constrained Bayes risk can be written in terms of the Risk distance if the loss is bounded.
\end{examples}

\begin{remark}
The Risk distance is not designed to be explicitly computed for arbitrary pairs of problems. Much like the Gromov-Hausdorff distance, its usefulness lies in the stability results which it facilitates. However, for completeness, we briefly remark on computational considerations. Given that the computation of the structurally similar but seemingly simpler Gromov-Hausdorff distance is known to be NP-hard \citep{agarwal_computing_2015,schmiedl_computational_2017,memoli_gromov-hausdorff_2021}, 
we expect that estimation of the exact Risk distance for problems with finite input and output spaces is also NP-hard in general.  In this direction, we will soon prove, as Corollary~\ref{cor:gi-hard}, that deciding whether $\dexp(P,P')=0$ for arbitrary problems $P$ and $P'$ with finite input and output spaces is at least as hard as graph isomorphism. 
\end{remark}

In later sections, we will also explore two variants of the Risk distance that are constructed by tweaking the treatment of the correspondences $\cC(H,H')$ in Definition~\ref{def:metric}. First, by endowing the predictor sets $H$ and $H'$ with probability measures $\lambda$ and $\lambda'$, one can replace $\cC(H,H')$ with $\Pi(\lambda,\lambda')$ and the supremum over $R$ with an $L^p$-style integral over a coupling. This gives rise to the \emph{$L^p$-Risk distance} introduced in Section~\ref{sec:probabilistic}. Alternatively, by restricting the set of possible correspondences $\cC(H,H')$ to the subset $\con(H,H')$ consisting only of those satisfying a certain topological condition, we arrive at the \emph{Connected Risk distance} of Section~\ref{sec:topological}.

\subsection{Connections to the Gromov-Wasserstein Distance}\label{subsec:OT-connection}
    The Hausdorff, Gromov-Hausdorff, Wasserstein, and Gromov-Wasserstein distances all follow a common paradigm. One first considers the space of all ``alignments'' between the objects in question, whether ``alignment'' means a correspondence between sets or a coupling between measures. One then settles on a way to numerically score the distortion of a correspondence or coupling, whether through a supremum or integral. The distance between the objects is then the infimal possible score. The Risk distance follows the same paradigm. An ``alignment'' of two problems consists of a correspondence between their predictor sets and a coupling between their joint laws. We use the loss function to numerically score the coupling and correspondence, then infimize over all choices.

    Indeed, if we think of a loss function as a notion of distance, then the definition of $\dexp$ most closely resembles that of the Gromov-Wasserstein distance. The only extra component is a supremum over the predictor correspondence. To make the comparison more explicit, suppose $(X,d_X,\mu_X)$ and $(X',d_{X'},\mu_{X'})$ are metric measure spaces. We can construct the following problems:
    \begin{align*}
    P & := (X,X,\mu_X\otimes \mu_X, d_X, \{\id_X\})\\
    P' & := (X',X',\mu_{X'}\otimes \mu_{X'}, d_{X'}, \{\id_{X'}\}).
    \end{align*}
    We then write
    \begin{align*}
    \dexp(P,P') = \inf_{\gamma \in \Pi(\mu_{X}\otimes \mu_{X},\mu_{X'}\otimes \mu_{X'})} \int \big|d_X(x,y) - d_{X'}(x',y')\big|\, \gamma(dx{\times} dy {\times} dx'{\times} dy').
    \end{align*}
    Here $\otimes$ represents the product of measures.
    This is exactly twice the Gromov-Wasserstein distance (Definition~\ref{def:gw}) except for one difference: the measure $\gamma$. The infimum runs over all of
    $\Pi(\mu_{X}\otimes \mu_{X},\mu_{X'}\otimes \mu_{X'})$, while in the true Gromov-Wasserstein distance, $\gamma$ must take the specific form $\gamma = \sigma\otimes \sigma$ for some $\sigma\in \Pi(\mu_X,\mu_{X'})$. Indeed, the expression above (divided by 2) is identified as the ``second lower bound'' of the Gromov-Wasserstein distance by \citet{memoli_GH_Distance}.

    Another connection between the Gromov-Wasserstein and Risk distances can be drawn as well. Specifically, one can encode compact metric measure spaces $X$ as supervised learning problems $i(X)$ in a natural way which preserves the identifications induced by the two pseudometrics.
    \begin{proposition}\label{prop:metric-measure-embedding}
        Given a compact metric measure space $(X,d_X,\mu_X)$, define a problem
        \[i(X) := (X, X, (\Delta_X)_\sharp \mu_X, d_X, H_X),\]
        where
        \begin{itemize}
            \item $\Delta_X:X\to X\times X$ is the diagonal map, and
            \item $H_X$ is the set of all constant functions $X\to X$.
        \end{itemize}
        If $X$ and $X'$ are metric measure spaces, then
        \[\dGWp{1}(X,X') = 0 \iff \dexp(i(X),i(X')) = 0.\]
    \end{proposition}
    The proof can be found in Appendix~\ref{app:conn-dr-dgw} where we also discuss a connection with a variant of the Gromov-Wasserstein distance by \citet[Section 4.1]{hang2019topological} that is recovered as $\dexp(i(X),i(Y))$. 
    It is a folklore result in optimal transport that deciding whether $\dGWp{1}(X,X')=0$ for finite metric measure spaces $X$ and $X'$ is at least as hard as the graph isomorphism problem.
    This is not difficult to prove; let $G = (V,E)$ and $G'=(V',E')$ be finite simple graphs. We can encode $G$ as a metric measure space $V = (V,d_V,\mu_V)$, where $d_V$ is the shortest path metric and $\mu_V$ is the uniform probability measure. Similarly construct $V' = (V',d_{V'},\mu_{V'})$. Then $\dGWp{1}(V,V')=0$ if and only if $G$ and $G'$ are isomorphic. Indeed, if $G$ and $G'$ are isomorphic then the isomorphism induces the requisite coupling between $\mu_V$ and $\mu_{V'}$. Conversely, if $\dGWp{1}(V,V')=0$, then
    $V$ and $V'$ are isomorphic as metric measure spaces, and in particular are isometric as metric spaces. An isometry between $V$ and $V'$ that preserves the shortest path metric is in fact a graph isomorphism.
    
    Hence Proposition~\ref{prop:metric-measure-embedding} yields the following corollary.
    \begin{corollary}\label{cor:gi-hard}
        The problem of deciding whether $\dexp(P,P')=0$ for two problems with finite input and response spaces is at least as hard as graph isomorphism.
    \end{corollary}
    
    It is desirable to draw a stronger connection between the Risk distance and Gromov-Wasserstein distance by finding a function $i$ such that $d_{\mathrm{GW},1}(X,X') = \dexp(i(X),i(X'))$. We will nearly accomplish this goal in Section~\ref{subsec:OT-connection-weighted} by linking the $L^p$-Risk distance with a slightly modified Gromov-Wasserstein distance which is known in the literature. 

\section{Stability}\label{sec:stability} 
One major use of the Risk distance is to facilitate stability results.
By ``stability result,'' we mean a result of one of the two following forms. Suppose we modify a problem $P$ via some process to produce another problem $f(P)$. We say that \emph{$\dexp$ is stable under $f$} if every problem $P$ satisfies an inequality of the form
\[\dexp(P,f(P)) \leq J(f)\]
for some fixed functional $J$ which measures how dissimilar $f$ is from the identity function $\mathrm{id}:\mathcal{P}\to \mathcal{P}$ and satisfies $J(\mathrm{id})=0$. We would say that ``the Risk distance is stable under $f$''.
Stability results of this form assure us that small modifications to a problem do not move that problem far under $\dexp$. For instance, Condition~2 from Section~\ref{subsec:motivation} is a stability result. In particular, Condition~2 implies that $\dexp$ is stable under changes to the loss function.

Alternatively, a ``stability result'' could mean the following. Suppose that, from any problem $P$, we can compute a descriptor $g(P)$ which lies in a metric space $(Z,d_Z)$. We say that $g$ is \emph{stable under $\dexp$} if there is some non-negative constant $C$ satisfying
\[d_Z(g(P),g(P')) \leq C \dexp(P,P')\]
for all problems $P$ and $P'$. Stability results of this form act as quantitative guarantees that problems ``close'' under $\dexp$ have similar descriptors.

\subsection{Stability of the Constrained Bayes Risk and Loss Profile}\label{subsec:loss-control}
From a machine learning perspective, the most important feature of a predictor is the loss that it incurs. This information can be captured by a predictor's \emph{loss profile}, a probability measure on the real line describing the loss that a predictor incurs on a random observation. If $P$ is a problem and $h \in H$ is a predictor, the \emph{loss profile of $h$} is the pushforward $(\ell_h)_\sharp \eta \in \prob(\bbR)$.

Note that the loss profile of $h$ is the distribution of the random variable $\ell(h(\mathsf X),\mathsf Y)$, where $\mathsf X$ and $\mathsf Y$ are distributed according to $\eta$. This is closely related to the random variable $\ell(h(\mathsf X),\mathsf Y) - \cB(P)$, called the \emph{excess loss random variable} in the literature, differing only by an additive constant. The distribution of the excess loss is related to convergence rates of empirical risk minimization \citep{van_erven_fast_2015,grunwald_fast_2020}.

Letting $h$ range over all of $H$, we get an informative description of a problem.
\begin{definition}
    Given a problem $P$, the \emph{loss profile set of $P$} is the set
    \[L(P) := \set{(\ell_{h})_\sharp\eta}{h\in H}.\]
\end{definition}

Note that the loss profile $L(P)$ of $P$ contains enough information to recover the constrained Bayes risk. This follows immediately from the definition (we omit the proof).
\begin{proposition}\label{prop:bayes-from-profile}
Given a problem $P$, one has
$$\cB(P) = \inf_{\alpha \in L(P)}\mathrm{mean}(\alpha).$$
\end{proposition}

Since probability measures on the real line, such as loss profiles, can be compared using the Wasserstein metric $\dW$, \emph{collections} of probability measures on $\mathbb R$ form subsets of the metric space $(\prob(\mathbb R), \dW)$, and therefore can be compared using the Hausdorff distance $\dH$ with underlying metric $\dW$. That is, if $A,B \subseteq \prob(\mathbb R)$, the Hausdorff distance between $A$ and $B$ is given by
\begin{align*}
    \dH^{\dW}(A,B) = \inf_{R\in \cC(A,B)} \sup_{(\mu,\nu)\in R} \dW(\mu,\nu).
\end{align*}
Now that we have a way to compare loss profile sets, we can prove their stability under $\dexp$. The theorem below establishes this stability property, and in addition, also proves that the constrained Bayes risk is  stable in the sense of the Risk distance.

\begin{theorem}[Stability of the Constrained Bayes Risk and Loss Profile]\label{thm:loss-control}
    If $P$ and $P'$ are problems, then
    \[|\cB(P)-\cB(P')|\leq \dH^{\dW}(L(P),L(P')) \leq \dexp(P,P')\]
\end{theorem}

The proof, which is relegated to Appendix~\ref{appendix:proofs}, proceeds in two steps. We establish the first inequality using the fomulation of the constrained Bayes risk presented in Proposition~\ref{prop:bayes-from-profile}. To prove the second inequality, we notice that the quantity in the middle defines a pseudometric on the collection of all problems, then apply the maximality claim in Theorem~\ref{thm:characterization} to show that the Risk distance is larger.
By ignoring the middle quantity in Theorem~\ref{thm:loss-control}, we obtain stability the constrained Bayes risk. 
\begin{examples}[Constrained Bayes Risk and Bounded Loss]
Here we remark that if the problems $P$ and $P'$ have bounded loss functions, then the stability of the constrained Bayes risk, namely that 
if $P$ and $P'$ are problems, then
    \[|\cB(P) - \cB(P')| \leq \dexp(P,P'),\]
can be easily and directly obtained through arguments involving the metric properties of the Risk distance. Indeed, if $c>\max(\sup \ell,\sup \ell')$, then,   by Example \ref{ex:bayes-one-point} and the triangle inequality for the Risk distance, we have that
$$c-\cB(P) = \dexp(P,P_\bullet(c)) \leq \dexp(P,P') + \dexp(P',P_\bullet(c)) = \dexp(P,P') + c-\cB(P')$$
so that $\cB(P')-\cB(P)\leq \dexp(P,P').$ By swapping the roles of $P$ and $P'$ we obtain the inequality $\cB(P)-\cB(P')\leq \dexp(P',P)$ which then, together with the fact that $\dexp(P',P) = \dexp(P,P')$,  implies the stability claim. Note that Theorem \ref{thm:loss-control} implies the same claim without assuming the boundedness condition on the loss functions.
\end{examples}

\subsection{Stability of Rademacher Complexity}\label{subsec:rademacher}
In statistical learning theory, the \emph{Rademacher complexity} quantifies the flexibility of a predictor set and controls generalization error \citep{mohri_foundations_2018}. We establish the standard definition of Rademacher complexity of a function class, along with a related version stated in the language of supervised learning problems. In what follows, if $\mu$ is a measure on a measurable space $X$, let $\mu^{\otimes m}$ represent the $m$-fold product measure $\mu\otimes \dots \otimes \mu \in \prob(X^m)$.
\begin{definition}
    Let $m\geq 1$. Given a probability space $(X,\mu)$ and a family $\mathcal F$ of measurable real-valued functions on $X$, the classical \emph{$m$th Rademacher complexity of $\mathcal F$} is given by
    \[R_m(\mathcal F) := \int\int \sup_{f\in \mathcal F} \frac{1}{m} \sum_{i=1}^m \sigma_i f(x_i)
        \, \rad^{\otimes m}(d\sigma)
        \,\mu^{\otimes m}(d\overline x),\]
    where $\rad$ is the Rademacher distribution (that is, the uniform distribution on $\{-1,1\}$), $\sigma = (\sigma_1,\dots, \sigma_m)$, $\overline x = (x_1,\dots, x_m)$.
    Similarly, the \emph{$m$th Rademacher complexity of a problem $P$} is given by
    \[
    R_m(P) := \int\int \sup_{h\in H} \frac{1}{m} \sum_{i=1}^m \sigma_i \ell_h(x_i,y_i)
    \, \rad^{\otimes m}(d\sigma)
    \,\eta^{\otimes m}(d\overline x \times d\overline y).
    \]
    That is, $R_m(P)$ is the $m$th Rademacher complexity of the function class $\{\ell_h | h\in H\}$.
\end{definition} 
The Rademacher complexity of a problem $P$ is high if, given a random sample $((x_i, y_i))_{i=1}^m$ of $m$ observations and a vector $\sigma$ of random $\pm 1$ noise, we can typically find some predictor $h$ whose vector of losses $(\ell_h(x_i,y_i))_{i=1}^m$ is highly correlated with the noise $\sigma$. This would suggest that $H$ is quite flexible, and possibly that empirical risk minimization over $H$ runs the risk of overfitting on a random sample of size $m$.

For a fixed $m$, the Rademacher complexity $R_m$ need not be stable under the Risk distance.
\begin{examples}\label{ex:rademacher-discontinuous}
    For $n\geq 2$, define the problem
    \[P_n = ([n], \{0,1\}, \mathrm{unif}([n]) \otimes \delta_0, \ell, H_n)\]
    where $[n] := \{1,\dots,n\}$, $\ell$ is the 0-1 loss, and $H_n$ is the set of indicator functions of singletons $1_{k}:[n]\to \{0,1\}$ for all $k\in [n]$. Recall also the one-point problem $P_\bullet = P_\bullet(0)$ from Example~\ref{ex:one-point-problem}.

    We claim that $R_1(P_\bullet) = 0$, $R_1(P_n) = 1/2$ for all $n$, but $\dexp(P_n,P_\bullet)\to 0$ as $n\to \infty$. This shows that the Rademacher complexity  $R_1$ is not continuous with respect to the Risk distance.  
    One can easily see that $R_1(P_\bullet)=0$, since the loss function of $P_\bullet$ is identically 0. We can also easily compute
    \begin{align*}
        R_1(P_n) & = \int \left(\frac{1}{2} \sup_{k\in [n]} \ell(1_{k}(x),y) - \frac{1}{2}\inf_{k\in [n]} \ell(1_{k}(x),y) \right)\, (\mathrm{unif}([n])\otimes \delta_0)(dx\times dy)
        \\ & =
        \int \left(\frac{1}{2} (1) - \frac{1}{2}(0) \right)\, (\mathrm{unif}([n])\otimes \delta_0)(dx\times dy) = \frac{1}{2}.
    \end{align*}
    Meanwhile, since there is only one coupling with the point mass $\delta_{(\bullet,\bullet)}$ and only one correspondence with the singleton $\{\id_{\{\bullet\}}\}$, we can compute
    \begin{align*}
        \dexp(P_n,P_\bullet) & =
        \dis_{P_n,P_\bullet}(H_n \times \{\id_{\{\bullet\}}\}, \mathrm{unif}([n]) \otimes \delta_0\otimes \delta_{(\bullet,\bullet)})
        \\ & =
        \sup_{k\in[n]}\int |\ell(1_{k}(x),0) - 0| \, \mathrm{unif}([n])(dx)
        \\ & =
        \sup_{k\in[n]}\int 1_{k}(x) \, \mathrm{unif}([n])(dx)
        = \sup_{k\in [n]} \frac{1}{n} = \frac{1}{n}.
    \end{align*}
    Hence, even though $P_n \to P$ in the Risk distance, the Rademacher complexities $R_1(P_n)$ do not converge to $R_1(P_\bullet)$.
\end{examples}

Example~\ref{ex:rademacher-discontinuous} shows that $R_1$ is not even continuous with respect to the Risk distance. This suggests more generally that we should not hope for straightforward stability of $R_m$ under the Risk distance when $m$ is small. If we instead take $m$ to be arbitrarily large, we can still obtain a kind of stability.
\begin{theorem}[Stability of Rademacher Complexity]\label{thm:Rademacher-control}
    Let $P$ and $P'$ be problems and consider the family
    \[\mathcal F := \set{\big| \ell_h(-,-) - \ell'_{h'}(-,-) \big|\,}{h\in H, h'\in H'}\]
    of functions $X\times Y \times X' \times Y' \to \mathbb R$.
    Then \[\big|R_m(P) - R_m(P')\big| \leq \dexp(P,P') + 2R_m(\mathcal F).\]
	If $P$ and $P'$ have bounded loss functions and $\mathcal F$ forms a universal Glivenko-Cantelli class, then
    \[
    \limsup_m \big|R_m(P) - R_m(P')\big| \leq \dexp(P,P').
    \]
\end{theorem}
The proof can be found in Appendix~\ref{appendix:proofs}. The conditions under which $R_m(\mathcal F)\to 0$ as $m\to \infty$, and the rate of said convergence, are classical subjects of research in statistical learning theory. Results in this direction typically provide convergence rates of order  $O(\sqrt{1/m})$ \citep{bartlett_rademacher_2002} or $O(\sqrt{\ln(m)/m})$ \citep[Chapter 3]{mohri_foundations_2018} under specific assumptions.

One can apply Proposition~\ref{prop:VC-sufficient} to produce at a weaker form of Theorem~\ref{thm:Rademacher-control} which does not reference the Glivenko-Cantelli property. Specifically, we can replace the universal Glivenko-Cantelli condition Theorem~\ref{thm:Rademacher-control} with the following pair of sufficient conditions:
\begin{enumerate}
    \item The sets $H$ and $H'$ can be endowed with Polish topologies, and subsequently Borel $\sigma$-algebras, such that the maps $(h,x,y)\mapsto \ell_h(x,y)$ and $(h',x',y')\mapsto \ell'_{h'}(x',y')$ are measurable.
    \item The family of subgraphs $\set{\mathrm{sg}(g_{h,h'})}{h \in H, h'\in H'}$ (see Definition \ref{def:subgraph}) has finite VC dimension, where
    \[g_{h,h'}(x,y,x',y') := \big|\ell_h(x,y) - \ell'_{h'}(x',y')\big|.\]
\end{enumerate}

In fact, we can state an even simpler sufficient condition involving the two problems $P$ and $P'$ separately.
\begin{corollary}\label{cor:rademacher-control-subspace}
    Let $P$ and $P'$ be problems with bounded loss functions. Assume that $H$ and $H'$ are equipped with Polish topologies such that the maps $(h,x,y)\mapsto \ell_h(x,y)$ and $(h',x',y')\mapsto \ell'_{h'}(x',y')$ are measurable.
    
    If both $\set{\ell_h}{h\in H}$ and $\set{\ell'_{h'}}{h'\in H'}$ lie in finite-dimensional subspaces of $L^1(X\times Y)$ and $L^1(X'\times Y')$ respectively, then
    \[
    \limsup_m \big|R_m(P) - R_m(P')\big| \leq \dexp(P,P').
    \]
\end{corollary}

The subspace condition in Corollary~\ref{cor:rademacher-control-subspace}, while simple, may seem strange. We claim the condition is not as odd as it may at first appear.
\begin{examples}
    Consider the generic linear regression problem $P$ from Example~\ref{ex:linear-regression}. For that problem, the functions $\ell_h$ are of the form
    \[\ell_h(x,y) = (a^Tx + b - y)^2\]
    for $a\in \bbR^n$, $b\in \bbR$. These functions are polynomials in $x$ and $y$ of degree at most 2. Hence the collection $\set{\ell_h}{h\in H}\subseteq L^1(\bbR^n\times \bbR)$ lies within the finite-dimensional subspace of polynomials of degree at most 2. Hence $P$ satisfies the subspace condition of Corollary~\ref{cor:rademacher-control-subspace}.

    More generally, any problem
    \[P = (\bbR^n, \bbR^m, \eta, \ell, H)\]
    will satisfy the subspace condition if $\ell$ is a polynomial loss function and $H$ is of the form
    \[H = \set{a_1 f_1(x) + \dots + a_k f_k(x)}{a_1,\dots,a_k \in \bbR}\]
    for some fixed $f_1,\dots f_k:\bbR^n\to \bbR^m$. This is because the set $\set{\ell_h}{h\in H}$ will fall within a finite-dimensional space of polynomials in the variables $f_1,\dots f_k$ and $y$.
\end{examples}

\begin{proof}[Corollary~\ref{cor:rademacher-control-subspace}]
    Let
    \[\cS:= \mathrm{span}(\set{\ell_h(-,-) - \ell'_{h'}(-,-)}{h\in H, h'\in H'}),\]
    where the span is taken in the space of all measurable functions $X\times Y \times X' \times Y' \to \bbR$. Consider
    \[\max(\cS,\cS) := \set{\max(f,g)}{f,g\in \cS},\]
    where the maximum $\max(f,g)$ is taken pointwise. Since $\cS$ is finite-dimensional, it follows \citep[Lemma 2.6.15 and Lemma 2.6.18(ii)]{van_der_vaart_weak_1996} that the collection of subgraphs of functions in $\max(\cS,\cS)$ has finite VC dimension.  The set $\set{g_{h,h'}}{h\in H, h'\in H'}$ lies in $\max(\cS,\cS)$, so Proposition~\ref{prop:VC-sufficient} finishes the proof.
\end{proof}

\subsection{Stability under Sampling Bias}\label{subsec:stability-bias}
        A fundamental source of error in learning problems of any type is bias in the collection of training data. Suppose $\eta$ is the true joint law that we are trying to approximate with a predictor. While we would like to sample from $\eta$ directly, an imperfect collection process may over-represent certain regions of the data space. One would hope that minor bias would only have a minor effect on the problem. More formally, we might ask if the Risk distance is stable under small modifications to a problem's joint law. This section provides conditions for such stability results.

    The Risk distance is stable under changes to the joint law with respect to the Wasserstein distance with a certain underlying metric. Under certain conditions, it is also stable with respect to the total variation distance (Definition~\ref{def:tv-distance}).
    Define the pseudometric $s_{\ell,H}$ on $X\times Y$ by
    \[s_{\ell,H}((x,y),(x',y')) := \sup_{h\in H} \big | \ell(h(x),y) - \ell(h(x'),y') \big |.\]
    \begin{proposition}\label{prop:H-wasserstein-bound}
        Let $P$ be a problem, and let $P'$ be a problem that is identical to $P$ except possibly for its joint law $\eta'$. Then
        \begin{align*}
            \dexp(P,P') \leq \dW^{s_{\ell,H}}(\eta,\eta').
        \end{align*}

        Suppose furthermore that the loss function $\ell$ is bounded above by some constant $\ell_{\sf{max}}$. Then
        \begin{align*}
            \dexp(P,P') \leq \ell_{\sf{max}}\TV(\eta,\eta').
        \end{align*}
    \end{proposition}
    \begin{proof}
        One can bound $\dexp(P,P')$ by choosing the diagonal correspondence between $H$ and itself to get
        \begin{align*}
            \dexp(P,P')
            & \leq
            \inf_{\gamma \in \Pi(\eta,\eta')} \sup_{h\in H} \int \big|\ell(h(x),y) - \ell(h(x'),y') \big| \, \gamma(dx {\times} dy {\times} dx' {\times} dy')
            \\ & \leq
            \inf_{\gamma \in \Pi(\eta,\eta')} \int \sup_{h\in H} \big|\ell(h(x),y) - \ell(h(x'),y') \big| \, \gamma(dx {\times} dy {\times} dx' {\times} dy')
            = \dW^{s_{\ell,H}}(\eta,\eta'),
        \end{align*}
        proving the first part of the proposition. Now, if $\ell$ is bounded by $\ell_{\sf{max}}$ then $s_{\ell,H}$ is bounded by $\ell_{\sf{max}}d$, where $d$ is the discrete metric. Hence
        \begin{equation*}
            \dW^{s_{\ell,H}}(\eta,\eta') \leq \dW^{\ell_{\sf{max}}d}(\eta,\eta') = \ell_{\sf{max}}\TV(\eta,\eta') %\qedhere
        \end{equation*}
        as desired.
    \end{proof}

    In the presence of sampling bias, the data collection process gives inordinate weight to some regions of the data space $X\times Y$ at the cost of other regions. We can think of this biased data collection as sampling not from the true probability measure $\eta$, but instead from a scaled probability measure given by
    \[f\eta(A) := \int_A f(x,y)\, \eta(dx\times dy)\]
    for all measurable $A$, where $f:X\times Y \to \mathbb R_{\geq 0}$ is a density function with respect to $\eta$, meaning $f$ is a measurable function with $\int f(x,y)\,\eta(dx{\times} dy) = 1$. Here $f$ is larger than 1 in regions of $X\times Y$ that are overrepresented by the data collection process, and less than 1 in regions underrepresented.

    In extreme cases of sampling bias, some regions of the data space $X\times Y$ may be completely missing from the training data. In this case, instead of being sampled from the true joint law $\eta$, the data is sampled from some subset $A\subseteq X\times Y$ with $\eta(A)>0$ according to the probability measure
    \[\eta|_A(B) := \frac{1}{\eta(A)} \eta(A \cap B).\]

    Using Proposition~\ref{prop:H-wasserstein-bound}, we can show that $\dexp$ is stable under the above formulations of sampling bias.
    \begin{corollary}[Stability Under Sampling Bias]\label{cor:stability-bias}
        Let $P$ be a problem.
        \begin{enumerate}[(1)]
            \item Let $f:X\times Y \to \mathbb R_{\geq 0}$ be a density function with respect to $\eta$. Define
            \[P_f := (X,Y, f\eta, \ell, H).\]
            Then
            \begin{align*}
                \dexp(P,P_f) \leq \frac{1}{2} \int \big|1-f(x,y)\big| \, \eta(dx{\times} dy).
            \end{align*}
            \item Let
            $A\subseteq X\times Y$ be a measurable subset with $\eta(A)>0$. Define
            \[P|_A := (X,Y, \eta|_A, \ell, H).\]
            Then
            \begin{align*}
            \dexp(P,P|_A) \leq 1- \eta(A).
            \end{align*}
        \end{enumerate}
    \end{corollary}
    \begin{proof}
        First recall that
        \begin{align*}
            \TV(\eta,\eta') = \sup_{B\subseteq X\times Y} \eta(B) - f\eta(B) = \sup_{B\subseteq X\times Y} f\eta(B) - \eta(B),
        \end{align*}
        where the infima range over all measurable subsets of $X\times Y$. Then
        \begin{align*}
            \TV(\eta,\eta')
            & = \frac{1}{2}\left(
                \sup_{B\subseteq X\times Y} \eta(B) - f\eta(B)
                +
                \sup_{B\subseteq X\times Y} f\eta(B) - \eta(B)
            \right)
            \\ &
            \displaybreak[1]
            = \frac{1}{2}\left(
                \sup_{B\subseteq X\times Y} \int_B 1-f(x,y)\, \eta(dx{\times} dy)
                +
                \sup_{B\subseteq X\times Y} \int_B f(x,y) - 1\, \eta(dx {\times} dy)
            \right)
            \\ & =\frac{1}{2}\left(
                \int_{\{f<1\}} 1-f(x,y)\, \eta(dx{\times} dy)
                +
                \int_{\{f>1\}} f(x,y) - 1\, \eta(dx {\times} dy)
            \right)
            \\ &
            = \frac{1}{2} \int \big|1-f(x,y) \big| \, \eta(dx{\times}dy).
        \end{align*}
        To finish the proof, apply part~(1) when $f = \frac{1}{\eta(A)}1_{A}$.
    \end{proof}

\subsection{Stability under Sampling Noise}\label{subsec:stability-noise}
    Another source of error in sampling is the presence of noise in the process of collecting data. While we would like to sample from $\eta$ directly, our data is rarely collected with perfect fidelity. One would hope that a small amount of noise would not change the problem much. In this section, we model noise with a \emph{noise kernel} and show that we can often guarantee that the Risk distance is stable under the application of such noise.

    First we consider the case where the loss function $\ell$ is bounded by a metric on $Y$. Our bounds are in terms of a generalization of the Wasserstein metric for Markov kernels.
    \begin{definition}\label{def:markov-coupling}
        Let $X$ be a measure space and let $M:X\to \prob(Y)$, $N:X\to \prob(Z)$ be Markov kernels. A \emph{coupling} between $M$ and $N$ is a Markov kernel $\tau:X \to \prob(Y\times Z)$ such that for a.e. $x\in X$,
        \begin{align*}
            (\pi_1)_\sharp (\tau(x)) &= M(x),\\
            (\pi_2)_\sharp (\tau(x)) &= N(x).
        \end{align*}
        Just as we use $\Pi(\mu,\nu)$ to denote the set of all couplings of probability measures $\mu$ and $\nu$, we use $\Pi(M,N)$ to denote the set of all couplings of Markov kernels $M$ and $N$.
    \end{definition}
    Note that if $X$ is a single point, we recover the classical definition of a coupling of probability measures.
    \begin{definition}[Wasserstein Distance for Markov Kernels]\label{def:wass-dist-for-markov-kernels}
        Let $(X,\mu)$ be a measure space and let $(Y,d_Y)$ be a metric space. Let $M,N:X\to \prob(Y)$ be a pair of Markov kernels. For $1\leq p < \infty$, the \emph{$p$-Wasserstein distance between $M$ and $N$} is given by
        \[\dWkernp{p}(M,N) := \left(\inf_{\tau\in \Pi(M,N)}\int_X \int_{Y\times Y} d_Y^p(y,y') \tau(x)(dy\times dy') \mu(dx) \right)^{1/p}.\]
    \end{definition}
    Definitions~\ref{def:markov-coupling} and \ref{def:wass-dist-for-markov-kernels} are due to \citet[Defs 5.1 and 5.2]{patterson2020}. Patterson notes that terminology conflicts between ``products'' and ``couplings'' of Markov kernels are present in the literature.
    As with the Wasserstein distance, we will have little use for $\dWkernp{p}$ for $p\neq 1$, and so establish the convention that the \emph{Wasserstein distance for Markov kernels} refers to $\dWkern$.
    An alternative, equivalent definition of the metric $\dWkern$ may be more intuitive:
    \[\dWkern(M,N) = \int_X \dW^{d_Y}(M(x),N(x)) \, \mu(dx).\]
    Equivalence of these definitions was noted by \citet{patterson2020}, and follows from results in the theory of optimal transport \citep[Corollary 5.22]{Villani_old_and_new_2008}. A distance based on a generalization of the above alternative definition was independently discovered and studied by \citet{kitagawa_two_2023} under the name ``disintegrated Monge-Kantorovich distance''.

    We are abusing notation by using the same symbol for the Wasserstein distances between measures and between Markov kernels. We justify this by noting that the former is a special case of the latter. Indeed, recall that if $\mu,\nu \in \prob(X)$ are probability measures on a metric space $X$, we can think of $\mu$ and $\nu$ as Markov kernels from the one-point measure space $(\{\bullet\}, \delta_\bullet)$. The Wasserstein distance between these Markov kernels then agrees with the Wasserstein distance between the measures $\mu$ and $\nu$.

    The Wasserstein distance for Markov kernels also generalizes the $L^1$ distance on measurable functions. Indeed, if we have measurable functions $f,g:X\to Y$, then the distance $\dW(\delta_f,\delta_g)$ between the induced deterministic Markov kernels is just the $L^1$ distance between $f$ and $g$.
    
    A simple kind of noise can be modeled with a Markov kernel $M: Y\to \prob(Y)$. We interpret this noise as follows. When we make an observation that would have label $y\in Y$, we instead observe a random corrupted label with law $M(y)$. More specifically, let $P$ be a problem whose joint law has $X$-marginal $\alpha$ and $X$-disintegration $\beta:X\to \prob(Y)$. Without the presence of noise, an observation has a random input component $x$ chosen with law $\alpha$, and a random label chosen with law $\beta(x)$. The presence of noise represented by $M$ creates a new problem where the labels are instead chosen according to the law $(\beta \cdot M)(x)$. In other words, we get a problem with the same $X$-marginal $\alpha$, but with disintegration $\beta':X\to \prob(Y)$ given by
    \[\beta'(x) := (\beta\cdot M)(x).\]
    We might call this type of noise ``$X$-independent $Y$-noise,'' (otherwise known as ``simple label noise'') since the noise $M$ is applied in the $Y$ direction, and the same noise kernel $M$ is applied for all $x\in X$.
    
    A more general kind of noise would allow for dependence on $X$. Instead of a Markov kernel $M:Y\to \prob(Y)$, we model our noise as a Markov kernel $N:X\times Y \to \prob(Y)$. We write $N_x$ to mean the Markov kernel $N(x,-): Y\to \prob(Y)$. In the presence of this kind of noise, the problem $P$ becomes a new problem $P'$ with the same $X$-marginal $\alpha$, but $X$-disintegration $\beta'$ now given by
    \begin{align*}
        \beta'(x):=(\beta \cdot N_x)(x).
    \end{align*}
    We might call this kind of noise ``$X$-dependent $Y$-noise''.
    
    Consider a noise kernel $N = \delta_{\pi_2}$, which is the Markov kernel sending $(x,y)$ to the point mass $\delta_y$. Then for all $x\in X$, we have $N_x(y) = \delta_y$ and hence $\beta \cdot N_x = \beta$. In other words, $N$ represents no noise at all, and applying $N$ to a problem would leave the problem unchanged. Our next stability result states that, if any kernel $N$ is close to this ``no noise'' case under $\dWkern$, then applying $N$ will not move the problem much in the Risk distance, provided that $\ell$ is well-controlled by some metric on $Y$.

    \begin{lemma}\label{lem:disintegration-bound}
    Suppose that $P$ and $P'$ are problems given by
    \[P = (X,Y,\eta,\ell,H)\]
    \[P' = (X,Y,\eta',\ell,H)\]
    where $\eta$ and $\eta'$ have the same $X$-marginal $\alpha \in \prob(X)$, but different disintegrations along $X$, called $\beta$ and $\beta'$ respectively. Suppose also that $Y$ is equipped with a metric $d_Y$ with respect to which $\ell$ is Lipschitz in the second argument, meaning there is a constant $C>0$ such that
    \begin{align*}
        \big| \ell(z,y) - \ell(z,y') \big| \leq C d_Y(y,y')
    \end{align*}
    for all $y,y',z\in Y$. Then
    \[\dexp(P,P') \leq C\dWkern^{d_Y}(\beta,\beta').\]
    \end{lemma}
    \begin{theorem}[Stability under $X$-dependent $Y$-noise]\label{thm:metric-noise-stability}
        Let $P$ be a problem whose output space $Y$ is equipped with a metric $d_Y$, with respect to which the loss function $\ell$ is $C$-Lipschitz in the second argument. Let $P'$ be the problem obtained by applying the noise kernel $N:X\times Y \to \prob(Y)$ to $P$. Then
        \[\dexp(P,P') \leq C \dWkern^{d_Y}(N,\delta_{\pi_2}).\]
    \end{theorem}
    The proofs are relegated to Appendix~\ref{appendix:proofs}.
    We illustrate Theorem~\ref{thm:metric-noise-stability} with an example regarding label noise in binary classification.

    \begin{examples}\label{ex:binary-label-noise}
        Let $P = (X,\{0,1\}, \eta,\ell,H)$ be a binary classification problem where $\ell$ is the 0-1 loss $\ell(y,y')=1_{y\neq y'}$. Suppose we apply noise to $P$ by taking each observation $(x,y)$ and doing the following:
        \begin{itemize}
            \item With probability $\epsilon$, re-assign its label $y$ uniformly at random.
            \item With probability $(1-\epsilon)$, leave the label $y$ as it is.
        \end{itemize}
        Through this process we are effectively applying the noise kernel
        \[N_\epsilon := \epsilon u + (1-\epsilon)\delta_{\pi_2},\]
        where $u$ is the uniform measure on $\{0,1\}$. Let $P'_\epsilon$ be the problem $P$ with the noise kernel $N_\epsilon$ applied.
        Since the 0-1 loss is a metric, we can apply Theorem~\ref{thm:metric-noise-stability} with $d_Y = \ell$ and $C=1$ to get
        \begin{align*}
            \dexp(P,P'_\epsilon) & \leq \dW(\epsilon u + (1-\epsilon)\delta_{\pi_2}, \delta_{\pi_2})
            \\ & = \int \dW^{\ell}(\epsilon u + (1-\epsilon) \delta_y, \delta_y)\, \eta(dx{\times}dy).
        \end{align*}
        Recall (Section~\ref{subsec:total-variation}) that when the underlying metric is the discrete metric, the Wasserstein distance is the same as the total variation distance. Furthermore, one can show that $\TV(\epsilon u + (1-\epsilon) \delta_y, \delta_y) = \epsilon \TV(u,\delta_y) = \epsilon/2$. Hence we get
        \[\dexp(P,P'_\epsilon) \leq \epsilon/2.\]

        One could generalize further; suppose that the probability $\epsilon$ of reassignment \emph{depends on $x$}. That is, instead of selecting a constant $\epsilon$, we select a measurable function $\epsilon:X\to [0,1]$. Then, for an observation $(x,y)$, with probability $\epsilon(x)$ we re-assign $y$ uniformly at random. Then the same argument above instead gives us the bound
        \[\dexp(P,P'_\epsilon) \leq \frac{1}{2}\int \epsilon(x)\alpha(dx),\]
        where $\alpha$ is the $X$-marginal of $\eta$. That is, the risk distance between $P$ and the noised problem $P'_\epsilon$ is bounded by half the \emph{average} value of $\epsilon(x)$.
    \end{examples}

    In Section~\ref{subsec:improved-stability}, we will consider an even more general noise model, allowing for noise in the $X$ and $Y$ directions simultaneously. We will show that a modification of the Risk distance, called the $L^p$-Risk distance, is stable under such noise in Theorem~\ref{thm:non-metric-bound}, which is analogous to Theorem~\ref{thm:metric-noise-stability}.

\subsection{Stability under Modification of Predictor Set}\label{subsec:stability-H}

    We have shown that $\dexp$ is stable under changes to the loss function and certain changes to $\eta$. We round out this section by demonstrating stability of $\dexp$ under changes to the predictor set of a problem.
    %\vspace{-.2in}
    \begin{figure}
    \centering \includegraphics[width=0.3\linewidth]{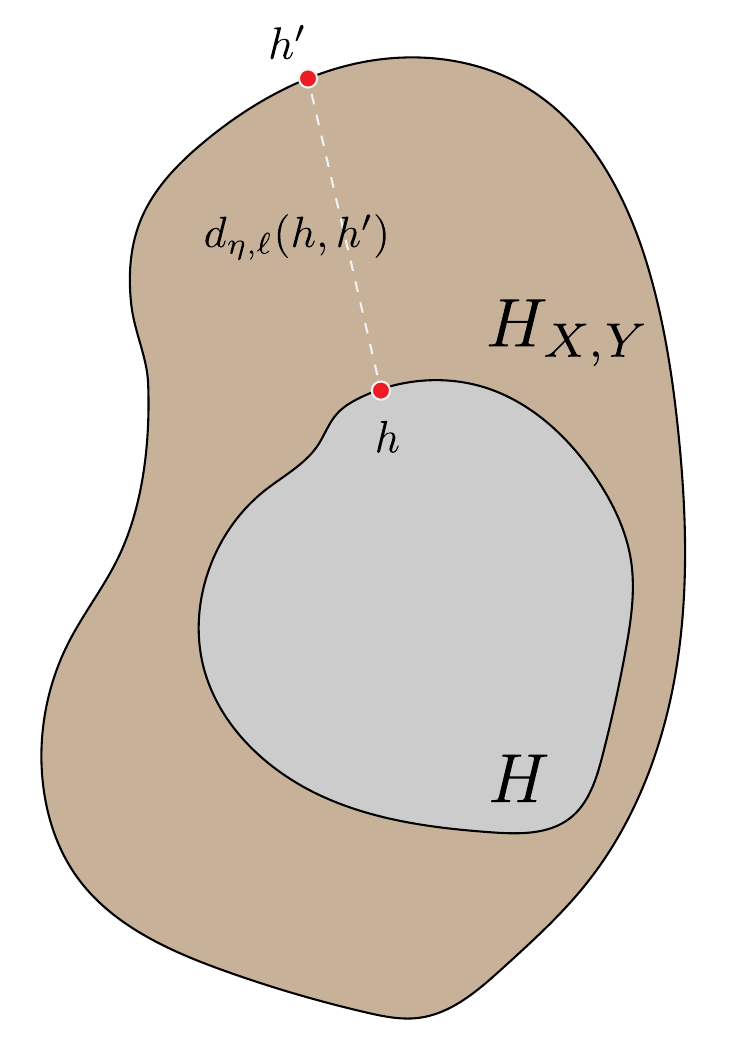}
    \caption{Distance between $H$ and $H_{X,Y}$; see Example \ref{ex:HXY}.}
    \label{fig:HXY}
    \end{figure}
    To state such a stability result, we first need a way to quantify changes to the predictor set. The most salient aspect of a predictor $h$ is its performance with respect to the chosen loss function. In other words, more crucial than the predictor $h$ itself is the function $\ell_h(x,y) := \ell(h(x),y)$. With this perspective in mind, we can map $H$ into $L^1(\eta)$ via the map $h \mapsto \ell_h$, and compare predictors using the $L^1$ norm there. That is, we compare two predictors $h$ and $h'$ via the quantity
    \begin{align*}
        d_{\ell,\eta}(h,h') := \|\ell_h - \ell_{h'}\|_{L^1(\eta)} = \int \big|\ell_h(x,y)-\ell_{h'}(x,y)\big|\, \eta(dx{\times}dy).
    \end{align*}
    
    This defines a pseudometric on the set of all measurable functions $h:X\to Y$ for which $\|\ell_h\|_{L^1(\eta)}<\infty$, which any predictor satisfies since the definition of a problem mandates that predictors have finite risk. The map $h\mapsto \ell_h$ identifies two predictors $h$ and $h'$ if they $\eta$-almost always incur the same loss on a random observation in $X\times Y$.

    We can now state a stability theorem in terms of the Hausdorff distance on predictor sets with underlying metric $d_{\ell,H}$, i.e. the metric
    \[\dH^{d_{\ell,\eta}}(H,H') = \inf_{R\in \cC(H,H')} \sup_{(h,h')\in R} \|\ell_h - \ell_{h'}\|_{L^1(\eta)}.\]

    \begin{theorem}\label{thm:H-stability}
        Let $P$ be a problem, and let $P'$ be a problem that is identical to $P$ except possibly for the predictor set $H'$. Then
        \begin{equation*}
            \dexp(P,P') \leq \dH^{d_{\ell,\eta}}\left(H,H'\right).
        \end{equation*}
    \end{theorem}

    \begin{examples}\label{ex:HXY}
    Note that if $P=(X,Y,\eta,\ell,H)$ and $P'=(X,Y,\eta,\ell,H_{X,Y})$, where $H_{X,Y}$ stands for the set of all measurable functions from $X$ to $Y$,  Theorem \ref{thm:loss-control} implies that $$\cB^*(P)\leq \cB(P)\leq \cB^*(P)+\sup_{h'\in H_{X,Y}}\inf_{h\in H} d_{\ell,\eta}(h,h').$$
    The right-hand side contains the worst case $d_{\ell,\eta}$ distance between an arbitrary measurable  function $h'$ and the (constrained) predictor set $H$. In other words, this quantity expresses how far  $H$ is from the superclass $H_{X,Y}$ and therefore controls the \emph{approximation error}: the difference between the constrained and unconstrained Bayes risks; see Figure \ref{fig:HXY}.
    \end{examples}

    \begin{proof}[Theorem \ref{thm:H-stability}]
        We bound $\dexp(P,P')$ by choosing the diagonal coupling between $\eta$ and itself. That is,
        \begin{align*}
            \dexp(P,P') & \leq \inf_{R\in \cC(H,H')} \sup_{(h,h')\in R} \int \big| \ell(h(x),y) - \ell(h'(x),y)\big | \, \eta(dx{\times} dy)
            \\ &
            = \inf_{R\in \cC(H,H')} \sup_{(h,h')\in R} \| \ell_h - \ell_{h'}\|_{L^1(\eta)} = \dH^{d_{\ell,\eta}}(H,H'), %\qedhere
        \end{align*}
        which is the desired bound.
    \end{proof}

    We illustrate the metric $\dH^{d_{\ell,\eta}}$ and Theorem~\ref{thm:H-stability} with a toy example.
    \begin{examples}
        Let $P = (\bbR^2, \{0,1\}, \eta, \ell, H)$ be the classification problem from Example~\ref{ex:linear-classification}. Recall that the predictors in $H$ were the functions of the form
        \begin{align*}
            h_{a,b,c}(x_1,x_2) := \begin{cases}
            1 & ax_1 + bx_2 \geq c \\
            0 & ax_1 + bx_2 < c
            \end{cases}.
        \end{align*}
        We modify $H$ by adding to each predictor a small area on which they always predict class 1. That is, we define a new predictor set $H'_t := \set{h'_{a,b,c,t}}{a,b,c\in\bbR}$ where $h'_{a,b,c,t}:\bbR^2\to \{0,1\}$ is given by
        \begin{align*}
            h'_{a,b,c,t}(x_1,x_2) := \begin{cases}
                1 & x\in B_t(0)\\
                h_{a,b,c}(x) & \text{otherwise} 
            \end{cases}.
        \end{align*}
        Here $B_t(0)$ represents the ball of radius $t$ centered at 0.
        We can construct a correspondence between $H$ and $H'_t$ by coupling each $h_{a,b,c}$ with $h'_{a,b,c,t}$. That is, we define
        \[R := \set{(h_{a,b,c},h'_{a,b,c,t})}{a,b,c \in \bbR}.\]
        Then $R\in \cC(H,H'_t)$, so we can write the bound
        \begin{align*}
            \dH^{d_{\ell,\eta}}(H,H'_t)
            & \leq \sup_{(h_{a,b,c},h'_{a,b,c,t})\in R} \| \ell_{h_{a,b,c}} - \ell_{h_{a,b,c,t}}\|_{L^1(\eta)}
            \\ & =
            \sup_{a,b,c \in \bbR} \int \big| \ell(h_{a,b,c}(x),y) - \ell(h_{a,b,c,t}(x),y) \big| \,\eta(dx{\times}dy)
            \\ & =
            \sup_{a,b,c \in \bbR} \int_{B_t(0)\times \{0,1\}} \big| \ell(h_{a,b,c}(x),y) - \ell(h_{a,b,c,t}(x),y) \big| \,\eta(dx{\times}dy)
            \\ & \leq \eta(B_t(0)\times\{0,1\}). 
        \end{align*}
        Furthermore, assuming $\eta({0}\times \{0,1\})=0$, we can take the limit as $t\to 0$ to get $\lim_{t\to0} \dH^{d_{\ell,\eta}}(H,H'_t) = 0$. If we define $P_t := (\bbR^2, \{0,1\}, \eta, \ell, H'_t)$ for all $t>0$, By Theorem~\ref{thm:H-stability}, $P_t$ converges to $P$ under the Risk distance as $t\to 0$.
    \end{examples}

    We now provide a more serious application of Theorem~\ref{thm:H-stability}. Different classes of predictors are often compared in the context of \emph{universal approximation theorems}, which generally state that a certain kind of function can be approximated arbitrarily well using a certain class of model. Using Theorem~\ref{thm:H-stability}, one can translate universal approximation theorems into statements about the risk distance. While there are many universal approximation theorems \citep{pinkus_approximation_1999}, we will demonstrate with a classical example \citep[see][]{leshno_multilayer_1993}.
    
    \begin{proposition}
        Let $\phi:\bbR \to \bbR$ be a continuous non-polynomial function. Let $P = (X, \bbR, \eta, \ell, H)$ be a problem where
        \begin{itemize}
            \item $X$ is a compact subset of $\bbR^m$,
            \item $\eta$ is any probability distribution on $X$,
            \item $\ell$ is a uniformly continuous loss function, and
            \item $H$ is the set of all functions $X\to \bbR$ which can come from a feedforward network with a single hidden layer of arbitrary width and with activation function $\phi$. That is, $H$ is the set of all functions of the form
            \[(x_1,\dots,x_m) \mapsto \sum_{i=1}^M a_i \phi\left(\sum_{j=1}^m w_{ij} x_j + b_i\right)\]
            where the parameters $a_i$, $w_{ij}$, and $b_j$ range over $\bbR$, and $M$ ranges over the positive integers.
        \end{itemize}
        Let $P'$ be the same problem, but with predictor set $H'$ the set of all continuous functions $X \to \bbR$. Then $\dexp(P,P') = 0$.
    \end{proposition}
    \begin{proof}
        Let $h':X\to \bbR$ be continuous and let $\epsilon >0$ be arbitrary. By uniform continuity of $\ell$, there exists a $\delta > 0$ such that, if $|y-y'|<\delta$, then $|\ell(y,y'')-\ell(y',y'')|<\epsilon$ for any $y''\in \bbR$. By \cite[Thm 1]{leshno_multilayer_1993}, there exists an $h\in H$ with $|h(x) - h'(x)| < \delta$ for all $x\in X$. Then
        \begin{align*}
            d_{\ell,\eta}(h,h') & = \int \big| \ell(h(x),y) - \ell(h'(x),y) \big| \, \eta(dx {\times} dy)
            \leq \int \epsilon\, \eta(dx{\times}dy) = \epsilon.
        \end{align*}
        Since $H \subseteq H'$, this implies the Hausdorff distance $d_H^{d_{\ell,\eta}}(H,H')$ is less than $\epsilon$. By Theorem~\ref{thm:H-stability}, $\dexp(P,P')<\epsilon$ as well.
    \end{proof}

    In other words, under certain conditions, the risk distance does not distinguish between a model and the class of functions in which it is dense.

\section{Geometry of the Space of Problems}
\label{sec:geometry}
We now shift our attention from stability properties of $\dexp$ to its geometric properties. We will explore geodesics and their connection to optimal couplings and correspondences. We also show that classification problems (that is, those with finite output spaces) are dense in a larger class of problems. We will use this density to study the effect of limited sampling by studying the convergence of ``empirical problems,'' which are problems whose joint laws are given by empirical measures gleaned from finite random samples.

\begin{figure}
    \centering
    \begin{tikzpicture}
        \shade[ball color = lightgray, opacity = 0.75] (0,0,0) ellipse (3cm and 1.5cm);

        \coordinate (P) at (-2,0.5);
        \coordinate (Q) at (2,-0.5);

        \draw[thick, red] (P) .. controls (0,0.75) and (1,0.5) .. (Q);

        \fill (P) circle (2pt) node[above] {P};
        \fill (Q) circle (2pt) node[below] {P'};
    \end{tikzpicture}
    \caption{A simple depiction of two problems $P$ and $P'$ in the space of problems, joined by a geodesic.}
\end{figure}
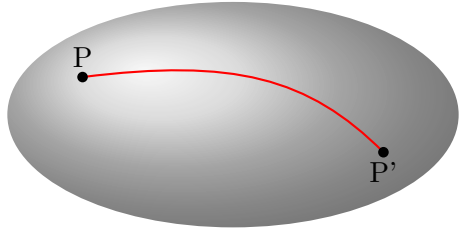

\subsection{Geodesics}\label{subsec:geodesics}
    We begin our exploration of the geometry of the space of problems by searching for geodesics. A \emph{geodesic} in a metric space $(X,d_X)$ is a continuous function from the unit interval $f:[0,1]\to X$ such that, for all $s,t\in [0,1]$,
    \[d_X(f(t),f(s)) = d_X(f(0),f(1))|s-t|.\]
    Some would call such an $f$ a ``shortest path'' \citep[Chapter 2.5]{BBI} and reserve the term ``geodesic'' for a path that \emph{locally} satisfies the above property. However, the above definition is standard in the field of optimal transport \citep[Section 9.2, Eq. 9.4]{ambrosio_lectures_2021}.

    While valuable in illuminating the geometry of a space, geodesics are not solely of geometric interest. A geodesic between two points gives us a way to interpolate or average between them. If a family $\{P_t\}_{t\in[0,1]}$ of problems defines a geodesic in the space of problems, then $P_{1/2}$ could reasonably be considered an average of $P_0$ and $P_1$. If we would rather take a weighted average, we could choose smaller or larger values of $t$ to give more weight to $P_0$ or to $P_1$ respectively.

    We will show that geodesics can be supplied by certain couplings and correspondences.
    
    \begin{definition}
        Let $P$ and $P'$ be problems.
        If there exists a coupling $\gamma\in \Pi(\eta,\eta')$ and a correspondence $R\in \cC(H,H')$ such that
        \begin{align*}
            \dexp(P,P') = \dis_{P,P'}(R,\gamma),
        \end{align*}
        then we call $\gamma$ an \emph{optimal coupling} and $R$ an \emph{optimal correspondence} for $P$ and $P'$. That is, $\gamma$ and $R$ are called optimal if they achieve the infima in the definition of $\dexp(P,P')$.
    \end{definition}
    In Section~\ref{subsec:optimal-couplings}, we will explore sufficient conditions for the existence of optimal couplings and correspondences. Presently, we motivate our interest in optimal couplings and correspondences by showing that they are closely related to the geometry of the space of problems; optimal couplings and correspondences provide us with explicit geodesics in the space of problems.

    \begin{theorem}\label{thm:geodesic-existence}
        Suppose $P_0$ and $P_1$ admit an optimal coupling $\gamma$ and an optimal correspondence $R$. Then the following family of problems is a geodesic between $P_0$ and $P_1$:
        \[P_t := (X_0\times X_1,Y_0\times Y_1,\gamma,\ell_t,R_\times)\]
        for all $t\in [0,1]$, where
        \begin{itemize}
            \item $\ell_t((y_0,y_1),(y'_0,y'_1)) := (1-t)\ell_0(y_0,y'_0) + t\ell_1(y_1,y'_1)$.
            \item $R_\times := \{h_0\times h_1 | h_0\in H_0, h_1\in H_1\}$.
        \end{itemize}
    \end{theorem}
    It may at first appear that we have given $P_0$ two conflicting definitions in the above statement, once as the given
    \[P_0 := (X_0,Y_0,\eta_0,\ell_0,H_0)\]
    and again by setting $t=0$ in our parameterized family, yielding
    \[P_0 := (X_0\times X_1,Y_0\times Y_1,\gamma,\ell_0,R_\times).\]
    However, the latter problem is a simulation of the former via the projection maps $X_0\times X_1 \to X_0$ and $Y_0\times Y_1 \to Y_0$. Hence the two problems are distance zero apart under the Risk distance and can be safely identified. A similar line of thought applies to $P_1$.

    The proof that this family is a geodesic is very similar to the proofs of analogous results for the Gromov-Hausdorff and Gromov-Wasserstein distances \citep{chowdhury_explicit_2018,sturm_space_2012}.
    \begin{proof}
        Define $C:= \dexp(P_0,P_1)$. Let $0\leq s\leq t \leq 1$. If we can show
        \[\dexp(P_s,P_t) \leq C(t-s),\]
        then equality will follow. Indeed, if the inequality were strict for some $0\leq s_0 \leq t_0 \leq 1$, then we could use the triangle inequality to write
        \begin{align*}
            C = \dexp(P_0,P_1)
            & \leq 
            \dexp(P_0,P_{s_0}) + \dexp(P_{s_0},P_{t_0}) + \dexp(P_{t_0},P_1)
            \\ & <
            s_0 C + (t_0 - s_0) C + (1-s_0)C = C,
        \end{align*}
        which is a contradiction.

        Now, we can bound $\dexp(P_s,P_t)$ by selecting the diagonal coupling from $\Pi(\gamma,\gamma)$ and diagonal correspondence from $\cC(R_\times,R_\times)$ to get
        \begin{align*}
            & \dexp(P_s,P_t)\\
            \leq &
            \sup_{h_0\times h_1 \in R_\times}
            \int \big| \ell_s((h_0(x_0),h_1(x_1)), (y_0,y_1)) - \ell_t((h_0(x_0),h_1(x_1)), (y_0,y_1))\big | \, \gamma(dx_0{\times} dy_0 {\times} dx_1 {\times}  dy_1)
            \\ = &
            \sup_{h_0\times h_1 \in R_\times}
            \int \big| (t-s)  \big(\ell_0(h_0(x_0),y_0) + \ell_1(h_1(x_1),y_1) \big )\big | \, \gamma(dx_0{\times} dy_0 {\times} dx_1 {\times}  dy_1)
            \\ &
            = (t-s) \sup_{(h_0,h_1) \in R} 
            \int \big|\ell_0(h_0(x_0),y_0) + \ell_1(h_1(x_1),y_1) \big | \, \gamma(dx_0{\times} dy_0 {\times} dx_1 {\times}  dy_1)
            \\ &
            = (t-s) C.
        \end{align*}
        The last equality is because $\gamma$ and $R$ are an optimal coupling and correspondence of $P_0$ and $P_1$.
    \end{proof}

\subsection{Existence of Optimal Couplings and Correspondences}\label{subsec:optimal-couplings}
    Motivated by Theorem~\ref{thm:geodesic-existence}, we search for conditions under which optimal couplings and correspondences exist. We will achieve this by imposing some continuity and compactness conditions on our problems.
    
    \begin{definition}
    A problem $P$ is called \emph{continuous} if, for all $h\in H$, $\ell_h$ is continuous $\eta$-almost everywhere.
    \end{definition}

    Note that the above condition applies to many relevant problems. A typical linear regression problem will be continuous. For instance, the linear regression problem in Example~\ref{ex:linear-regression} is continuous, since the loss is continuous, as is every predictor. Under reasonable assumptions, a classification problem can be continuous as well; one only needs the possible decision boundaries to each have measure zero in the $X$-marginal. Take the classification problem $P$ in Example~\ref{ex:linear-classification}. In that problem, $\ell_h$ is discontinuous for every predictor $h$. However, if we assume that any line in $\bbR^2$ has measure zero with respect to the $X$-marginal of $\eta$ (as would be the case if the $X$-marginal of $\eta$ is given by a continuous probability distribution), then any linear decision boundary has measure zero. Consequently, $\ell_h$ will be continuous (indeed, constant) $\eta$-almost everywhere.
    
    Let $P,P'$ be continuous problems. Recall that the risk distortion is the function
    \[
        \dis_{P,P'}: \cC(H,H') \times \Pi(\eta,\eta') \to \mathbb R
    \]
    \[
        \dis_{P,P'}(R,\gamma) := \sup_{(h,h')\in R} \int \big| \ell_h(x,y)-\ell_{h'}(x',y')\big| \, \gamma(dx{\times} dy {\times} dx'{\times} dy').
    \]
    We would like to discuss the potential continuity of this function, and hence require a suitable topology on its domain. We endow $\Pi(\eta,\eta')$ with the weak topology. We give the set $\set{\ell_h}{h\in H}$ the $L^1(\eta)$ topology, and endow $H$ with the initial topology induced by the map $h\mapsto \ell_h$. Equivalently, we endow $H$ with the pseudometric $d_{\ell,\eta}$ defined in Section~\ref{subsec:stability-H} and similarly equip $H'$ with $d_{\ell',\eta'}$.
    We then endow $H\times H'$ with the product metric $\delta((g,g'),(h,h')) := \max\{d_{\ell,\eta}(g,h), d_{\ell',\eta'}(g',h')\}$, and place the resulting Hausdorff distance $\dH^\delta$ on $\cC(H,H')$.

    Define
    \[
        \widetilde\dis_{P,P'}: H\times H' \times \Pi(\eta,\eta') \to \bbR
    \]
    \[
        \widetilde\dis_{P,P'}(h,h',\gamma) := \int \big| \ell_h(x,y)-\ell_{h'}(x',y')\big| \, \gamma(dx{\times} dy {\times} dx'{\times} dy').
    \]
    Hence
    \[\dis_{P,P'}(R,\gamma) = \sup_{(h,h')\in R} \widetilde\dis_{P,P'}(h,h',\gamma).\]

    We proceed with a sequence of lemmas. We begin by showing that, under certain conditions, $\widetilde\dis_{P,P'}$ is lower semi-continuous, then leverage this result to obtain lower semi-continuity of $\dis_{P,P'}$. Recall that a function $f:S\to \bbR$ on a topological space $S$ is \emph{lower semi-continuous} if, for all $x_0\in S$, we have $\liminf_{x\to x_0} f(x) \geq f(x_0)$. Lower semi-continuity is a useful notion because such functions satisfy a limited version of the extreme value theorem. Specifically, a lower semi-continuous function on a compact space achieves its minimum.

    \begin{lemma}\label{lem:distortion-continuous}
        Let $P$ and $P'$ be continuous problems. Then $\widetilde\dis_{P,P'}$ is lower semi-continuous.
    \end{lemma}

    \begin{lemma}\label{lem:full-distortion-continuous}
        Let $P$ and $P'$ be continuous problems. Then $\dis_{P,P'}$ is lower semi-continuous.
    \end{lemma}

    Using the fact that a lower semi-continuous function on a compact set achieves its minimum, we can prove the following Theorem.

    \begin{theorem}\label{thm:infimum-achieved}
        If $P$ and $P'$ are continuous problems with compact predictor sets, then there exists an optimal coupling and correspondence for $P$ and $P'$. That is, there exists $R\in \cC(H,H')$ and $\gamma \in \Pi(\eta,\eta')$ such that
        \begin{align*}
        \dexp(P,P') = \dis_{P,P'}(R,\gamma).
        \end{align*}
    \end{theorem}
    The proofs can be found in Appendix~\ref{appendix:proofs}.
    In turn, Theorem~\ref{thm:infimum-achieved} and the characterization of weak isomorphism provided by Theorem~\ref{prop:weak-iso-optimal} give us conditions under which weak isomorphism is equivalent to being Risk distance zero apart.
    \begin{corollary}\label{cor:compact-weak-isomorphism-equivalence}
        If $P$ and $P'$ are continuous problems with compact predictor sets, then $P$ and $P'$ are weakly isomorphic if and only if $\dexp(P,P')=0$.
    \end{corollary}

\subsection{Density of Classification Problems}\label{subsec:density-of-classification}
    One benefit of equipping a set with a pseudometric is that, in return, we receive a notion of convergence. One can ask whether a sequence of points converges, how quickly it does so, and what the limit might be. Density results pave the way for limiting arguments. It is often possible, by a limiting argument, to transfer knowledge about a ``simple'' dense subspace to the space as a whole.

    In this section, we show that classification problems (that is, problems whose output space is a finite collection of labels) are dense in a broader class of functions under $\dexp$. We will leverage this information in Section~\ref{subsec:empirical-convergence} to prove the main result of the section: the convergence of empirical problems (Theorem~\ref{thm:empirical-convergence}).

    \begin{definition}[Classification Problem]\label{def:classification-and-compact-problems}
        A problem is called a \emph{classification problem} if its response space is finite. A problem is called \emph{compact} if its response space is compact and its loss function is continuous.
    \end{definition}
    A classification problem is so called because we can think of its response space as a finite collection of labels. In Appendix~\ref{appendix:weak-iso}, we characterize the problems that are weakly isomorphic to classification problems.

    \begin{figure}
        \begin{minipage}{0.35\textwidth}
        \includegraphics[width=\textwidth]{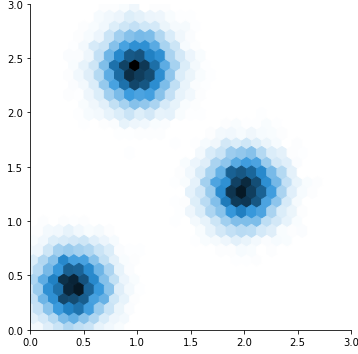}
        \end{minipage}
        $\qquad\underset{\text{Coarsening}}{\Scale[4]{\rightsquigarrow}}\qquad$
        \begin{minipage}{0.35\textwidth}
        \includegraphics[width=\textwidth]{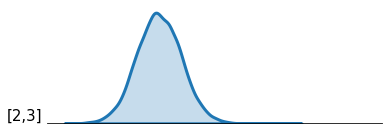}
        \includegraphics[width=\textwidth]{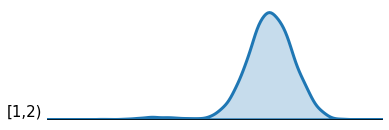}
        \includegraphics[width=\textwidth]{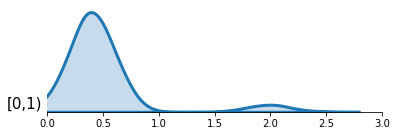}
        \end{minipage}
        \caption{ A depiction of the coarsening process in Example~\ref{ex:coarsening}. On the left is a heat map depicting the joint law $\eta$ of a problem $P$ whose input and output spaces are both the interval $X=Y=[0,3]$. We consider the partition $Q=\{[0,1), [1,2), [2,3]\}$ of $Y=[0,3]$. The coarsened problem $P_Q := (X,Q,\eta_Q,\ell_Q,H_Q)$ has three points in the output space $Q$, represented by the vertical axis. The joint law $\eta_Q$ is then a probability measure on three copies of $[0,3]$. We depict $\eta_Q$ on the right with density plots of the three class-conditional probability measures. 
        }
        \label{fig:coarsening}
    \end{figure}

    Given a compact problem $P$, we can produce a related classification problem $P_Q$ by clustering the response space into a finite partition $Q$ and applying a procedure we call \emph{coarsening}.
    The definition of a coarsening makes use of a certain map associated to a partition. Given a set $A$ and a partition $Q$ of $A$, consider the map $\pi_Q:A\to Q$ which sends each $a\in A$ to the unique partition block $q\in Q$ which contains $a$. That is, $\pi_Q$ is defined such that $a\in \pi_Q(a)$ for all $a\in A$. This map $\pi_Q$ is often called the \emph{quotient map} associated to $Q$.
    \begin{definition}[Coarsening] \label{def:coarsening}
        Let $P$ be a problem, let $Q$ be a finite partition of $Y$ into measurable sets, and let $\pi_Q:Y\to Q$ be the quotient map associated to $Q$.
        
        The \emph{coarsening of $P$ with respect to $Q$} is the \emph{classification} problem
        \[P_Q = (X,Q,\eta_Q,\ell_Q,H_Q)\]
        where
        \begin{itemize}
          \item $\eta_Q := (\id_X \times \pi_Q)_\sharp \eta$
          \item $\ell_Q(q,q') := \sup_{y\in q, y'\in q'} \ell(y,y')$,
          \item $H_Q := \set{\pi_Q\circ h}{h\in H}$.
        \end{itemize}
    \end{definition}

Note that if $Q$ is a finite partition of $Y$, then $Q$ is itself a finite set of partition blocks. The response space $Q$ of the coarsened problem $P_Q$ is a finite set with one element for each block of the partition $Q$, and hence $P_Q$ is indeed a classification problem.

\begin{examples}\label{ex:coarsening}
    An example of the coarsening procedure is depicted in Figure~\ref{fig:coarsening}. We imagine a problem
    \[P = ([0,3], [0,3], \eta,\ell,H)\]
    where $\eta$ is as pictured, $\ell(y_1,y_2):=|y_1 - y_2|$ is the Euclidean distance, and $H$ is the set of increasing continuous bijections $[0,3] \to [0,3]$. We choose the partition $Q = \{[0,1), [1,2), [2,3]\}$ of $P$ and consider the coarsened problem
    \[P_Q = ([0,3], Q, \eta_Q, \ell_Q, H_Q).\]
    Then $P_Q$ is a classification problem with three points in its output space $Q = \{q_1,q_2,q_3\}$, where $q_1:=[0,1), q_2:=[1,2),$ and $q_3:=[2,3]$.
    Then $\eta_Q$ is a probability measure on $[0,3]\times Q$, a space which consists of three line segments: $[0,3]\times \{q_1\}$, $[0,3]\times \{q_2\}$, and $[0,3]\times \{q_3\}$. 
    
    The coarsened loss $$\ell_Q(q_i,q_j) := \sup_{y\in q_i, y'\in q_j} \ell(y,y') \,\,\mbox{for $i,j=1,2,3$},$$ can be written in the form of a table:
    \begin{center}
        \begin{tabular}{|c|c c c| } 
        \hline
        $\ell_Q$ & $q_1$ & $q_2$ & $q_3$ \\
        \hline
        $q_1$ & 1 & 2 & 3 \\ 
        $q_2$ & 2 & 1 & 2 \\ 
        $q_3$ & 3 & 2 & 1 \\ 
        \hline
        \end{tabular}
    \end{center}
    The predictor set $H_Q$ is $\set{\pi_Q \circ h}{h\in H}$, where $\pi_Q : Y\to Q$ is the quotient map. Then $H_Q$ can also be written as the set of all functions $g:[0,3] \to Q$ of the form
    \[g(t) = \begin{cases}
        q_1 & t < a\\
        q_2 & a \leq t < b \\
        q_3 & b \leq t
    \end{cases}\]
    for all $0<a \leq b <3$. 
\end{examples}
    
    Coarsening is reminiscent of the ``probing reduction'' \citep{langford_probing_2005}, which takes a probability estimation problem and solves it with an ensemble of classifier learners, hence ``reducing'' the problem of probability estimation to the problem of binary classification. Indeed, one can view the probing reduction as a practical method of approximating a problem whose output space is $Y=[0,1]$ by iteratively constructing finer and finer partitions of $Y$ and solving the resulting coarsened problems with binary classifiers.

    \begin{theorem}\label{thm:coarsen-approx}
    Let $P$ be a compact problem. For any
    $\epsilon>0$, there is a partition $Q$ of the response space such that
    \[\dexp(P,P_Q)<\epsilon.\]
    \end{theorem}
    The proof is relegated to Appendix~\ref{appendix:proofs}.
    \begin{remark}\label{remark:coarsening-distribution-independent}
        One notable feature of the proof of Theorem~\ref{thm:coarsen-approx} is that the partition $Q$ provided by the theorem does not depend on the joint law $\eta$ of $P$. More precisely, suppose $P$ is a problem. For any probability measure $\eta'\in \prob(X\times Y)$, we can create a new problem
        \[P_{\eta'} := (X,Y,\eta', \ell, H).\]
        The proof of Theorem~\ref{thm:coarsen-approx} shows that, not only can we find a partition $Q$ of $Y$ such that $\dexp(P,P_Q)<\epsilon$, but we can choose $Q$ such that for all $\eta'\in \prob(X\times Y)$, we have $\dexp((P_{\eta'})_Q,P_{\eta'})<\epsilon$. We will use this stronger fact to prove Theorem~\ref{thm:empirical-convergence}.
    \end{remark}

    Theorem~\ref{thm:coarsen-approx} immediately gives us a density result.

    \begin{corollary}[Classification Density Theorem]\label{cor:finite-dense-in-compact}
    Under $\dexp$, the set of classification problems is dense in the space of compact problems.
    \end{corollary}

\subsection{Convergence of Empirical Problems}\label{subsec:empirical-convergence}
    With the density of classification problems in hand, we are ready to examine the effect that limited sampling can have on a problem. One often does not have direct knowledge of the underlying joint law from which data is sampled; one only has access to a finite data set. In terms of our supervised learning problems, instead of accessing the joint law $\eta$ of a problem $P$ directly, we can only access the empirical measure $\eta_n$ of $\eta$. Namely, given an i.i.d. set $\{(x_i,y_i)\}_{i=1}^n$ sampled from the probability measure $\eta$, we set $$\eta_n := \frac{1}{n}\sum_{i=1}^n \delta_{(x_i,y_i)}.$$
    The resulting problem with joint law $\eta_n$ is called the \emph{empirical problem}. 

    \begin{definition}
        If $P$ is a problem, the \emph{$n$th empirical problem induced by $P$} is the random problem defined by
        \[P_n := (X,Y,\eta_n,\ell,H).\]
    \end{definition}

    Being dependent on a random sample, $\eta_n$ is a random measure. Consequently, $P_n$ is a random problem. One would hope that as the size of the sample grows, the empirical problem $P_n$ converges to the true problem $P$ with probability 1. Under certain conditions, this is indeed true.

    \begin{theorem}[Convergence of the Empirical Problem]\label{thm:empirical-convergence}
        Let $P$ be a problem satisfying the following:
        \begin{enumerate}
            \item $H$ is compact with respect to $d_{\ell,\eta}$.
            \item $Y$ is compact.
            \item $\ell$ is continuous.
            \item The collection $\{|\ell_h - \ell_{h'}|\}_{h,h'\in H}$ forms a Glivenko-Cantelli class with respect to $\eta$. 
        \end{enumerate}
        Let $P_n$ be the $n$th empirical problem induced by $P$. Then $\dexp(P,P_n) \to 0$ almost surely.
    \end{theorem}

    The proof, found in Appendix~\ref{appendix:proofs}, uses Theorem~\ref{thm:coarsen-approx} to reduce to the case where $P$ is a classification problem. We briefly illuminate the Glivenko-Cantelli condition in Theorem~\ref{thm:empirical-convergence} by providing simpler sufficient conditions.
    
    \begin{proposition}\label{prop:GC-sufficient}
        Define $g_{h,h'}(x,y) := | \ell_h(x,y) - \ell_{h'}(x,y)|$.
        Theorem~\ref{thm:empirical-convergence} still holds if Condition~4 is replaced with either of the following conditions.
        \begin{enumerate}
            \item[4a.] The collection of subgraphs $\set{\mathrm{sg}(g_{h,h'})}{h,h'\in H}$ (see Definition \ref{def:subgraph})
            has finite VC dimension. 
            \item[4b.] For any $h_0,h'_0\in H$ and $\epsilon>0$, there exists an open neighborhood $(h_0,h'_0)\in U \subseteq H\times H$ such that
            \begin{align*}
                \int \sup_{(h,h')\in U} g_{h,h'}(x,y) \eta(dx{\times} dy) -\epsilon
                & <
                \int g_{h_0,h'_0}(x,y) \eta(dx{\times} dy)
                \\&<
                \int \inf_{(h,h')\in U} g_{h,h'}(x,y) \eta(dx{\times} dy) + \epsilon.
            \end{align*}
        \end{enumerate}
    \end{proposition}

    Conditions~2 and 4a together imply condition 4, since the compactness of $Y$ implies that the $g_{h,h'}$ are uniformly bounded, and hence that Proposition~\ref{prop:VC-sufficient} applies. Note that Condition~4a does not depend on the joint law $\eta$. Conditions~1 and 4b together imply condition 4 as well by a theorem of \citet{gaenssler_1975}.

    Theorem~\ref{thm:empirical-convergence} is a convergence result. A quantitative result about the \emph{rate} of convergence would be more useful. Certainly the convergence rate will depend on properties of $P$ and $P'$; the proof of Theorem~\ref{thm:empirical-convergence} suggests that the proximity of $P$ and $P'$ to less complex problems is crucial.

    \begin{question}\label{q:conv-thm}
        How does the rate of the convergence given by Theorem~\ref{thm:empirical-convergence} depend on the properties of $P$ and $P'$? Can one bound the convergence rate for specific classes of problems?
    \end{question}

\section{The $L^p$-Risk Distance: Weighting Problems}\label{sec:probabilistic}

In this section we introduce a variant of the Risk distance called the \emph{$L^p$-Risk distance} which incorporates probability measures on the predictor sets. This additional information lets us prioritize certain predictors over others by weighting them more heavily, subsequently softening the Risk distance and improving its stability. The weighting can be interpreted as a Bayesian prior, so the perspective adopted in this section is broadly aligned with the principles of Bayesian statistics. While approaches such as Bayesian decision theory \citep{berger1985statistical} incorporate elements like loss functions, prior work does not operate on a space of problems, nor does it attempt to endow such a space with a metric structure as we do in this paper.

\subsection{Weighting Predictor Sets and the $L^p$-Risk Distance}

Many of our results thus far have required assumptions of compact predictor sets or bounded loss functions.
Without these assumptions, the Risk distance tends to be poorly-behaved.
For instance, problems such as the linear regression in Example~\ref{ex:linear-regression} tend to be infinitely far away from problems of bounded loss under the Risk distance.
One reason for this phenomenon is the supremum in the definition of the Risk distance.
The supremum treats unreasonable predictors, such as those whose graphs are far from the support of the joint law, with the same consideration as more reasonable options.
One might reasonably complain about this design decision.

One solution to this problem is to \emph{weight} the predictors, de-prioritizing extreme options by giving them less weight.
These weights could come from prior knowledge or belief about which predictors might perform well or from a measure of complexity of each predictor (so that more complex predictors receive lower weights). Indeed, considering these weights as part of the specification of a problem can be seen as a  proxy for adding a regularization term to the constrained Bayes risk.

More formally, we provide weights on the predictor set in the form of a probability measure.
This lets us replace the supremum in the definition of $\dexp$ with an $L^p$-type integral, effectively softening the Risk distance.

\begin{definition}
    A \emph{weighted problem} is a pair $(P,\lambda)$ where
    \begin{itemize}
        \item $P$ is a problem whose predictor set $H$ comes equipped with a Polish topology, and
        \item $\lambda$ is a probability measure on $H$.
    \end{itemize}
    When there is no chance of confusion, we may suppress the measure $\lambda$ and denote the weighted problem by $P$. 
\end{definition}
\begin{assumption}
    We make the following measurability assumption. Given a weighted problem $(P,\lambda)$, consider the map 
    \begin{align*}
        H\times X\times Y & \to \bbR_{\geq 0}
        \shortintertext{given by}
        (h,x,y) &\mapsto \ell_h(x,y).
    \end{align*}
    Since $H$, $X$, and $Y$ are topological spaces, we can equip each space with the corresponding Borel $\sigma$-algebra. We can then equip the product $H\times X \times Y$ with the product $\sigma$-algebra. For any weighted problem $(P,\lambda)$, we henceforth assume that the above map $(h,x,y)\mapsto \ell_h(x,y)$ is measurable with respect to this product $\sigma$-algebra.
\end{assumption}

\begin{definition}
    Let $(P,\lambda)$ and $(P',\lambda')$ be weighted problems. For $p\in[1,\infty]$ and  
    for any couplings $\gamma \in \Pi(\eta,\eta')$ and $\rho \in \Pi(\lambda,\lambda')$, define the \emph{$L^p$-risk distortion} of $\gamma$ and $\rho$ to be
    \begin{align*}
         \dis_{P,P',p}(\rho,\gamma) &:= \left(\int \left(\int \big|\ell(h(x),y) - \ell'(h'(x'),y')\big| \, d\gamma(x,y, x',y') \right)^{p}d\rho(h,h')\right)^{1/p}&\mbox{ $p\in[1,\infty)$,}\\
         \dis_{P,P',\infty}(\rho,\gamma) &:= \sup_{(h,h')\in\supp[\rho]} \int \big|\ell(h(x),y) - \ell'(h'(x'),y')\big| \, d\gamma(x,y, x',y')  & \mbox{$p=\infty$.}
    \end{align*}
       The \emph{$L^p$-Risk distance} between weighted problems $(P,\lambda)$ and $(P',\lambda')$ is then defined to be
    \begin{align*}
        \dexp[,p]((P,\lambda),(P',\lambda')) & := \inf_{\rho, \gamma} \dis_{P,P',p}(\rho,\gamma)
    \end{align*}
    where $\gamma$ ranges over all couplings $\Pi(\eta,\eta')$ and $\rho$ ranges over all couplings $\Pi(\lambda,\lambda')$.
    
\end{definition}
\begin{theorem}
The $L^p$-Risk distance is a pseudometric on the collection of all weighted problems. Furthermore, we have 
\begin{itemize}
    \item[(i)] $\dexp[,p]\leq \dexp[,q]\leq \dexp[,\infty]$ for all $1\leq p\leq q \leq \infty$;
    \item[(ii)] $\dexp\leq \dexp[,\infty]$.
\end{itemize}
\end{theorem}
\begin{remark} Property (i) is analogous to a similar property enjoyed by the Gromov-Wasserstein distance \citep[Theorem 5.1 (h)]{memoli_gromovwasserstein_2011} and it follows from \citep[Lemma 2.2]{memoli_gromovwasserstein_2011}. Property (ii) states that the $L^\infty$-Risk distance is never smaller than the standard Risk distance from Definition \ref{def:metric}. This is analogous to the fact that the Gromov-Hausdorff distance is bounded above by the $L^\infty$-Gromov-Wasserstein distance \citep[Theorem 5.1 (b)]{memoli_gromovwasserstein_2011}. The proof of the triangle inequality for the $L^p$-Risk distance is identical to the proof of the same for the Risk distance, save for an application of Minkowski's inequality for the outer integral. As indicated above, the proof of the other claimed properties follows similar steps as those for the Gromov-Wasserstein distance and therefore we omit them.
\end{remark}

\begin{examples}\label{ex:one-point-weighted-problem}
    Just as we defined the one-point problems
    \[P_\bullet(c) := (\{\bullet\},\{\bullet\}, \delta_{(\bullet,\bullet)}, c, \{\id_{\{\bullet\}}\})\]
    for each $c\geq 0$ in Example~\ref{ex:one-point-problem}, we can define a corresponding \emph{one-point weighted problem} by endowing $\{\id_{\{\bullet\}}\}$ with the point-mass probability measure. We will still use $P_\bullet$ to denote $P_\bullet(0)$. Then, in a manner similar to Example \ref{ex:bayes-one-point}, if $(P,\lambda)$ is a weighted problem, one can show
    \begin{align*}
        \dexp[,p](P,P_\bullet) = \left(\int_H\left(\int_{X\times Y}\ell(h(x),y)\,\eta(dx\times dy)\right)^p \, \lambda(dh)\right)^{1/p} = \left(\int_H \cR_P^p(h) \, \lambda(dh)\right)^{1/p}.
    \end{align*}
    This can be seen as a kind of weighted constrained Bayes risk of $(P,\lambda)$. While the constrained Bayes risk of an unweighted problem $P$ is the infimal possible risk, the above expression $\dexp[,p](P,P_\bullet)$ represents the $p$th moment of possible risks when a predictor is chosen at random according to $\lambda$. In particular, $\dexp[,1](P,P_\bullet)$ is simply the average risk of a predictor chosen randomly according to $\lambda$. 

    One can establish a connection with the usual constrained Bayes risk as follows. Assume that the problem $P$ is such that its loss function $\ell$ is bounded above by a constant $\ell_{\max}>0 $. Then, $$\dexp[,p](P,P_\bullet(\ell_{\max}))=\left(\int_H \big(\ell_{\max}-\cR_P(h)\big)^p\lambda(dh)\right)^{1/p}.$$ If the support of $\lambda$ is the whole $H$, we obtain, as $p\to\infty$, that $\dexp[,p](P,P_\bullet(\ell_{\max}))\to \ell_{\max}-\cB(P).$ 
\end{examples}

The distance $\dexp[,p](P,P_\bullet)$ from Example~\ref{ex:one-point-weighted-problem} will appear again in Theorem~\ref{thm:weighted-density}, so we establish a definition.
\begin{definition}
    Given a weighted problem $P = (P,\lambda)$ and $1\leq p <\infty$, the \emph{$p$-diameter of $P$} is given by
    \[\dexp[,p](P,P_\bullet) = \left(\int_H\left(\int_{X\times Y}\ell(h(x),y)\,\eta(dx\times dy)\right)^p \, \lambda(dh)\right)^{1/p} = \left(\int_H \cR_P^p(h) \, \lambda(dh)\right)^{1/p}.\]
\end{definition}
The word ``diameter'' is chosen in analogy with metric geometry, where the diameter of a metric space can be found from its Gromov-Hausdorff distance from the one-point metric space.

\subsection{Connection to the Gromov-Wasserstein Distance}\label{subsec:OT-connection-weighted}

In Section~\ref{subsec:OT-connection}, we saw two ways to turn metric measure spaces into problems, each of which relates the Gromov-Wasserstein distance to the Risk distance. In this section, we will draw a similar connection between the Gromov-Wasserstein distance and the $L^p$-Risk distance. Specifically, we will provide a way of turning metric measure spaces into weighted problems such that the $L^1$-Risk distance between the resulting problems aligns with a well-known modification of the Gromov-Wasserstein distance between the original spaces. 

Given any metric measure space $(X,d_X,\mu_X)$, define a weighted problem $(P_X, \mu_X)$,
\[P_X := (X, X, (\Delta_X)_\sharp \mu_X, d_X, H_X),\]
where
\begin{itemize}
    \item $\Delta_X:X\to X\times X$ is the diagonal map $x\mapsto (x,x)$, and
    \item $H_X$ is the set of all constant functions $X\to X$.
\end{itemize}

Here we think of $\mu_X$ as a measure on both $X$ and $H_X$ since the two are in natural bijection.

\begin{proposition} \label{prop:conn-w-dGW}
Let $(X,d_X,\mu_X)$ and $(Y,d_Y,\mu_Y)$ be two metric measure spaces. Then, 
\begin{equation}\dexp[,1](P_X,P_Y)=\inf_{\substack{\gamma \in \Pi(\mu_X, \mu_Y)\\
    \rho \in  \Pi(\mu_X, \mu_Y)}}
    \int \int \big | d_X(x_2,x_1) - d_Y(y_2,y_1) \big| \, \gamma(dx_1{\times}dy_1)\,\rho(dx_2{\times}dy_2)\tag{$\star$}\label{expression:bilinear}.\end{equation}
\end{proposition}
The proof of this proposition can be found in Appendix \ref{appendix:proofs}.

Use $\cF(\gamma,\rho)$ to denote the double integral in the expression \eqref{expression:bilinear}. If we add the additional constraint that $\gamma = \rho$, then the expression \eqref{expression:bilinear} becomes $2\dGWp{1}(X,Y)$. As written, \eqref{expression:bilinear} is a lower bound on $2\dGWp{1}(X,Y)$, which could be arrived at via a ``bilinear relaxation'' of the optimization problem posed by $\dGWp{1}$. That is, one could begin with the objective function for the Gromov-Wasserstein optimization problem
\[\cF'(\gamma) := \int \int \big | d_X(x_2,x_1) - d_Y(y_2,y_1) \big| \, \gamma(dx_1{\times}dy_1)\,\gamma(dx_2{\times}dy_2),\]
which is quadratic in the transport plan $\gamma$, and decouple the two instances of $\gamma$ to get the objective $\cF(\gamma,\rho)$, which is bilinear in two transport plans. This ``relaxation" was  considered by \cite[Section 7]{memoli_gromovwasserstein_2011} for estimating the Gromov-Wasserstein distance via an alternate minimization procedure. Similar optimal transport problems, which simultaneously optimize two transport plans in a Gromov-Wasserstein-like objective function, appear in the literature. Indeed, when $X$ and $Y$ are finite, the expression \eqref{expression:bilinear} is a special case of CO-Optimal transport \citep{titouan_co-optimal_2020}. \citet{sejourne_unbalanced_2021} apply this relaxation technique, which they call the ``biconvex relaxation,'' to an unbalanced variant of the Gromov-Wasserstein distance. This inspired \citet{beier_multi-marginal_2023} to explore a similar relaxation for a ``multi-marginal'' variant of the same problem. A theoretical study of the metric properties of the right-hand side of \eqref{expression:bilinear} appears in \citep[Appendix A.4]{chen2022weisfeiler}.

\subsection{Stability of Loss Profile Distribution}
Being softer than the Risk distance, the loss profile is less stable under the $L^p$-Risk distance. To obtain a stability result, one cannot compare the loss profile sets directly, but instead compare certain probability distributions on the loss profile sets.
\begin{definition}
    Let $(P,\lambda)$ be a weighted problem. Define $F_P:H\to \mathbb \prob(\bbR)$ by
    \[F_P(h) := (\ell_h)_\sharp\eta.\]
    The \emph{loss profile distribution of $P$} is then defined to be $\cL(P) := (F_P)_\sharp \lambda$. 
\end{definition}
Note that $\cL(P)\in \prob(\prob(\bbR))$ is a probability measure on the space of real probability measures. More specifically, $F_P(h)$ is the loss profile of $h$, so $\cL(P)$ describes what the loss profile of a random predictor chosen according to $\lambda$ is likely to be.

While Theorem~\ref{thm:loss-control} shows that $L(P)$ is stable under the Risk distance with respect to the Hausdorff distance, the following analogous theorem shows that the measure $\cL(P)$ is stable under the $L^p$-Risk distance with respect to the Wasserstein distance.
\begin{theorem}\label{thm:loss-control-weighted}
    Let $P$ and $P'$ be weighted problems. Then
    \[\dWp{p}^{\dW}(\cL(P),\cL(P')) \leq \dexp[,p](P,P')\]
    where $\dWp{p}^{\dW}$ is the $p$-Wasserstein metric on $\prob(\prob(\bbR))$ with underlying metric $\dW$ on $\prob(\bbR)$.
\end{theorem}
The proof is in Appendix~\ref{appendix:proofs}.

\subsection{Improved Density}
A useful theorem in measure theory states that any probability measure $\mu$ on a Polish space $X$ is \emph{tight}, meaning that for any $\epsilon>0$, there is a compact subset $K\subseteq X$ with $\mu(K)>1-\epsilon$. In a slogan, ``Polish probability spaces are almost compact.'' Applying this principle to our setting suggests that any weighted problem should be close to one with a compact predictor set. This can be made rigorous with a density theorem.

\begin{theorem}\label{thm:weighted-density}
Let $p \geq 1$ and let $(P,\lambda)$ be a weighted problem of finite $p$-diameter. Then there exists a compact weighted problem $(P',\lambda')$ with $\dexp[,p](P,P')<\epsilon$.
\end{theorem}
Additionally, a result analogous to Theorem~\ref{thm:coarsen-approx} still holds for the $L^p$-Risk distance.
\begin{proposition}\label{prop:finite-dense-in-compact-weighted}
    Let $p \geq 1$ and let $(P,\lambda)$ be a compact weighted problem. Then there exists a weighted classification problem $(P',\lambda')$ with $\dexp[,p](P,P')<\epsilon$.
\end{proposition}
The proofs of both Theorem~\ref{thm:weighted-density} and Proposition~\ref{prop:finite-dense-in-compact-weighted} can be found in Appendix~\ref{appendix:proofs}. Chaining these two density results together proves the following corollary.
\begin{corollary}\label{cor:finite-dense-in-p-diam}
    Under $\dexp[,p]$ for $p\geq 1$, weighted classification problems are dense in the space of weighted problems of finite $p$-diameter.
\end{corollary}

\begin{examples}
    Consider the linear regression problem $P$ from Example~\ref{ex:linear-regression}. If we endow the predictor set of $P$ with a joint law satisfying some finite moment assumptions, the resulting weighted problem has finite $p$-diameter.
    Corollary~\ref{cor:finite-dense-in-p-diam} then tells us that $P$ can be approximated arbitrarily well by weighted classification problems under $\dexp[,p]$.

    More precisely, let $P$ be a linear regression problem, and suppose the predictor set
    \[H = \set{x\mapsto a^T x + b\,}{\,(a,b)\in \bbR^n\times \bbR}\] is equipped with a probability measure $\lambda$. We can think of both $\lambda$ and $\eta$ as probability measures on $\bbR^n\times \bbR$. If
    \begin{enumerate}
    \item all $n+1$ marginals of $\eta$ have finite second moment, and
    \item all $n+1$ marginals of $\lambda$ have finite $2p$-th moment,
    \end{enumerate}
    then $\diam_p(P) < \infty$.

    To prove this, we can write
    \begin{align*}
    \big(\op{diam}_p(P)\big)^{1/2}
    & = \left(
        \int \left(
        \int \ell(a^T x+b,y) \, \eta(dx\times dy)
        \right)^p \, \lambda(da\times db)
    \right)^{1/2p}
    \\ & = \left(
        \int \left( \left(
        \int (a^T x+b-y)^2 \, \eta(dx\times dy)
        \right)^{1/2}\right)^{2p} \, \lambda(da\times db)
    \right)^{1/2p}.
    \end{align*}
    By liberally applying Minkowski's inequality and simplifying, we can bound this expression by the sum
    \begin{align*}
    &\sum_{i=1}^n\left(
        \int a_i^{2p} \, \lambda(da\times db)
    \right)^{1/2p}
    \left(
        \int x_i^2 \, \eta(dx\times dy)
        \right)^{1/2}
    \\ + &
    \left(
        \int \left( b\right)^{2p} \, \lambda(da\times db)
    \right)^{1/2p}
    +
    \left(
        \int y^2 \, \eta(dx\times dy)
        \right)^{1/2}.
    \end{align*}
    These integrals are the moments of $\eta$ and $\lambda$ that we have assumed to be finite. 
\end{examples}

\subsection{Improved Stability}\label{subsec:improved-stability} %\facundo{(Facundo: Check: conns with L-infty version)}
The $L^p$-Risk distance enjoys many of the same stability and density results as the Risk distance. Indeed, we will soon see that many of these results can be \emph{strengthened} for the $L^p$-Risk distance. That is, the $L^p$-Risk distance enjoys \emph{stronger} stability and density properties than the Risk distance.

First we will establish that the $L^p$-Risk distance inherits some stability results from the Risk distance, saving the improvements for later.
Many of our bounds on the Risk distance are proven using the bound
\begin{equation}
\dexp(P,P')\leq \inf_{\gamma \in \Pi(\eta,\eta')} \dis_{P,P'}(R_{\mathrm{diag}}, \gamma)\label{ineq:diagonal_bound}
\end{equation}
where $R_{\mathrm{diag}}$ is the diagonal correspondence, which is valid whenever $P$ and $P'$ share a predictor set.
Hence we can extend the results of those sections using the following lemma.
\begin{lemma}\label{lem:p-distance-stability}
    Let $(P,\lambda)$ and $(P',\lambda)$ be weighted problems that share a predictor set $H$ and predictor weights $\lambda$. Then
    \[\dexp[,p]((P,\lambda),(P',\lambda))
    \leq \inf_{\gamma\in \Pi(\eta,\eta')}\dis_{P,P'}(R_{\mathrm{diag}}, \gamma)\]
    where $R_{\mathrm{diag}}$ is the diagonal correspondence of $H$ with itself. 
    \end{lemma}

\begin{proof}
    First write
    \[\dexp[,p]((P,\lambda),(P',\lambda)) \leq 
    \dis_{P,P',p}(\rho_{\mathrm{diag}}, \gamma)\]
    where $\rho_{\mathrm{diag}}$ is the diagonal coupling of $\lambda$ with itself.
    
    Next we bound the $L^p$-risk distortion by the $L^\infty$-risk distortion as follows:
    \begin{align*}
        \dis_{P,P',p}(\rho_{\mathrm{diag}}, \gamma) 
        & \leq \dis_{P,P',\infty}(\rho_{\mathrm{diag}}, \gamma)
        \\ & = \sup_{(h,h')\in \supp{[\rho_{\mathrm{diag}}}]} \int\big| \ell(h(x),y) - \ell'(h'(x'),y')\big|\, d\gamma(x,y,x',y')
        \\ & \leq \sup_{(h,h') \in R_{\mathrm{diag}}} \int \big|\ell(h(x),y) - \ell'(h'(x'),y')\big| \, d\gamma(x,y, x',y')
        \\ & = \dis_{P,P'}(R_{\mathrm{diag}},\gamma),
    \end{align*}
    
    where the second inequality follows from the fact that the support of $\rho_{\mathrm{diag}}$ lies within the diagonal $R_{\mathrm{diag}}\subseteq H\times H$.
\end{proof}
Condition~2 from Theorem~\ref{thm:characterization}, the Risk distance bounds of Sections~\ref{subsec:stability-bias} (Stability under bias), \ref{subsec:stability-noise} (Stability under noise), and \ref{subsec:empirical-convergence} (Convergence of empirical problems) are all proven via the diagonal bound \eqref{ineq:diagonal_bound}.\footnote{In addition to the diagonal bound, the proof of Theorem~\ref{thm:empirical-convergence} also uses the density of classification problems (Theorem~\ref{cor:finite-dense-in-compact}). However, the proof of the analogous result for weighted problems is still valid if we replace this step with an application of the density of weighted classification problems (Theorem~\ref{prop:finite-dense-in-compact-weighted}).}
By Lemma~\ref{lem:p-distance-stability}, these results still hold if we replace the Risk distance with the $L^p$-Risk distance and replace both problems with weighted problems that share a common weight measure $\lambda$.
That is, in the above results, we can replace each instance of $\dexp(P,P')$ with $\dexp[,p]((P,\lambda), (P',\lambda))$ and they will still hold true.

Moreover, some of our results for the Risk distance can be strengthened for the $L^p$-Risk distance. Indeed, we have already seen that the density results of Section~\ref{subsec:density-of-classification} can be strengthened into Theorem~\ref{thm:weighted-density}. We can strengthen some stability results in the same manner. For instance, the $L^p$-Risk distance satisfies the following version of Condition~2 in which the supremum is replaced by an integral.
\begin{proposition}
    Let $(P,\lambda)$ and $(P',\lambda)$ be weighted problems with
    \begin{align*}
        P &= (X,Y,\eta,\ell,H)\\
        P' &= (X,Y,\eta,\ell',H)
    \end{align*}
    Then
    \begin{align*}
        \dexp[,p](P,P') \leq \left(\int \left( \int
            \big| \ell(h(x),y) - \ell'(h(x),y) \big | \,\eta(dx{\times} dy)\,
        \right)^p
        \lambda(dh)
        \right)^{1/p}.
    \end{align*}
\end{proposition}
\begin{proof}
    Bound the $L^p$-Risk distance by selecting the diagonal couplings.
\end{proof}

We can also modify Theorem~\ref{thm:H-stability} by replacing the Hausdorff distance with a Wasserstein distance.

\begin{proposition}
    Let $(P,\lambda)$ and $(P,\lambda')$ be weighted problems with
    \begin{align*}
        P &:= (X,Y,\eta,\ell,H)\\
        P' &:= (X,Y,\eta,\ell,H').
    \end{align*}
    Let $i:H\to H\cup H'$ and $j:H'\to H\cup H'$ be the inclusion maps, and endow $H\cup H'$ with the pseudometric $d_{\ell,\eta}(h_1,h_2):= \| \ell_{h_1} - \ell_{h_2}\|_{L^1(\eta)}$ introduced in Section~\ref{subsec:optimal-couplings}.
    
    Then
    \[\dexp[,p](P,P') \leq \dWp{p}^{d_{\ell,\eta}}(i_\sharp \lambda,j_\sharp \lambda').\]
\end{proposition}
\begin{proof}
    Selecting the diagonal coupling between $\eta$ and itself gives us the bound
    \begin{align*}
        \dexp[,p](P,P') &\leq
        \inf_{\rho\in \Pi(\lambda,\lambda')} \left( \int \left(\int
            \big| \ell_h(x,y) - \ell_{h'}(x,y)\big|\, \eta(dx{\times}dy)
        \right)^p \rho(dh{\times}dh')\right)^{1/p}
        \\ &
        =
        \inf_{\rho\in \Pi(i_\sharp \lambda,j_\sharp \lambda')} \left( \int \left(\int
            \big| \ell_h(x,y) - \ell_{h'}(x,y)\big|\, \eta(dx{\times}dy)
        \right)^p \rho(dh{\times}dh')\right)^{1/p}
        \\ &
        = \inf_{\rho\in \Pi(i_\sharp \lambda,j_\sharp \lambda')} \left(\int \left(d_{\ell,\eta}(h,h')\right)^p \rho(dh{\times}dh')\right)^{1/p}
        =\dWp{p}^{d_{\ell,\eta}}(i_\sharp \lambda,j_\sharp \lambda')
    \end{align*}
    as desired.
\end{proof}

Proposition~\ref{prop:H-wasserstein-bound} can be modified for weighted problems as well. We first construct an analog of the metric $s_{\ell,H}$ from Section~\ref{subsec:stability-bias}. Replacing the supremum over $H$ in the definition of $s_{\ell,H}$ with an $L^p$-style integral against $\lambda$, we arrive at the definition
\begin{align*}
    s_{\ell,\lambda,p}((x,y),(x',y'))
    :=& \left(\int_H\big|\ell(h(x),y) - \ell(h(x'),y') \big |^p \, \lambda(dh) \right)^{1/p}.
\end{align*}
In other words, to determine the distance between $(x,y)$ and $(x',y')$, we ask how different the loss on the observation $(x,y)$ will be from the loss on $(x',y')$ for an average predictor.
We can now state our stability result.
\begin{proposition}\label{prop:H-wasserstein-bound-weighted}
    Let $(P,\lambda)$ be a weighted problem, and let $(P',\lambda)$ be a weighted problem that is identical to $P$ except possibly for its joint law $\eta'$. Then
    \begin{align*}
        \dexp[,p](P,P') \leq \dW^{s_{\ell,\lambda,p}}(\eta,\eta').
    \end{align*}
\end{proposition}
\begin{proof}
    By selecting the diagonal coupling between $\lambda$ and itself and applying Minkowski's inequality for integrals, we can bound $\dexp[,p]$ by
    \begin{align*}
        \dexp[,p](P,P') & \leq \inf_{\gamma \in \Pi(\eta,\eta')}
        \left(
            \int \left(
                \int \big |
                \ell(h(x),y) - \ell(h(x'),y')
                \big |
                \, \gamma(dx{\times}dy{\times}dx'{\times}dy')
            \right)^p
            \lambda(dh)
        \right)^{1/p}
        \\ &
        \leq 
        \inf_{\gamma \in \Pi(\eta,\eta')}
        \int \left(
            \int \big |
            \ell(h(x),y) - \ell(h(x'),y')
            \big |^p
            \, \lambda(dh)
        \right)^{1/p}
        \gamma(dx{\times}dy{\times}dx'{\times}dy')
        \\&
        = \dW^{s_{\ell,\lambda,p}}(\eta,\eta'), %\qedhere
    \end{align*}
    giving us the desired bound.
\end{proof}

Lastly, we turn our attention to stability under noise.
We have shown in Section~\ref{subsec:stability-noise} that the Risk distance is stable with respect to noise in the joint law when $\ell$ is a metric.
By contrast, for a general $\ell$, the Risk distance $\dexp$ can be highly sensitive to changes in the joint law. The following example demonstrates this sensitivity.
\begin{examples}
    Define a problem
    \[P = (\{\bullet\}, \mathbb R, \delta_0, \ell, H)\]
    where $\ell(y,y') = (y-y')^2$, and $H$ is the collection of all functions $\{\bullet\} \to \mathbb R$, making $H$ itself homeomorphic to $\mathbb R$.\footnote{In other words, we can identify $H$ with $\mathbb{R}$.} Now, for any $\epsilon > 0$, define
    \[P_\epsilon = \left(\{\bullet\}, \mathbb R, \frac{1}{2}(\delta_0 + \delta_\epsilon), \ell, H\right).\]
    We can think of $P_\epsilon$ as $P$ with some noise; when making an observation, there is a 50\% of observing the correct $y$-value of 0, and a 50\% chance of making a small error and reading a $y$-value of $\epsilon$. We would hope that $P_\epsilon$ would converge to $P$ under $\dexp$ as $\epsilon \to 0$, but this is not true. There is only one coupling between $\delta_0$ and $1/2(\delta_0 + \delta_{\epsilon})$, so the Risk distance is
    \begin{align*}
        \dexp(P,P_\epsilon)
        & =
        \inf_{R\in \cC(H,H)} \sup_{(a,b) \in R}
        \int\int\big| (a-y)^2 - (b-y')^2|\, \delta_0(dy)\left(\frac{1}{2}\delta_0 + \frac{1}{2}\delta_{\epsilon}\right)(dy')
        \\ &
        =\inf_{R\in \cC(H,H)} \sup_{(a,b) \in R}\left[
            \frac{1}{2}\big|
            a^2 - b^2
            \big|
            +
            \frac{1}{2}\big|
            a^2 - (b-\epsilon)^2
            \big|
        \right]
        \displaybreak[1]
        \\ & \geq 
        \frac{1}{2}\inf_{R\in \cC(H,H)} \sup_{(a,b) \in R}
            \big|
            b^2
            - (b-\epsilon)^2
            \big|
        \\ & =
        \frac{1}{2}\sup_{b\in \mathbb R} |2b\epsilon-\epsilon^2| = \infty.
    \end{align*}
\end{examples}

The metric $\dexp[,p]$ enjoys a stronger stability than $\dexp$ with respect to noise in a problem's joint law.
We use a more general noise model than the one outlined in Section~\ref{subsec:stability-noise}, allowing for noise in the $X$ direction as well as the $Y$ direction. More specifically, we let our noise be given by any Markov kernel $N:X\times Y \to \prob(X\times Y)$. That is, given a problem $P$, define $P_N$ to be the same problem but with joint law on $X\times Y$ given by
\begin{align*}
    \eta\cdot N(A) := \int_{X\times Y} N(x,y)(A) \, \eta(dx\times dy)
\end{align*}
for all measurable $A\subseteq X\times Y$. We justify the use of Markov kernel composition notation by thinking of the measure $\eta$ as a Markov kernel $\eta:\{\bullet\} \to \prob(X\times Y)$.

We can now state our improved result about stability under noise.
\begin{theorem}\label{thm:non-metric-bound}
    Let $(P,\lambda)$ be a weighted problem, and $N:X\times Y \to \prob(X\times Y)$ a Markov kernel. Then
    \begin{align*}
    \dexp[,p](P,P_N)
    \leq
    \dWkern^{s_{\ell,\lambda,p}}(\delta_{\id_{X\times Y}}, N).
    \end{align*}
    Here $\delta_{\id_{X\times Y}}$ is the Markov kernel sending $(x,y)$ to $\delta_{(x,y)}$, and $\dWkern^{s_{\ell,\lambda,p}}$ is the Wasserstein metric on Markov kernels $X\times Y \to \prob(X\times Y)$ when $X\times Y$ is equipped with the joint law $\eta$ and the metric $s_{\ell,\lambda,p}$.
\end{theorem}
The proof can be found in Appendix~\ref{appendix:proofs}.

\begin{examples} 
    Let $(P,\lambda)$ be a weighted problem with $P = (\mathbb R, \mathbb R, \eta, \ell, H)$, where
    \begin{itemize}
        \item $\ell(y,y') := (y-y')^2$,
        \item $H := \{f_a|a\in \mathbb R\}$, where $f_a(x): = ax$,
        \item $\lambda$ is the standard normal distribution.
    \end{itemize}
    Let $N:\mathbb R\times \mathbb R \to \prob(\mathbb R\times \mathbb R)$ be a noise kernel that applies $x$-dependent noise in the vertical direction. That is, if $i_x^y:\mathbb R \to \mathbb R^2$ is the function $i_x^y(y') = (x,y+y')$, set
    \[N(x,y) = (i_x^y)_\sharp M(x)\]
    for some Markov kernel $M:\mathbb R \to \mathbb R$. Assume that for all $x$, $M(x)$ is a symmetric measure with finite 2nd moment $v(x)$ and 4th moment $k(x)$.
    Theorem~\ref{thm:non-metric-bound} gives us $d_{R,2}(P,P_N) \leq \dWkern^{s_{\ell,\lambda,2}}(\delta_{\id_{\mathbb R^2}}, N).$ The right hand side is difficult to exactly compute, so we will compute the larger quantity $\dWkernp{2}^{s_{\ell,\lambda,2}}(\delta_{\id_{\mathbb R^2}},N)$.

    We first produce an explicit formula for $s_{\ell,\lambda,2}((x,y),(x',y'))$.
    \begin{align*}
    s_{\ell,\lambda,2}^2((x,y),(x',y'))
    & =
    \int\left(
        \ell_h(x,y) - \ell_h(x',y')
    \right)^2 \, \lambda(dh)
    \\ & =
    \int_{-\infty}^\infty\left(
        (ax-y)^2 - (ax'-y')^2
    \right)^2 \frac{e^{-a^2/2}}{\sqrt{2\pi}} da.
    \end{align*}
    After expanding the polynomial term, the integral can be split into five terms that represent various moments of the standard normal distribution. We arrive at the polynomial
    \[
    s_{\ell,\lambda,2}^2((x,y),(x',y')) = 3(x^2-x'^2)^2 + 4(xy-x'y')^2 + (y^2-y'^2)^2 + 2(x^2-x'^2)(y^2-y'^2).
    \]
    Next we can compute $d_{W,2}^{s_{\ell,\lambda,2}}(\delta_{(x,y)}, N(x,y))$ for any fixed $(x,y) \in \mathbb R^2$. Indeed, there is only one coupling between $\delta_{(x,y)}$ and  $N(x,y)$, so we can write
    \begin{align*}
    \left(d_{W,2}^{s_{\ell,\lambda,2}}\right)^2(\delta_{(x,y)}, N(x,y))
    & =
    \int s_{\ell,\lambda,2}^2((x,y),(x',y'))\, N(x,y)(dx'{\times}dy')
    \\ & =
    \int s_{\ell,\lambda,2}^2((x,y),(x',y'))\, (i_x^y)_\sharp M(x)(dx'{\times}dy')
    \\ & =
    \int s_{\ell,\lambda,2}^2((x,y),(x,y'+y))\, M(x)(dy')
    \\ & =
    \int 4x^2y'^2 + y'^2(y'+2y)^2\, M(x)(dy').
    \end{align*}
    Again, expanding the polynomial terms and splitting the integral reduces the problem to moment computations. Recalling that $M(x)$ is symmetric and hence has vanishing odd moments, we arrive at the expression
    \[
    \left(d_{W,2}^{s_{\ell,\lambda,2}}\right)^2(\delta_{(x,y)}, N(x,y)) = 4v(x)x^2+ 4v(x)y^2 + k(x).
    \]
    Integrating against $\eta$ gives us an expression for $\left(\dWkernp{2}^{s_{\ell,\lambda,2}}\right)^2(\delta_{\id_{\mathbb R^2}},N)$, and hence an upper bound for $\dexp[,p]^2(P,P_N)$.
    \[
    \dexp[,p]^2(P,P_N) \leq \int \left[4 v(x) \left(\int y^2 \, \beta(x)(dy)\right)  + 4 x^2v(x) + k(x) \right]\, \alpha(dx)
    \]
    where $\alpha$ is the first marginal of $\eta$, and $\beta$ is the disintegration. We take a moment to interpret the integrand of the outer integral. The integral $\int y^2 \beta(x)(dy)$ is the second moment of $\beta(x)$. Hence the term $v(x)\int y^2 \beta(x)(dy)$ measures the correlation between the spread of $\beta(x)$ and the noise $\sigma^2(x)$ applied at $x$. The term $x^2 \sigma^2(x)$ measures the application of noise far from the origin. Hence this bound suggests that noise has an effect on the squared Risk distance proportional to its 4th moment, with an additional effect proportional to its 2nd moment if applied at $x$-values far from zero, or where $\beta(x)$ is already spread out.
\end{examples}

\section{The Connected Risk Distance: Risk Landscapes}\label{sec:topological}

Having discussed the weighted variant of the Risk distance, we now introduce and study one more variant, called the \emph{Connected Risk distance}. This variant is motivated by the observation that the Risk distance is insensitive to the contours of a problem's risk landscape. We will show that by strengthening the Risk distance into the Connected Risk distance, a topological descriptor of the risk landscape called the \emph{Reeb graph} can be made to exhibit stability under the Connected Risk distance without breaking many of the stability results that the Risk distance enjoys. The Reeb graph of the risk landscape of a problem $P$ encodes information about the inherent \emph{complexity} associated to solving $P$.
%\facundo{reeb $\to$ Reeb everywhere (Reeb is a lastname)} \brantley{fixed}

\subsection{Insensitivities of the Risk distance}

We saw in section \ref{subsec:loss-control} that the constrained Bayes risk of a problem is stable under the Risk distance. The constrained Bayes risk $\cB(P)$ of a problem $P$ is a sufficiently descriptive invariant if we plan to solve $P$ by exhaustively searching the predictor set $H$ for a predictor with minimal risk. If we instead wish to use a more practical search method like gradient descent, then $\cB(P)$ is a woefully incomplete descriptor. The specific structure of the risk function $\cR_P:H\to \mathbb R_{\geq 0}$ becomes crucially important, rather than just its infimum. Is $\cR_P$ rife with local minima? What are the risks of the local minima? What are their basins like? The Risk distance is not sensitive to these kinds of concerns.

\begin{examples}\label{ex:risk-landscapes}
    Define a problem
    \[P := ([0,1],\{0,1\},\eta, \ell, H)\]
    and a family of problems
    \[P_t := ([0,1],\{0,1\},\eta,\ell, H_t), t\in (0,1],\]
    where
    \begin{enumerate}
        \item $\ell$ is the 0-1 loss $\ell(y,y'):= 1_{y\neq y'}$.
        \item $\eta$ is the uniform distribution on $[0,1]\times \{0\}$.
        \item $H:= \set{1_{[0,a)}}{a \in [0,1]} \cup \set{1_{(a,1]}}{a \in [0,1]}$.
        \item $H_t := \set{1_{[0,a)}}{a \in [t,1]} \cup \set{1_{(a,1]}}{a \in [0,1]}$.
    \end{enumerate}
    Here we take the convention that $[0,0) = (1,1] = \emptyset$.
    We claim that $\dexp(P,P_t) \leq t$. Indeed, by Theorem~\ref{thm:H-stability},
    \begin{align*}
        \dexp(P,P_t) & \leq \dH^{L^1(\eta)}\left(\set{\ell_h}{h\in H}, \set{\ell_{h'}}{h'\in H_t}\right)
        \\ &
        = \sup_{h\in H} \inf_{h'\in H_t} \| \ell_h - \ell_{h'}\|_{L^1(\eta)}
        \\ &
        = \sup_{h \in H\setminus H_t} \inf_{h'\in H_t}
        \int_0^1 \big| \ell(h(x),0) - \ell(h'(x),0) \big | \,dx
        \\ &
        = \sup_{a\in [0,t)} \inf_{h'\in H_t} \int_0^1 \big| 1_{1_{[0,a)}(x) \neq 0} - 1_{h'(x) \neq 0} \big | \,dx
        \\ &
        = \sup_{a \in [0,t)} \inf_{h'\in H_t} \int_0^1 \big| 1_{[0,a)}(x)- h'(x) \big | \,dx.
        \shortintertext{Choosing $h' = 0$ gives us the upper bound bound}
        &
        \leq \sup_{a \in [0,t)} \int_0^1 1_{[0,a)}(x)\,dx = t.
    \end{align*}
    Hence $\dexp(P_t,P) \to 0$ as $t\to 0$.
    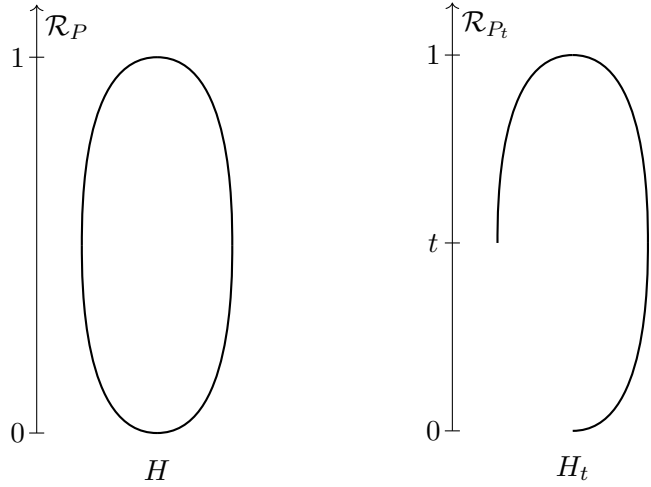
\begin{figure}
        \centering
        \begin{tikzpicture}
            \draw [thick] (0,0) .. controls (-1,0) and (-1,2) .. (-1,2.5);
            \draw [thick] (0,0) .. controls (1,0) and (1,2) .. (1,2.5);
            \draw [thick] (0,5) .. controls (-1,5) and (-1,3) .. (-1,2.5);
            \draw [thick] (0,5) .. controls (1,5) and (1,3) .. (1,2.5);

        \draw[->] (-1.6,0) -- (-1.6,5.7);
        \draw (-1.7,0) -- (-1.5,0);
        \draw (-1.7,5) -- (-1.5,5);
        \node at (-1.2,5.4) {$\cR_P$};
        \node at (-1.85,0) {$0$};
        \node at (-1.85,5) {$1$};
        \node at (0,-0.5) {$H$};
        \end{tikzpicture}
        \qquad\qquad\qquad
        \begin{tikzpicture}
            \draw [thick] (0,0) .. controls (1,0) and (1,2) .. (1,2.5);
            \draw [thick] (0,5) .. controls (-1,5) and (-1,3) .. (-1,2.5);
            \draw [thick] (0,5) .. controls (1,5) and (1,3) .. (1,2.5);

        \draw[->] (-1.6,0) -- (-1.6,5.7);
        \draw (-1.7,0) -- (-1.5,0);
        \draw (-1.7,5) -- (-1.5,5);
        \draw (-1.7,2.5) -- (-1.5,2.5);
        \node at (-1.15,5.4) {$\cR_{P_t}$};
        \node at (-1.85,0) {$0$};
        \node at (-1.85,5) {$1$};
        \node at (-1.85,2.5) {$t$};
        \node at (0,-0.5) {$H_t$};
        \end{tikzpicture}

        \caption{The risk landscapes for $P$ and $P_t$ in Example~\ref{ex:risk-landscapes}.}\label{fig:risk-landscape-example}
    \end{figure}
    At the same time, the risk landscapes of $P$ and $P_t$ look very different for any $t>0$. These risk landscapes are depicted in Figure~\ref{fig:risk-landscape-example}. Under the metric $d_{\ell,\eta}$, $H$ is topologically a circle. The risk function $\cR_P:H \to \mathbb R_{\geq 0}$ has one local minimum, which is achieved at the function $h(x) \equiv 0$. For $t>0$, $H_t$ is formed from $H$ by removing a small segment of the circle, making $H_t$ topologically a closed interval under $d_{\ell,\eta}$. Then $\cR_{P_t}$ has a local minimum at each end of the endpoints of the interval, one at $h(x)=0$ and the other at $h(x) = 1_{[0,t)}$.
\end{examples}

In this section, we strengthen the Risk distance to exert more control over the risk landscape. Inspired by \citet{bauer_reeb_2021}, we restrict the correspondences in the definition of $\dexp$ to only those satisfying a certain connectivity property.
\subsection{The Connected Risk distance}
    If $H,H'$ are topological spaces, define $\con(H,H')$ to be the set of correspondences $R\subseteq H\times H'$ such that the projections $R\to H$ and $R\to H'$ are \emph{inverse connected}, meaning the preimage of any connected set is connected.
    \begin{definition}[Connected Risk distance]
        \begin{align*}\dexpcon(P,P')
        :=&
        \inf_{\substack{\gamma \in \Pi(\eta,\eta')\\ R\in \con(H,H')}} \dis_{P,P'}(R,\gamma)
        \\ =&
        \inf_{\substack{\gamma \in \Pi(\eta,\eta')\\ R\in \con(H,H')}} \sup_{(h,h')\in R}
        \int \big| \ell_h(x,y) - \ell'_{h'}(x',y')\big| \, \gamma(dx{\times} dy{\times} dx'{\times} dy').
        \end{align*}
    \end{definition}
   
    Note that the definition of $\dexpcon$ is the same as that of $\dexp$ (Definition~\ref{def:metric}), except that the correspondence $R$ is chosen from $\con(H,H')$.
    We claim $\dexpcon$ is a pseudometric. The only nontrivial property is the triangle inequality, which requires a modified version of the gluing lemma.
    \begin{lemma}\label{lem:connected-gluing-lemma}
        Let $H_1,H_2,H_3$ be topological spaces, and let $R_{1,2}\in \con(H_1,H_2)$, $R_{2,3}\in \con(H_2,H_3)$. Then there exists an $R \subseteq H_1\times H_2 \times H_3$ such that the projection maps onto each of the three coordinates are all inverse connected, and such that $\op{proj_{1,2}}(R) = R_{1,2}$ and  $\op{proj_{2,3}}(R) = R_{2,3}$.
    \end{lemma}
    With this lemma in hand, we can prove the triangle inequality in exactly the same way that we did for $\dexp$.
    The proof of Lemma~\ref{lem:connected-gluing-lemma} can be found in Appendix~\ref{appendix:proofs}.

    \begin{remark}\label{rem:dexpcon-inherited-results}
    Note that $\dexpcon$ is a larger metric than $\dexp$, since $R$ is infimized over $\con(H_1,H_2)\subseteq \cC(H_1,H_2)$. Therefore we would expect $\dexpcon$ to be less stable than $\dexp$. Some of the stability results of $\dexp$ were proven by selecting $R$ to be the diagonal coupling between a predictor set $H$ and itself. Since $\con(H,H)$ still contains the diagonal coupling, these proofs apply just as well to $\dexpcon$, just as we saw with the $L^p$-Risk distance in Section~\ref{subsec:improved-stability}. Specifically, Condition~2 from Section~\ref{subsec:motivation}, as well as the results of Sections~\ref{subsec:stability-bias} (stability under bias) and \ref{subsec:stability-noise} (stability under noise) still apply to the Connected Risk distance.
    \end{remark}

    Additionally, since $\dexpcon$ is larger than $\dexp$, the problem descriptors stable under the former are also stable under the latter. In particular, the results of Section~\ref{subsec:loss-control} and Section~\ref{subsec:rademacher}
    still hold for $\dexpcon$; still stable under the Connected Risk distance are the loss profile set and consequently the Bayes risk, as well as the Rademacher complexity for an arbitrarily large sample size.

\subsection{Risk Landscapes and Reeb Graphs}

    We will show that a certain topological summary of the risk landscape of a problem, known as the \emph{Reeb graph}, is stable under the Connected Risk distance.

    \begin{definition}[Induced Reeb Graph]\label{def:reeb-graph}
        Let $X$ be a topological space and $f:X\to \mathbb R$ a continuous function. Define an equivalence relation on $X$ by declaring that $x\sim y$ if and only if $x$ and $y$ belong to the same connected component of some level set of $f$. Define $F:= X/{\sim}$. Let $q:F\to \mathbb R$ be the unique continuous function that factors through $f$.        The pair $(F,q)$ is called the \emph{Reeb graph induced by $f$}.
    \end{definition}
    \begin{figure}
        \includegraphics[width=\textwidth]{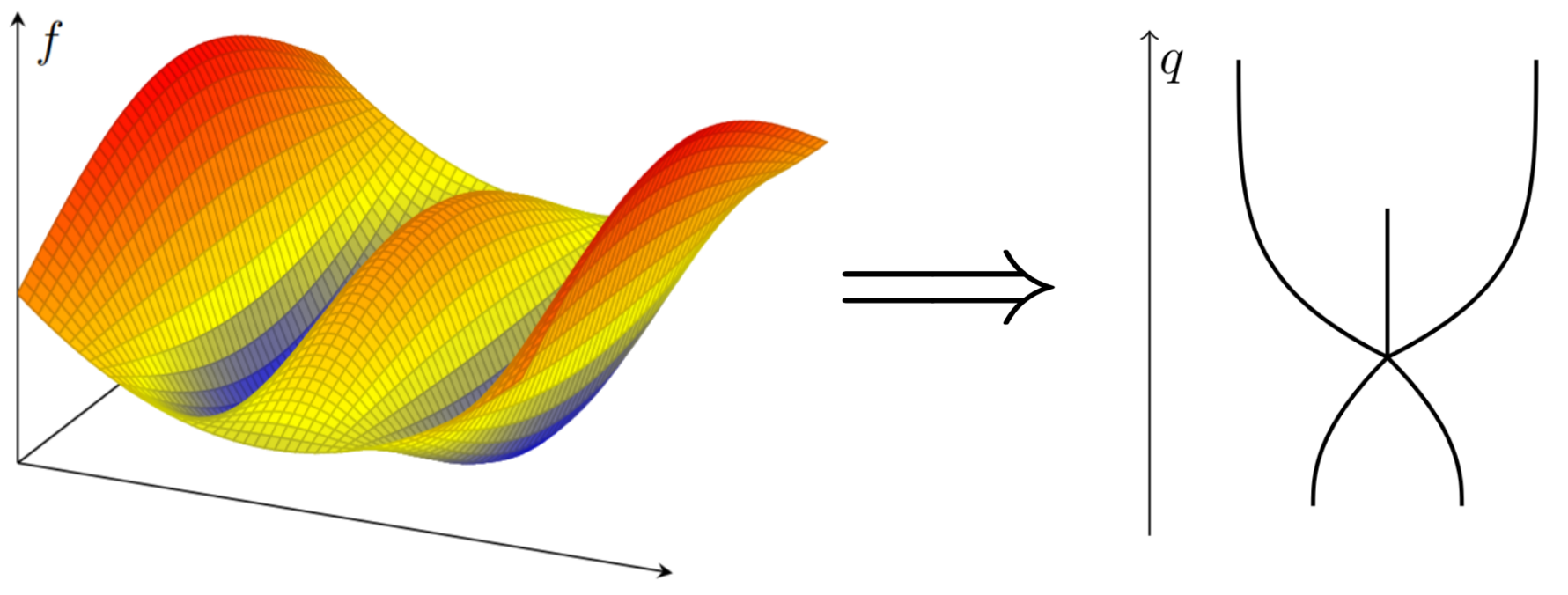}
        \caption{An illustration of a function $f$ on the unit square and its induced Reeb graph. The two local minima of the function produce downward-pointing leaves of equal height on the Reeb graph. Similarly, the three local maxima produce upward-pointing leaves, with the lower maximum producing a lower leaf.}\label{fig:reeb}
    \end{figure}
    Typically, we picture a Reeb graph $(F,q)$ by thinking of the function $q$ as a ``height'' function; we imagine that $F$ is situated in space so that $q(x)$ is the altitude of the point $x$. See Figure~\ref{fig:reeb}. In the computer science literature, many authors define the Reeb graph by declaring that $x\sim y$ if and only if $x$ and $y$ lie in the same \emph{path-connected} component instead of the same \emph{connected} component. While these definitions are not equivalent in general, they are equivalent under reasonable topological assumptions, such as when the level sets of $f$ are locally path-connected.
    
    Reeb graphs are used as summaries of landscapes induced by real functions. While originally designed by Georges Reeb to study functions on smooth manifolds \citep{reeb_1946} and independently rediscovered by \citet{Kro50}, Reeb graphs were first popularized as a computational tool by \citet{shinagawa_surface_1991}. Since then, Reeb graphs have grown in popularity, finding applications in the comparison \citep{hilaga_topology_2001,escolano_complexity_2013}, parameterization \citep{patane_para-graph_2004,zhang_feature-based_2005}, and denoising \citep{wood_removing_2004} of shapes, in symmetry detection \citep{thomas_symmetry_2011}, and in sketching of static \citep{chazal_gromov-hausdorff_2013,ge_data_2011} and time-varying \citep{edelsbrunner_time-varying_2004} data, among many others. See \citet{biasotti_reeb_2008,yan_scalar_2021} for surveys. 
    
    While it is easiest to imagine ``Reeb graphs" as \emph{topological graphs}\footnote{The following definition will suffice: A \emph{topological graph} is a topological space produced from a graph by realizing the vertices as isolated points and the edges as arcs connecting them; see \citep[Chapter 0]{hatcher2005algebraic} for details.} with attached height functions, one should be wary that, in general, the Reeb graph of a function need not admit the structure of a topological graph. Thankfully, many classes of spaces and functions are known to induce Reeb graphs which are also topological graphs. One list of sufficient conditions is provided by \citet[Example 2.2]{deSilva2016}.

    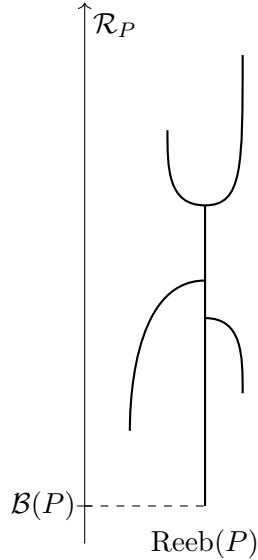
\begin{figure}
        \centering
        \begin{tikzpicture}
            \draw [thick] (0,0) -- (0,4);
            \draw [thick] (-1,1) .. controls (-1,1.2) and (-1,3) .. (0,3);
            \draw [thick] (0.5,1.5) .. controls (0.5,2) and (0.5,2.5) .. (0,2.5);
            \draw [thick] (0,4) .. controls (-0.5,4) and (-0.5,4.5) .. (-0.5,5);
            \draw [thick] (0,4) .. controls (0.5,4) and (0.5,4.5) .. (0.5,6);

            \draw[->] (-1.6,-0.5) -- (-1.6,6.7);
            \draw (-1.7,0) -- (-1.5,0);
            \node at (-1.2,6.4) {$\cR_P$};
            \node at (-2.15,0) {$\cB(P)$};
            \draw[dashed] (-1.6,0) -- (0,0);
            \node at (0,-0.5) {$\reeb(P)$};
        \end{tikzpicture}
        \caption{The Reeb graph of some problem $P$. Based on the downward-pointing leaves of the Reeb graph, this problem's risk landscape seems to have three basins. The height of the lowest point in the Reeb graph corresponds to the constrained Bayes risk $\cB(P)$.}\label{fig:problem-reeb-space}
    \end{figure}

    \begin{definition}
        Let $P$ be a problem. The \emph{Reeb graph of $P$}, denoted $\reeb(P)$, is the Reeb graph induced by the risk function $\cR_P:H\to \mathbb R_{\geq 0}$.
    \end{definition}

    The Reeb graph of a problem $P$ is a summary of the risk landscape. (See Figure~\ref{fig:problem-reeb-space}.) For instance, the height of the lowest point of $\reeb(P)$ is the constrained Bayes risk. If $\cR_P$ has a unique minimizer, then $\reeb(P)$ will have a unique lowest point. Conversely, if $\reeb(P)$ has a unique lowest point, then the minimizers of $\cR_P$ form a connected set. Local minima appear as downward-pointing leaves, with the height of the edges corresponding to the depth of the basins around those minima. Similarly, local maxima are represented by upward-pointing leaves. The Reeb graph of $P$  therefore captures information that can be useful for quantifying the complexity of the problem $P$.

    Reeb graphs can be compared via any of several related metrics. The largest such metric in popular use is called the \emph{universal distance} \citep{bauer_reeb_2021}.

    \begin{definition}[Universal Distance on Reeb graphs]
        The \emph{universal distance} between two Reeb graphs $(F,f)$ and $(G,g)$ is given by
        \[d_{\mathrm{U}}((F,f),(G,g)) := \inf_{Z,\widehat f,\widehat g:Z \to \mathbb R} \|\widehat f - \widehat g\|_{\infty}\]
        where the infimum ranges over all spaces $Z$ and functions $\widehat f, \widehat g:Z\to \bbR$ such that $\widehat f$ induces the Reeb graph $(F,f)$ and $\widehat g$ induces the Reeb graph $(G,g)$.
    \end{definition}

    Even with respect to the universal distance, the Reeb graph of a problem is stable under the Connected Risk distance.
    \begin{theorem}\label{thm:reeb-control}
        Let $P$ and $P'$ be problems. Then
        \begin{align*}
            |\cB(P) - \cB(P')| \leq d_{\mathrm{U}}(\reeb(P),\reeb(P')) \leq  \dexpcon(P,P').
        \end{align*}
    \end{theorem}
    Hence the Reeb graph of a problem is a stronger descriptor than the constrained Bayes risk while remaining stable under the Connected Risk distance. The proof can be found in Appendix~\ref{appendix:proofs}.

    \begin{examples}
        Let $P$ be a problem whose loss $\ell$ is a metric, such as a 0-1 loss in a classification problem, mean absolute error in a regression problem, or a Wasserstein distance in a probability estimation problem. Suppose that we expect noise in the labels of our data such that the average distance under $\ell$ between the true label and the observed label is $\epsilon_1$. Write $\eta'$ for the noisy distribution from which the observed data is drawn. Suppose also that the loss function $\ell$ is difficult to exactly compute or has poor optimization properties, but there is an approximation $\ell'$ that does not share the same issues and is $\epsilon_2$-close to $\ell$ in the supremum norm. Replacing $\eta$ with the noisy $\eta'$ and $\ell$ with the approximation $\ell'$ could introduce new local minima to the risk function. We would like to understand how bad these local minima can be. Let $P' = (X,Y,\eta',\ell,H)$ and $P'' = (X,Y,\eta',\ell',H)$. We aim to bound the universal distance between the Reeb graphs of the idealized problem $P$ and the corrupted problem $P''$. Since Theorem~\ref{thm:metric-noise-stability} applies to the Connected Risk distance (Remark~\ref{rem:dexpcon-inherited-results}), we have
        \[\dexpcon(P,P')\leq \epsilon_1.\]
        Since $\dexpcon$ also satisfies Condition~2 from Section~\ref{subsec:motivation}, we also have
        \[\dexpcon(P',P'') \leq \max_{h\in H} \int \big|\ell(h(x),y) - \ell'(h(x),y)\big|\, \eta(dx{\times}dy) \leq \epsilon_2.\]
        Then, using the triangle inequality for $\dexpcon$ and Theorem~\ref{thm:reeb-control}, we can write the bound
        \begin{align*}
            d_{\mathrm{U}}(\reeb(P),\reeb(P')) & \leq \dexpcon(P,P'') \leq \dexpcon(P,P') + \dexpcon(P',P'') \leq \epsilon_1 + \epsilon_2.
        \end{align*}
        In other words, if we replace the idealized problem $P$ with the corrupted problem $P''$, the Reeb graph can change by at most $\epsilon_1 + \epsilon_2$ in the universal distance. Hence any downward-pointing leaf of $\reeb(P)$ of depth more than $2(\epsilon_1 + \epsilon_2)$ from top to bottom has a corresponding downward-pointing leaf in $\reeb(P'')$, and the difference in the heights of the tips of two corresponding leaves is at most $\epsilon_1 + \epsilon_2$. We conclude that a basin of depth greater than $2(\epsilon_1 + \epsilon_2)$ in the risk landscape of the original problem corresponds to some basin in the corrupted problem, and the risk at the bottom of two corresponding basins differ by at most $\epsilon_1 + \epsilon_2$. Similarly, any new downward-pointing leaves in the Reeb graph created by corrupting $P$ to get $P''$ can be of length at most $2(\epsilon_1 + \epsilon_2)$ from top to bottom, so any new basins introduced by the corruption can be of depth at most $2(\epsilon_1 + \epsilon_2)$ from the deepest to highest point of the basin. 
    \end{examples}

\section{Discussion}
Machine learning is currently engaged in a rapid and wide-ranging exploration of its problem space. This paper proposes a guiding and organizing principle, focusing on supervised learning. As the field grows explosively, such structure is increasingly needed to make sense of emerging approaches. In this work, we argue that, given a notion of distance between supervised learning problems, various corruptions and modifications of problems can be seen from a geometric perspective. We propose the Risk distance and two variants thereof as reasonable choices for this distance notion, bolstering our claims with a garniture of stability results.
The development in this paper can be seen as a direct descendent of earlier work in the machine 
learning community on reductions between learning problems, surrogate 
losses and different forms of regularization.

The present work has opened the way for further understanding of machine 
learning analogous to several productive, related research programs. The idea of the \emph{function 
space} lies at the center of the abstract theory of functional analysis \citep[page 259]{birkhoff1984establishment}, the development of which was highly productive for 
20th century mathematics \citep{dieudonne1981history}. Similarly productive for metric geometry was the work of \citet{gromov_structures_1981}, which created the metric space of all metric spaces.
Other similar examples include category theory, which is concerned with putting relational structures on collection of objects \citep{mac2013categories}, and the theories of optimal transport \citep{Villani_old_and_new_2008} and information geometry \citep{amari_information_2016}, which both propose geometric spaces of probability measures.

Common among all of these examples is the viewpoint that \emph{to understand an object, it is productive to place it in a larger space of similar objects}.
Following this viewpoint, the present work opens the door to similar techniques by establishing candidates for \emph{spaces of supervised learning problems}.
In particular, the definition of the Risk distance directly echoes the motifs of an established lineage of distances from metric geometry and optimal transport (Section~\ref{subsec:MG-and-OT}).
We have shown that the geometry is meaningful; the geometric notion of a geodesic is related to optimal couplings and correspondences (Section~\ref{subsec:geodesics}), and that the topological density of classification problems (Section~\ref{subsec:density-of-classification}) is useful in proving the convergence of empirical problems (Section~\ref{subsec:empirical-convergence}).
Throughout Section~\ref{sec:probabilistic}, the weighting $\lambda$ placed on the set of predictors can be interpreted as a Bayesian prior---an approach aligned with the principles of Bayesian statistics. However, classical Bayesian frameworks do not consider a space of problems or attempt to endow it with a metric structure, as we do here. While Bayesian decision theory~\citep{berger1985statistical} incorporates elements such as loss functions, it does not address the geometry of the problem space. It is conceivable that the present work may offer new perspectives in that direction.

 More generally, we are optimistic that further geometric, topological, and even categorical studies of the space of problems can shed light on questions of interest in supervised learning. For example, the following concrete research directions appear to be natural candidates for exploration through the lens of the ideas developed in this paper:\begin{itemize}
\item \emph{Convergence Rates for the Empirical Problem}. We posed Question~\ref{q:conv-thm} aimed at making Theorem~\ref{thm:empirical-convergence}, on the convergence of the empirical problem, more explicit and quantitative.
\item \emph{Leverage the framework to design broader, practical reductions between learning problems}, allowing more components of ML tasks to vary. As shown in Section~\ref{subsec:density-of-classification}, classification problems are dense—in the sense of the Risk distance—within the space of all compact problems. This suggests a structural connection between classification and regression problems, despite their apparent differences. Conversely, it is natural to investigate reductions in the opposite direction: for example, reducing a classification problem with severe label noise to a regression problem with biased sampling, where both problems are close in Risk distance. Such a reduction would enable robust regression techniques to be repurposed for noisy classification tasks, facilitating the transfer of insights across distinct problem classes.
\item \emph{Analyze the robustness of other theoretical results—e.g., fast rates of convergence—through the lens of the Risk distance (or suitable variants thereof)}. In particular, is “fast convergence” a stable property under appropriate structural assumptions on the problems?
\item \emph{Assess the empirical value of the optimal couplings and correspondences produced by our theory.} An optimal coupling and correspondence (as studied in Section~\ref{subsec:optimal-couplings}) between two problems provide a kind of dictionary relating the two problems. What new insights can be drawn from such a dictionary?
\item \emph{Investigate concrete models of \emph{sampling bias} within our framework}, such as distribution shift~\citep{quinonero-candela_2009_dataset} and selection bias~\citep{meng2018statistical}, which arise in many practical learning scenarios. 
\end{itemize}

\acks{The authors thank Dr. Mark Reid for interesting discussions on topics related to this work that took place during 2010-2011. The authors also thank Dr. Nan Lu for
comments on a draft.

F.M. and B.V. were partially supported by the NSF under grants  CCF-1740761, DMS-1547357, CCF-2310412 and RI-1901360. R.W. was supported by NICTA and the Deutsche Forschungsgemeinschaft under Germany’s Excellence Strategy — EXC number 2064/1 — Project number 390727645.}

\appendix

\section{Relegated Proofs}\label{appendix:proofs}
\subsection{Proofs and Details from Section \ref{sec:distance}}

\begin{proof}[Theorem~\ref{thm:dexp-pseudometric}]
    The only nontrivial pseudometric axiom to prove is the triangle inequality. Furthermore, the claim that $\dexp$ vanishes on weakly isomorphic problems will follow from the triangle inequality and the fact that $\dexp$ satisfies Condition~1 (Theorem~\ref{thm:characterization}). That is, if $P''$ is common simulation of $P$ and $P'$, then $\dexp(P,P') \leq \dexp(P,P'') + \dexp(P'',P') = 0+0$. Hence we need only prove that $\dexp$ satisfies the triangle inequality.

    Let $P_i = (X_i,Y_i,\eta_i,\ell_i,H_i)$ for $i=1,2,3$.
    Let $\gamma_{1,2} \in \Pi(\eta_1,\eta_2)$, $\gamma_{2,3}\in \Pi(\eta_2,\eta_3)$, $R_{1,2}\in \cC(H_1,H_2)$, and $R_{2,3}\in \cC(H_2,H_3)$ be arbitrary.
    We apply the the gluing lemma (Lemma~\ref{lem:gluing}) and its counterpart for correspondences (Lemma~\ref{lem:gluing-correspondences}). Let $\gamma$ be a gluing of $\gamma_{1,2}$ and $\gamma_{1,3}$, and let $\gamma_{1,3}$ be its marginal on the first and third coordinates.
    Similarly, let $R$ be a gluing of $R_{1,2}$ and $R_{2,3}$, and let $R_{1,3}$ be its image on its first and third components.
    We can then write
    \begin{align*}
    \dis_{P,P'}(R_{1,3},\gamma_{1,3}) =
    \sup_{(h_1,h_3)\in R_{1,3}}
    & \int
    \big|\ell_1(h_1(x_1),y_1) - \ell_3(h_3(x_3),y_3)\big|
    \, \gamma_{1,3}(dx_1{\times} dy_1 {\times} dx_3 {\times} dy_3)
    \\ =
    \sup_{(h_1,h_2,h_3)\in R}
    & \int
    \big|\ell_1(h_1(x_1),y_1) - \ell_3(h_3(x_3),y_3)\big|
    \, \gamma(dx_1{\times} dy_1 {\times} dx_2 {\times} dy_2 {\times} dx_3 {\times} dy_3)
    \displaybreak[1]
    \\ \leq
    \sup_{(h_1,h_2,h_3)\in R}
    & \int
    \big|\ell_1(h_1(x_1),y_1) - \ell_2(h_2(x_2),y_2)\big|
    \, \gamma(dx_1{\times} dy_1 {\times} dx_2 {\times} dy_2 {\times} dx_3 {\times} dy_3)
    \\ +
    \sup_{(h_1,h_2,h_3)\in R}
    & \int
    \big|\ell_2(h_2(x_2),y_2) - \ell_3(h_3(x_3),y_3)\big|
    \, \gamma(dx_1{\times} dy_1 {\times} dx_2 {\times} dy_2 {\times} dx_3 {\times} dy_3)
    \displaybreak[1]
    \\ =
    \sup_{(h_1,h_2)\in R_{1,2}}
    & \int
    \big|\ell_1(h_1(x_1),y_1) - \ell_2(h_2(x_2),y_2)\big|
    \, \gamma_{1,2}(dx_1{\times} dy_1 {\times} dx_2 {\times} dy_2)
    \\ +
    \sup_{(h_2,h_3)\in R_{2,3}}
    & \int
    \big|\ell_2(h_2(x_2),y_2) - \ell_3(h_3(x_3),y_3)\big|
    \, \gamma_{2,3}(dx_2 {\times} dy_2 {\times} dx_3 {\times} dy_3)
    \\ 
    & \mkern-100mu = \dis_{P,P'}(R_{1,2},\gamma_{1,2})  + \dis_{P,P'}(R_{2,3},\gamma_{2,3}).
    \end{align*}
    Since we can construct such a $\gamma_{1,3}$ and $R_{1,3}$ from any given $\gamma_{1,2},\gamma_{2,3},R_{1,2},R_{2,3}$, we can take the infimum of each side to get the desired inequality.
\end{proof}

\begin{proof}[Theorem~\ref{thm:characterization}]
    First we demonstrate that the Risk distance does indeed satisfy Conditions~1 and 2. Let $P'$ be a simulation of $P$ via the maps $f_1:X' \to X$ and $f_2:Y'\to Y$. Define $f:= f_1\times f_2$. Since $f_\sharp \eta' = \eta$, $f$ induces a coupling $\gamma_f\in \Pi(\eta',\eta)$. Define $R\subseteq H'\times H$ to be all the pairs $(h',h)$ such that $h\circ f_1 = f_2\circ h'$. Since $f$ induces a simulation, $R$ is a correspondence. By definition of the Risk distance, we can now write
    \begin{align*}
        \dexp(P,P') & \leq 
        \sup_{(h',h)\in R}
        \int \big|
            \ell'(h'(x'),y') - \ell(h(x),y)
        \big|\, \gamma_f(dx' {\times} dy' {\times} dx {\times} dy)
        \\ & =
        \sup_{(h',h)\in R}
        \int \big|
            \ell'(h'(x'),y') - \ell(h\circ f_1(x'),f_2(y'))
        \big|\, \eta'(dx' {\times} dy').
    \end{align*}
    We claim the integrand is identically zero. Indeed,
    \begin{align*}
        \ell(h\circ f_1(x'),f_2(y')) = \ell(f_2\circ h'(x'),f_2(y')) = \ell'(h'(x'),y').
    \end{align*}
    Hence $\dexp(P,P') = 0$.

    Let $P = (X,Y,\eta,\ell,H)$ and $P' = (X,Y,\eta,\ell',H)$ be any pair of problems that differ only in their loss functions.
    We can bound $\dexp(P,P')$ by selecting the diagonal coupling between $\eta$ and itself, and the diagonal correspondence between $H$ and itself. This directly yields Condition~2.

    Now suppose $d$ is a pseudometric satisfying Conditions~1 and 2, and let $P_1$ and $P_2$ be problems. Let $\gamma \in \Pi(\eta_1,\eta_2)$ and $R\in \cC(H_1,H_2)$ be arbitrary. Define two new problems
    \begin{align*}
        P'_1 & := (X_1\times X_2, Y_1\times Y_2, \gamma, \ell'_1, R_\times)\\
        P'_2 & := (X_1\times X_2, Y_1\times Y_2, \gamma, \ell'_2, R_\times),
    \end{align*}
    where
    \begin{itemize}
        \item $\ell'_1((y_1,y_2),(y'_1,y'_2)) := \ell_1(y_1,y'_1)$,
        \item $\ell'_2((y_1,y_2),(y'_1,y'_2)) := \ell_2(y_2,y'_2)$,
        \item $R_\times := \set{h_1\times h_2}{(h_1,h_2)\in R}$.
    \end{itemize}
    Then $P'_1$ is a simulation of $P_1$ via the projection maps $X_1\times X_2 \to X_1$ and $Y_1\times Y_2 \to Y_1$, so by Condition~1, $d(P_1,P'_1) = 0$. Similarly, $d(P_2,P'_2) = 0$. Combining this with Condition~2, we get
    \begin{align*}
        & d(P_1,P_2)
        =
        d(P'_1,P'_2)
        \\ & \leq
        \max_{h_1\times h_2 \in R_\times} \int \big|
            \ell'_1((h_1\times h_2)(x_1,x_2),(y_1,y_2))
            -
            \ell'_2((h_1\times h_2)(x_1,x_2),(y_1,y_2))
        \big|\, \gamma(dx_1{\times} dx_2 {\times} dy_1 {\times} dy_2)
        \\ & =
        \max_{(h_1, h_2) \in R} \int \big|
            \ell_1(h_1(x_1),y_1)
            -
            \ell_2(h_2(x_2),y_2)
        \big|\, \gamma(dx_1{\times} dx_2 {\times} dy_1 {\times} dy_2)
        \\ & =
        \dis_{P_1,P_2}(R,\gamma).
    \end{align*}
    Taking an infimum over all choices of $\gamma$ and $R$ finishes the proof.
\end{proof}

\begin{proof}[Proposition~\ref{prop:weak-iso-optimal}]
    Let $P$ and $P'$ be weakly isomorphic, and let $P''$ be a common simulation of both $P$ and $P'$ via the maps
    \[
    \begin{tikzcd}
        & X'' \ar[dl,"f_1",swap] \ar[dr,"g_1"]\\
        X & & X'
    \end{tikzcd}\qquad \qquad
    \begin{tikzcd}
        & Y'' \ar[dl,"f_2",swap] \ar[dr,"g_2"]\\
        Y & & Y'
    \end{tikzcd}.
    \]
    Define the map
    \begin{align*}
        m:X''\times Y'' & \to X\times Y \times X' \times Y'
        \\
        (x'',y'') & \mapsto (f_1(x''), f_2(y''), g_1(x''), g_2(y'')),
    \end{align*}
    which, we note, makes the following diagram commute:
    \[
    \begin{tikzcd}
        & X''\times Y'' \ar[ddl,"f_1\times f_2",swap, bend right] \ar[ddr,"g_1\times g_2", bend left] \ar[d,"m", dashed]\\
        & X\times Y \times X' \times Y' \ar[dl,"\pi_{12}"] \ar[dr,"\pi_{34}",swap]\\
        X\times Y & & X'\times Y'
    \end{tikzcd}.
    \]
    Define $\gamma := m_\sharp \eta''.$ Note that, since the above diagram commutes, $\gamma$ is a coupling between $\eta$ and $\eta'$. Also, let $\phi:H''\to H$ be such that $\ell''_{h''} = \ell_{\phi(h'')}\circ (f_1\times f_2)$ holds $\eta''$-almost everywhere. Define $\psi:H'' \to H'$ similarly. Then
    \[R := \set{(\phi(h''),\psi(h''))}{h''\in H''}\]
    is a correspondence between $H$ and $H'$. Finally, compute
    \begin{align*}
        \dis_{P,P'}(R,\gamma)
        & =
        \sup_{(h,h')\in R}\int \big|\ell_h(x,y) - \ell'_{h'}(x',y')\big|\, \gamma(dx{\times}dy{\times}dx'{\times} dy')
        \\ & =
        \sup_{(h,h')\in R}\int \big|\ell_h\circ \pi_{12}(x,y,x',y') - \ell'_{h'}\circ \pi_{34}(x,y,x',y')\big|\, \gamma(dx{\times}dy{\times}dx'{\times} dy')
        \\ & =
        \sup_{(h,h')\in R}\int \big|\ell_h\circ \pi_{12}\circ m(x'',y'') - \ell'_{h'}\circ \pi_{34}\circ m(x'',y'')\big|\, \eta''(dx''{\times}dy'')
        \\ & =
        \displaybreak[1]
        \sup_{(h,h')\in R}\int \big|\ell_h(f_1(x''),f_2(y'')) - \ell'_{h'}(g_1(x''),g_2(y''))\big|\, \eta''(dx''{\times}dy'')
        \\ & =
        \sup_{h''\in H''}\int \big|\ell_{\phi(h'')}(f_1(x''),f_2(y'')) - \ell'_{\psi(h'')}(g_1(x''),g_2(y''))\big|\, \eta''(dx''{\times}dy'')
        \\ & =
        \sup_{h''\in H''}\int \big|\ell''_{h''}(x'',y'') - \ell''_{h''}(x'',y'')\big|\, \eta''(dx''{\times}dy'') = 0.
    \end{align*}
    
    Now suppose that there exist $R\in \cC(H,H')$ and $\gamma \in \Pi(\eta,\eta')$ such that $\dis_{P,P'}(R,\gamma)=0$. Construct a new problem
    \[P'' = (X\times X', Y\times Y', \tau_\sharp \gamma, \ell'', R_\times)\]
    where
    \begin{itemize}
        \item $\tau:X\times Y \times X' \times Y' \to X\times X' \times Y \times Y'$ swaps the second and third components,
        \item $\ell''((y_1,y'_1),(y_2,y'_2)) := \ell(y_1,y_2)$
        \item $R_\times := \set{h\times h'}{(h,h')\in R}$.
    \end{itemize}
    Then the projection maps $f_1:X\times X' \to X$ and $f_2:Y\times Y' \to Y$ show that $P''$ is a simulation of $P$. Indeed, since $\gamma$ is a coupling of $\eta$ and $\eta'$, we have
    \begin{align*}
        (f_1\times f_2)_\sharp \tau_\sharp \gamma = ((f_1\times f_2)\circ \tau)_\sharp \gamma = (\pi_{12})\sharp\gamma = \eta.
    \end{align*}
    Furthermore, for any $h\in H$, we can consider some function $(h\times h')\in R_\times$ and write
    \begin{align*}
        \ell''_{(h\times h')}((x,x'),(y,y'))
        & =
        \ell''((h(x),h'(x')),(y,y'))
        =
        \ell(h(x),y) = \ell_{h}(f_1(x,x'),f_2(y,y')).
    \end{align*}
    Similarly, $P''$ is a simulation of $P'$ via the projection maps $g_1:X\times X' \to X'$ and $g_2:Y\times Y' \to Y'$. Indeed, an almost identical calculation to the one above shows that $(g_1\times g_2)_\sharp \tau_\sharp \gamma = \eta'$. Furthermore, if $h'\in H'$ and $(h\times h')\in R_\times$, we have already calculated that $\ell''_{(h\times h')}((x,x'),(y,y')) = \ell(h(x),y)$.
    The fact that $R$ and $\gamma$ have Risk distortion zero impliex that the latter is equal $\tau_\sharp \gamma$-almost everywhere to $\ell'(h'(x'),y') = \ell'_{h'}(g_1(x,x'),g_2(y,y'))$, completing the proof.
\end{proof}
\subsubsection{Connecting the Risk Distance with the Gromov-Wasserstein Distance}\label{app:conn-dr-dgw}

Here we discuss a connection with a variant of the Gromov-Wasserstein distance considered by \citet{hang2019topological}. Given two metric measure spaces $X,X'$, a correspondence $R\in\cC(X,X')$ and a coupling $\mu\in\Pi(\mu_X,\mu_{X'})$, one considers the following notion of distortion:
$$\overline{\dis}(R,\mu):=\sup_{(x_2,x_2')\in R}\int_{X\times X'}\big|d_X(x_1,x_2)-d_{X'}(x_1',x_2')\big|\mu(dx_1\times dx_1'),$$
which leads to the definition $$\overline{d_{\mathrm{GW}}}(X,X'):=\inf_{R,\mu} \overline{\dis}(X,X').$$ This definition, which blends an $L^\infty$ matching (expressed through the correspondence $R$) with an $L^1$ one (expressed via the coupling $\mu$) was considered by \citet{hang2019topological} due its beneficial interaction with stability questions surrounding Fr\'echet functions in the persistent homology setting.\footnote{The precise variant considered in \citep{hang2019topological} uses ``continuous" correspondences. The version we consider here is structurally identical but slightly softer due to our use of arbitrary correspondences.} 
\
With the notation and definitions from Section \ref{subsec:OT-connection}, we have the following statement.
\begin{proposition}\label{prop:conn-eq} For all compact metric measure spaces $X$ and $X'$:
$\dexp(i(X),i(X'))=\overline{d_{\mathrm{GW}}}(X,X').$
\end{proposition}
\begin{proof}
First note that any correspondence $R$ between $H_X$ and $H_{X'}$ can be regarded as a correspondence between $X$ and $X'$. This is so because all functions in $H_X$ (resp. in $H_{X'}$) are constant. 

We then have the following calculation:
\begin{align*}
       \dis_{i(X),i(X')}(R,\gamma)
        & =
        \sup_{(h,h') \in R}\int \big| d_X(h(x_0),x_1) - d_{X'}(h'(x'_0),x'_1) \big| \, \gamma(dx_0{\times} dx_1 {\times} dx'_0 {\times} dx'_1)
        \\ & =
        \sup_{(x_2,x'_2) \in R} \int \big| d_X(x_2,x_1) - d_{X'}(x'_2,x'_1) \big| \, \gamma(dx_0{\times} dx_1 {\times} dx'_0 {\times} dx'_1)
        \\ & =
        \sup_{(x_2,x'_2) \in R} \int \big| d_X(x_2,x_1) - d_{X'}(x'_2,x'_1) \big| \, (\pi_{2,4})_\sharp\gamma(dx_0{\times} dx_1 {\times} dx'_0 {\times} dx'_1) \\
        & = \overline{\dis}(R,(\pi_{2,4})_\sharp\gamma).
    \end{align*}
    The above equality immediately implies that $\dexp(i(X),i(X'))\geq \overline{d_{\mathrm{GW}}}(X,X').$ The reverse inequality can be obtained from the following fact. Let $\tau:X\times X'\to X\times X\times X'\times X'$ be the map $(x,x')\mapsto (x,x,x',x')$. Then, if $\mu \in \Pi(\mu_X,\mu_{X'})$,  $\gamma_\mu:=\tau_\sharp \mu$ is an element of $\Pi((\Delta_X)_\sharp \mu_X, (\Delta_{X'})_\sharp \mu_{X'})$. Since $\pi_{2,4}\circ \tau = \mathrm{id}:X\times X'\to X\times X'$, the identity map, $(\pi_{2,4})_\sharp \gamma_\mu = \mu$ and the calculation above gives $\dis_{i(X),i(X')}(R,\gamma_\mu)=\overline{\dis}(\mu)$. From this we obtain that $\dexp(i(X),i(X'))\leq \overline{d_{\mathrm{GW}}}(X,X').$ 
\end{proof}
\begin{proof}[Proposition~\ref{prop:metric-measure-embedding}]
    If $\dGWp{p}(X,X') = 0$, then $X$ and $X'$ are isomorphic as metric measure spaces. Using the isomorphism $X\to X'$, it is not hard to construct a coupling and correspondence between $i(X)$ and $i(X')$ of risk distortion zero. Conversely, let $\dexp(i(X),i(X')) = 0$. It is not a priori obvious that there exists a correspondence $R\in \cC(H_X,H_{X'})$ and a coupling $\gamma \in \Pi((\Delta_X)_\sharp \mu_X, (\Delta_{X'})_\sharp \mu_{X'})$ that together achieve the infimal risk distortion of zero. In Section~\ref{subsec:optimal-couplings}, we will explore the existence of such couplings and correspondences. In this case, since $d_X$ and $d_{X'}$ are continuous functions and $X$ and $X'$ are assumed to be compact, Theorem~\ref{thm:infimum-achieved} does guarantee that there exists an $R$ and a $\gamma$ with $\dis_{i(X),i(X')}(R,\gamma) = 0$. Furthermore, recall that a correspondence $R$ between $H_X$ and $H_{X'}$ is equivalent to a correspondence between $X$ and $X'$. Hence, as in the proof of Proposition \ref{prop:conn-eq}, we have $0 = \dis_{i(X),i(X')}(R,\gamma) = \overline{\dis}(R,(\pi_{2,4})_\sharp\gamma)$.      Set $R' := \op{supp}((\pi_{2,4})_\sharp\gamma)$. Since the integrand defining $\overline{\dis}(R,(\pi_{2,4})_\sharp\gamma)$ is continuous and non-negative, we can conclude that
    \begin{equation}
        d_X(x_2,x_1) = d_{X'}(x'_2,x'_1) \label{eq:isometry}
    \end{equation}
    for all $(x_2,x'_2) \in R$ and $(x_1,x'_1)\in R'$. Furthermore, since $(\pi_{2,4})_\sharp\gamma$ is a coupling between the fully supported measures $\mu_X$ and $\mu_{X'}$, its support $R'$ is a correspondence between $X$ and $X'$ \citep[see][Lemma 2.2]{memoli_gromovwasserstein_2011}.
\
    We claim that $R = R'$. Indeed, if $(x_2,x'_2) \in R$, selecting any pair $(x_2,x'_1) \in R'$ yields $d_X(x_2,x_2) = d_{X'}(x'_2,x'_1)$ and hence $x'_2 = x'_1$. This shows $R\subseteq R'$, and a similar argument shows the opposite inclusion. We furthermore claim that $R$ is the graph of a bijection. Let $(x,x'), (x,x'') \in R$. Then $d_X(x,x) = d_{X'}(x',x'')$, showing $x' = x''$. Again, a similar argument holds when reversing the roles of $X$ and $X'$.
    \
    Let $f:X\to X'$ be the bijection whose graph is $R$. Then equation~\eqref{eq:isometry} becomes $d_X(x_2,x_1) = d_{X'}(f(x_2),f(x_1))$, showing that $f$ is an isometry. Additionally, since $(\pi_{2,4})_\sharp\gamma$ is a coupling of $\mu_X$ and $\mu_{X'}$ which is supported on $R$, the projections $R\to X$ and $R\to X'$ are measure-preserving bijections. It follows that the composition of the natural maps $X\to R \to X'$, which agrees with $f$, is a measure-preserving map. Hence $f$ is an isomorphism of metric measure spaces, and $\dGWp{1}(X,X') = 0$.
\end{proof}

\subsection{Proofs from Section \ref{sec:stability}}
\begin{proof}[Theorem~\ref{thm:loss-control}]
    Let $\prob_1(\bbR_{\geq 0})\subseteq \prob(\bbR_{\geq 0})$ be the subset of probability measures with finite mean. The $\mathrm{mean}$ function
    \begin{align*}
    \mathrm{mean}:\prob_1(\bbR_{\geq 0}) &\to \bbR_{\geq 0}
    \end{align*}
    is 1-Lipschitz when the domain is endowed with the 1-Wasserstein distance $\dWp{1}$ and the codomain is given the Euclidean distance. We will now ``lift'' the $\mathrm{mean}$ function to a function on power sets by applying it elementwise. That is, we define
    \begin{align*}
        \pow(\mathrm{mean}): \pow(\prob_1(\bbR_{\geq 0})) &\to \pow(\bbR_{\geq 0}) \\
        \set{\alpha_i}{i\in I} &\mapsto \set{\mathrm{mean}(\alpha_i)}{i\in I}.
    \end{align*}
    We also endow the domain and codomain with the Hausdorff distance, where the underlying metrics are $\dWp{1}$ and the Euclidean distance respectively. Since the $\mathrm{mean}$ function is 1-Lipschitz, so is $\pow(\mathrm{mean})$. The infimum function $\inf:\pow(\mathbb R_{\geq 0}) \to \mathbb R_{\geq 0}$ 1-Lipschitz as well, so the composition $\inf \circ \pow(\mathrm{mean})$ is 1-Lipschitz. In other words, for any $M,N \subseteq \prob_1(\bbR_{\geq 0})$, we have
    \[
    \big| \inf\set{\mathrm{mean}(\mu)}{\mu \in M} - \inf\set{\mathrm{mean}(\nu)}{\nu\in N} \big| \leq \dH^{\dWp{1}}(M,N).
    \]
    By Proposition~\ref{prop:bayes-from-profile}, applying this inequality when $M$ and $N$ are loss profile sets gives the first desired inequality.

    To prove the second inequality, observe that $d(P,P') := \dH^{\dW}(L(P),L(P'))$ is a pseudometric on the collection of all problems. By Theorem~\ref{thm:characterization}, we need only show that $d$ satisfies Conditions~1 and 2. Beginning with Condition~1, let $P'$ be a simulation of $P$ via the maps $f_1:X' \to X$ and $f_2:Y' \to Y$. Let $f:= f_1\times f_2$. If $h\in H$ and $h'\in H'$ are such that $f_2\circ h = h' \circ f_2$, then
    \begin{align*}
        \ell'_{h'} \circ f(x,y) & = \ell'(h'\circ f_1(x), f_2(y)) = \ell'(f_2\circ h(x), f_2(y)) = \ell(h(x),y) = \ell_h(x,y).
    \end{align*}
    Hence the loss profiles of $h$ and $h'$ are equal:
    \begin{align*}
        (\ell'_{h'})_\sharp \eta' & = (\ell'_{h'} \circ f)_\sharp \eta = (\ell_h)_\sharp \eta.
    \end{align*}
    Such an $h' \in H'$ exists for every $h\in H$ and vice versa by definition of a simulation, so we conclude that $L(P) = L(P')$ and $d(P,P') = 0$.

    For Condition~2, let $h\in H$. We can bound the Wasserstein distance between the loss profiles $(\ell_h)_\sharp \eta$ and $(\ell'_h)_\sharp \eta$ by selecting the coupling $(\ell_h, \ell'_h)_\sharp \eta$, which yields
    \begin{align*}
        \dW((\ell_h)_\sharp \eta, (\ell'_h)_\sharp \eta)
        & \leq 
        \int \big| r - s \big| (\ell_h,\ell'_h)_\sharp \eta(dr{\times}ds)
        \\ & =
        \int\big|\ell_h(x,y) - \ell'_h(x,y)\big|\, \eta(dx\times dy).
    \end{align*}
    Since this holds for all $h$, Condition~2 is satisfied as well.
\end{proof}

\begin{proof}[Theorem~\ref{thm:Rademacher-control}]
    Let $R \in \cC(H,H')$ and $\gamma \in \Pi(\eta,\eta')$ be arbitrary. Then
    \begin{align*}
        & \big| R_m(P) - R_m(P') \big|
        \\ = &
        \bigg|
        \int\int \sup_{h\in H} \frac{1}{m} \sum_{i=1}^m \sigma_i \ell_h(x_i,y_i)
        \, \rad^{\otimes m}(d\sigma)
        \,\eta^{\otimes m}(d\overline x \times d\overline y)
        \\ - &
        \int\int \sup_{h'\in H'} \frac{1}{m} \sum_{i=1}^m \sigma_i \ell'_{h'}(x'_i,y'_i)
        \, \rad^{\otimes m}(d\sigma)
        \,\eta^{\otimes m}(d\overline x' \times d\overline y') 
        \bigg|
        \\ \leq &
        \int\int \sup_{(h,h')\in R} \left| \frac{1}{m} \sum_{i=1}^m \sigma_i \ell_h(x_i,y_i) - \frac{1}{m} \sum_{i=1}^m \sigma_i \ell'_{h'}(x'_i,y'_i)\right|
         \, \rad^{\otimes m}(d\sigma)
        \,\gamma^{\otimes m}(d\overline x {\times} d\overline y {\times} d\overline x' {\times} d\overline y')
        \\ \leq &
        \int \sup_{(h,h')\in R} \frac{1}{m} \sum_{i=1}^m \big|\ell_h(x_i,y_i) - \ell'_{h'}(x'_i,y'_i)\big|
        \,\gamma^{\otimes m}(d\overline x {\times} d\overline y {\times} d\overline x' {\times} d\overline y').
    \end{align*}
    Let $\widehat \gamma$ be the unique probability measure on the countable product $(X\times Y)^{\mathbb N}$ whose marginal on the first $m$ components is $\gamma^{\otimes m}$. Then the above integral can be written
    \begin{equation}
    \label{eq:rademacher-bound}
        \int \sup_{(h,h')\in R} \frac{1}{m} \sum_{i=1}^m \big|\ell_h(x_i,y_i) - \ell'_{h'}(x'_i,y'_i)\big|
        \,\widehat \gamma(d\overline x {\times} d\overline y {\times} d\overline x' {\times} d\overline y').
    \end{equation}
    We now aim to establish the bound
    \begin{align*}
         & \int \sup_{(h,h')\in R} \frac{1}{m} \sum_{i=1}^m \big|\ell_h(x_i,y_i) - \ell'_{h'}(x'_i,y'_i)\big|
        \,\widehat \gamma(d\overline x {\times} d\overline y {\times} d\overline x' {\times} d\overline y')
        \\ - & \sup_{(h,h')\in R} \int \big|\ell_h(x_i,y_i)
        - \ell'_{h'}(x'_i,y'_i)\big| \gamma(d\overline x {\times} d\overline y {\times} d\overline x' {\times} d\overline y')
        \\ \leq & 2\rad_m(\mathcal F)
    \end{align*}
    since infimizing this inequality over $\gamma$ and $R$ will give us the result.
    It will suffice to establish
    \begin{align*}
        &\int \sup_{(h,h')\in R} \left[\frac{1}{m} \sum_{i=1}^m \big|\ell_h(x_i,y_i) - \ell'_{h'}(x'_i,y'_i)\big| - \int \big|\ell_h(x_i,y_i) - \ell'_{h'}(x'_i,y'_i) \big|\right]
        \,\widehat \gamma(d\overline x {\times} d\overline y {\times} d\overline x' {\times} d\overline y') \\& \leq 2\rad_m(\mathcal F),
    \end{align*}
    which, in probabilistic notation, is written
    \begin{equation}\label{eq:probabilistic-rademacher}
    \bbE \sup_{(h,h')\in R} \left[
        \frac{1}{m} \sum_{i=1}^m \big|\ell_h(X_i,Y_i) - \ell'_{h'}(X'_i,Y'_i)\big| - \bbE \big|\ell_h(X,Y) - \ell'_{h'}(X',Y')\big|
    \right]
    \leq 2\rad_m(\mathcal F)
    \end{equation}
    where the first expectation is taken over the tuples $(X_i,Y_i, X'_i, Y'_i)$ each sampled i.i.d. from $\gamma$, and the second expectation is over $(X,Y,X',Y')$ also sampled from $\gamma$ independently of the other tuples. The expression on the left hand side of Equation~\ref{eq:probabilistic-rademacher} is well-studied in the statistical learning theory literature, and the bound in equation~\ref{eq:probabilistic-rademacher} is easily obtained by a symmetrization argument like the one laid out in \cite[Theorem 4.1]{liao}.
    
    Now assume $\mathcal F$ is a universal Glivenko-Cantelli class. Return to the expression~\eqref{eq:rademacher-bound} and let $m\to \infty$.
    By the Glivenko-Cantelli assumption, the integrand of \eqref{eq:rademacher-bound} converges almost everywhere to
    \[\sup_{(h,h')\in R} \int \big | \ell_h(x,y) - \ell'_{h'}(x',y')\big| \, \gamma(dx{\times} dy {\times} dx'{\times} dy') = \dis_{P,P'}(R,\gamma).\]
    Furthermore, since $\ell$ and $\ell'$ are bounded, by the dominated convergence theorem, the limit of \eqref{eq:rademacher-bound} itself is $\dis_{P,P'}(R,\gamma)$ as well. Infimizing over the choices of $R$ and $\gamma$ completes the proof.
\end{proof}

\begin{proof}[Lemma~\ref{lem:disintegration-bound}]
    We aim to bound $\dexp(P,P')$. To this end, define 
    \begin{align*}
        \Delta:X\times Y \times Y &\to X\times Y \times X \times Y
        \\
        (x,y,y')  & \mapsto (x,y,x,y').
    \end{align*}
    Let $\Pi_g$ be the collection of all gluings of $\eta$ and $\eta'$ along $X$. That is,
    \begin{align*}
        \Pi_g := \set{\lambda \in \prob(X\times Y \times Y)}{(\pi_{1,2})_\sharp \lambda = \eta,(\pi_{1,3})_\sharp \lambda = \eta'}.
    \end{align*}
    By the gluing lemma (Lemma~\ref{lem:gluing}), $\Pi_g$ is nonempty.
    Note that if $\lambda\in \Pi_g$, then $\Delta_\sharp \lambda \in \Pi(\eta,\eta')$.

    We can then write
    \begin{align*}
        \dexp(P,P')
        &
        \leq \inf_{\gamma\in \Pi(\eta,\eta')} \int \big| \ell(h(x),y) - \ell(h(x'),y') \big| \, \gamma(dx\times dy\times dx'\times dy')
        \\ &
        \leq \inf_{\lambda \in \Pi_g} \int \big| \ell(h(x),y) - \ell(h(x'),y') \big| \, \Delta_\sharp \lambda(dx\times dy\times dx'\times dy')
        \\ &
        = \inf_{\lambda \in \Pi_g} \int \big| \ell(h(x),y) - \ell(h(x),y') \big| \, \lambda(dx\times dy\times dy')
        \\ &
        \leq C\inf_{\lambda \in \Pi_g} \int d_Y(y,y') \, \lambda(dx\times dy\times dy').
    \end{align*}

    Now suppose $\tau \in \Pi(\beta,\beta')$. Then one can show that the induced coupling $\lambda \in \prob(X\times Y \times Y)$, which has $X$-marginal $\alpha$ and $X$-disintegration $\tau:X\to \prob(Y\times Y')$, lies in $\Pi_g$. Hence we can continue our string of inequalities with
    \begin{align*}
        \leq C\inf_{\tau \in \Pi(\beta,\beta')} \int\int d_Y(y,y') \, \tau(x)(dy\times dy')\alpha(dx)
        & =
        C \dWkern^{d_Y}(\beta,\beta'),
    \end{align*}
    the desired bound.
\end{proof}

\begin{proof}[Theorem~\ref{thm:metric-noise-stability}]
    By Lemma~\ref{lem:disintegration-bound}, we obtain
    \begin{align*}
        \dexp(P,P') \leq C\dWkern^{d_Y}(\beta,\beta').
    \end{align*}
    We can write $\beta$ and $\beta'$ as
    \begin{align*}
        \beta & = \Delta_X \cdot (\delta_{\id_X}\otimes \beta) \cdot \delta_{\pi_2} \\
        \beta' & = \Delta_X \cdot (\delta_{\id_X}\otimes \beta) \cdot N.
    \end{align*}
    Here $\Delta_X$ is the diagonal kernel $\Delta_X:X\to X\times X$ given by $\Delta_X(x) = \delta_{(x,x)}$, and $\otimes$ represents the product of Markov kernels.
    Note that
    \begin{align*}
        \alpha\cdot\Delta_X \cdot (\delta_{\id_X}\otimes \beta)=\eta.
    \end{align*}
    Hence $\Delta_X \cdot (\delta_{\id_X}\otimes \beta)$ is an example of a measure non-increasing Markov kernel \citep[Proposition 6.1]{patterson2020}, so we can write
    \begin{align*}
        \dWkern^{d_Y}(\Delta_X \cdot (\delta_{\id_X}\otimes \beta) \cdot \delta_{\pi_2},\, \Delta_X \cdot (\delta_{\id_X}\otimes \beta) \cdot N)
        \leq \dWkern^{d_Y}(\delta_{\pi_2},N),
    \end{align*}
    completing the proof.
\end{proof}

\subsection{Proofs from Section \ref{sec:geometry}}
\begin{proof}[Lemma~\ref{lem:distortion-continuous}]
    The following proof uses ideas similar to those in the proof of Theorem 4.1 of \citet{Villani_old_and_new_2008}, with the additional wrinkle of an integrand that depends on $n$ and an additional trick of ``hot-swapping'' an integrand with a lower semi-continuous version.

    Since the topology on the domain is induced by a pseudometric, it is enough to consider an arbitrary convergent sequence $(h_n,h'_n,\gamma_n) \to (h,h',\gamma)$ in $H\times H'\times \Pi(\eta,\eta')$ and check that,
    \[\widetilde \dis_{P,P'}(h, h', \gamma) \leq \liminf_{n}\widetilde \dis_{P,P'}(h_n,h'_n,\gamma_n).\]
    To that end, let $\epsilon > 0$ be arbitrary. We define a function
    \begin{align*}
        f:X\times Y \times X' \times Y' &\to \bbR_{\geq 0}\\
        (x_0,y_0,x'_0,y'_0) & \mapsto \liminf_{(x,y,x',y') \to (x_0,y_0,x'_0,y'_0)}|\ell_h(x,y) - \ell'_{h'}(x',y')|.
    \end{align*}
    That is, $f$ is a lower semi-continuous version of the function $|\ell_h - \ell'_{h'}|$, constructed by modifying the latter on its points of discontinuity. Indeed, $f$ and $|\ell_h - \ell'_{h'}|$ agree wherever the latter is continuous, which is $\gamma$-almost everywhere on $X\times Y \times X'\times Y'$. To wit, if $A\subseteq X\times Y$ and $A' \subseteq X'\times Y'$ are the points at which $\ell_h$ and $\ell'_{h'}$ are discontinuous respectively, then the points of discontinuity for $|\ell_h - \ell'_{h'}|$ lie within $A\times X'\times Y'\cup X\times Y \times A'$, and one can check
    \begin{align*}
        \gamma(A\times X'\times Y'\cup X\times Y \times A') \leq \gamma(A\times X'\times Y') + \gamma(X\times Y \times A') = \eta(A) +\eta'(A') = 0.
    \end{align*}
    Hence $f$ is a lower semi-continuous function which agrees $\gamma$-almost everywhere with $|\ell_h - \ell'_{h'}|$.
    
    We can now proceed with our convergence argument; let $\epsilon>0$ and pick $N$ large enough that, for all $n\geq N$,
    \begin{itemize}
    \item $d_{\ell,\eta}(h_n,h) < \epsilon/3$,
    \item $d_{\ell',\eta'}(h'_n,h') < \epsilon/3$,
    \item and\footnote{In this proof, for conciseness, we will write $\int f(x,y,x',y')d\gamma$ instead of $\int f(x,y,x',y')\, \gamma(dx\times dy\times dx'\times dy')$, etc.} \[\int f(x,y,x',y')\, d\gamma < \int f(x,y,x',y')\, d\gamma_n + \epsilon/3.\]
    \end{itemize}
    The latter is possible since $f$ is a lower semi-continuous, non-negative function \citep[Lemma 4.3]{Villani_old_and_new_2008}.
    Then for large $n$ we have
    \begin{align*}
    \widetilde \dis_{P,P'}(h,h',\gamma)
    & =
    \int \big|\ell_h(x,y)-\ell'_{h'}(x',y')\big|\,d\gamma
    \\ & =
    \int f(x,y,x',y')\,d\gamma
    \\ & \leq
    \int f(x,y,x',y')\,d\gamma_n + \epsilon/3
    \\ & =
    \int \big|\ell_h(x,y)-\ell'_{h'}(x',y')\big|\,d\gamma_n + \epsilon/3
    \displaybreak[1]
    \\ & \leq
    \int \big|\ell_h(x,y)-\ell_{h_n}(x,y)\big|\,d\gamma_n
    \\ & +  \int \big|\ell_{h_n}(x,y)-\ell'_{h'_n}(x',y')\big|\,d\gamma_n
    \\ & +  \int \big|\ell_{h'_n}(x',y')-\ell'_{h'}(x',y')\big|\,d\gamma_n
    + \epsilon/3
    \\ & = d_{\ell,\eta}(h,h_n) + \widetilde\dis_{P,P'}(h_n,h'_n,\gamma_n) + d_{\ell',\eta'}(h',h'_n) + \epsilon/3 < \widetilde\dis_{P,P'}(h_n,h'_n,\gamma_n) + \epsilon
    \end{align*}
    as desired.
    In the last equality, we replace $\gamma_n$ with $\eta$ and $\eta'$ in the first and last integrals respectively since $\gamma_n$ is a coupling between $\eta$ and $\eta'$.
\end{proof}

In order to prove Lemma~\ref{lem:full-distortion-continuous}    we will require the following technical lemma.
        \begin{lemma}\label{lem:lower-semi-continuous-sup}
            Let $X$ be a pseudometric space and $f:X\to \bbR$ a lower semi-continuous function. Then the map
            \begin{align*}
                F:\pow(X) & \to \bbR\\
                A &\mapsto \sup_{a\in A} f(a)
            \end{align*}
            is lower semi-continuous as well when the power set $\pow(X)$ is endowed with the Hausdorff metric.
        \end{lemma}
        \begin{proof}
        Let the sequence of subsets $A_n\subseteq X$ converge to $A\subseteq X$ in the Hausdorff distance. Select $a' \in A$ such that
        \[\sup_{a\in A} f(a) < f(a') + \epsilon/2.\]
        Select also a sequence $a'_n \in X$ such that $a'_n \in A_n$ for all $n$, and $a'_n$ converges to $a'$. Then for sufficiently large $n$,
        \[f(a') < f(a'_n)+\epsilon/2\]
        by lower semi-continuity of $f$. Hence for large $n$,
        \begin{align*}
            F(A) = \sup_{a\in A} f(a) < f(a')+\epsilon/2 < f(a'_n) + \epsilon \leq \sup_{a\in A_n} f(a) + \epsilon = F(A_n) + \epsilon.
        \end{align*}
        Hence $F$ is lower semi-continuous.
        \end{proof}
    %\end{tcolorbox}

    \begin{proof}[Lemma~\ref{lem:full-distortion-continuous}] 
    Write
    \begin{align*}
    \dis_{P,P'}(R,\gamma) = \sup_{((h,\nu),(h',\nu')) \in R\times \{\gamma\}} \widetilde \dis_{P,P'} (h,h',\nu).
    \end{align*}
    Using the above equation, we can think of $\dis_{P,P'}$ as a composition of two maps:
    \begin{align*}
    \cC(H,H')\times \Pi(\eta,\eta') & \to \pow(H\times H' \times \Pi(\eta,\eta'))\\
    (R,\gamma) & \mapsto R\times \{\gamma\},
    \shortintertext{and}
    \pow(H\times H' \times \Pi(\eta,\eta')) & \to \bbR\\
    A & \mapsto \sup_{(h,h',\nu)\in A} \widetilde \dis_{P,P'}(h,h',\nu).
    \end{align*}
    The second map is lower semi-continuous by Lemma \ref{lem:lower-semi-continuous-sup}. The first map is an isometric embedding. Indeed, let $d_w$ be a metric on $\Pi(\eta,\eta')$ that induces the weak topology. The topology on the domain $\cC(H, H') \times \Pi(\eta,\eta')$ is induced by the metric
    \begin{align*}
    d_1((R_1,\gamma_1), (R_2,\gamma_2)) := \max\{\dH(R_1,R_2), d_w(\gamma_1,\gamma_2)\}.
    \end{align*}
    Meanwhile, unwrapping the metric on the codomain $\pow(H\times H' \times \Pi(\eta,\eta'))$ gives us
    \begin{align*}
    &d_2(R_1\times \{\gamma_1\}, R_2\times \{\gamma_2\})
    \\
    :=
    \max\Bigg\{&
        \sup_{(h_1,h_1',\nu_1) \in R_1\times \{\gamma_1\}}
        \inf_{(h_2,h_2',\nu_2) \in R_2\times \{\gamma_2\}}
        \max\{d_{H\times H'}((h_1,h_1')(h_2,h_2')), d_w(\nu_1,\nu_2)\},
        \\ &
        \sup_{(h_2,h_2',\nu_2) \in R_2\times \{\gamma_2\}}
        \inf_{(h_1,h_1',\nu_1) \in R_1\times \{\gamma_1\}}
        \max\{d_{H\times H'}((h_1,h_1')(h_2,h_2')), d_w(\nu_1,\nu_2)\}
    \Bigg\}
    \displaybreak[1]
    \\ 
    =
    \max\Bigg\{
        &\sup_{(h_1,h_1') \in R_1}
        \inf_{(h_2,h_2') \in R_2}
        d_{H\times H'}((h_1,h_1')(h_2,h_2')),\\ &
        \sup_{(h_2,h_2') \in R_2}
        \inf_{(h_1,h_1') \in R_1}
        d_{H\times H'}((h_1,h_1')(h_2,h_2')),
        d_w(\gamma_1,\gamma_2) \Bigg\}
    \displaybreak[1]
    \\
    = \max\{
        &\dH(R_1,R_2),
        d_w(\gamma_1,\gamma_2) \}, %\qedhere
    \end{align*}
    which is the definition of $d_1((R_1,\gamma_1), (R_2,\gamma_2))$. Hence the composition is lower semi-continuous.
\end{proof}

    \begin{proof}[Theorem~\ref{thm:infimum-achieved}]
        Throughout this proof, let $\cC_c(H,H')\subseteq \cC(H,H')$ be the collection of topologically closed correspondences between $H$ and $H'$.
        
        Note that $\Pi(\eta,\eta')$ is compact under the weak topology, since $\eta$ and $\eta'$ are probability measures on Polish spaces. Compactness of such coupling spaces is proved, for instance, by \citet[Lemma 1.2]{sturm_space_2012}, \citet[Lemma 2.2]{chowdhury_gromovwasserstein_2019} and in a proof by \citet[Prop 2.1]{villani_topics_2003}. Similarly, $\cC_c(H,H')$ is compact under the Hausdorff distance when $H$ and $H'$ are compact. (Here we are equipping $H$ and $H'$ with the pseudometrics $d_{\ell,\eta}$ and $d_{\ell',\eta'}$ respectively, introduced in Section~\ref{subsec:stability-H}.) As noted by \citet[Section 2]{chowdhury_explicit_2018}, compactness of such correspondence spaces follows from a theorem of Blaschke \citep[Theorem 7.3.8]{BBI}. Indeed, Blaschke's theorem implies that the space of all closed subsets of $H\times H'$ is compact. Since $\cC_c(H,H')$ is closed in the set of all closed subsets of $H\times H'$, being the intersection of the preimages of $H$ and $H'$ under the projection maps onto the first and second coordinates respectively, we get that $\cC_c(H,H')$ is compact as well.

        The lower semi-continuity of $\dis_{P,P'}$ (Lemma~\ref{lem:full-distortion-continuous}) in particular implies that for any correspondence $R$ and coupling $\gamma$, $\dis(R,\gamma) = \dis(\mathrm{cl}(R),\gamma)$, where $\mathrm{cl}$ represents topological closure. Indeed, since $R$ and its closure are distance zero apart in the Hausdorff distance, the constant sequence $R$ converges to $\mathrm{cl}(R)$ and vice versa. Hence, when computing the Risk distance, we can restrict our attention to closed correspondences:
        \[\dexp(P,P') = \inf_{\substack{R\in \cC(H,H')\\\gamma\in \Pi(\eta,\eta')}}\dis_{P,P'}(R,\gamma) = \inf_{\substack{R\in \cC_c(H,H')\\\gamma\in \Pi(\eta,\eta')}}\dis_{P,P'}(R,\gamma).\]
        Since $\dis_{P,P'}$ is lower semi-continuous, and the latter infimum is over a pair of compact sets, the infimum must be achieved.
    \end{proof}

\begin{proof}[Theorem \ref{thm:coarsen-approx}]    First consider an arbitrary finite partition $Q$ of $Y$. Let $\pi:Y\to Q$ be the quotient map. Since $(\id_X \times \pi)_\sharp \eta = \eta_Q$, the map $\id_X\times \pi$ induces a coupling $\gamma_{(\id_X\times \pi)}$ between $\eta$ and $\eta_Q$. We also define a correspondence $R_\pi$ between $H$ and $H_Q$ by pairing each $h\in H$ with $\pi\circ h \in H_Q$. Selecting this coupling and correspondence gives us a bound on the Risk distance:
    \begin{align}\label{eq:finite-approx-bound}
    \dexp(P,P_Q) & \leq \dis_{P,P_Q}(R_\pi,\gamma_{\id_X\times \pi})
    \notag
    \\ & \leq \sup_{(h,\pi\circ h) \in R_\pi}
    \int \big|\ell(h(x),y) - \ell'(\pi\circ h(x'),y')\big| \, \gamma_{(\id_X\times \pi)}(dx \times dy  \times dx'  \times dy')
    \notag
    \\ & = \sup_{h\in H} \int \big|\ell(h(x),y) - \ell'(\pi\circ h(x),\pi(y))\big| \, \eta(dx\times dy)
    \notag
    \\ & = \sup_{h\in H}
        \int \sup_{
        \substack{y'\in \pi\circ h(x) \\
        y''\in \pi(y) 
        }
        }\ell(y',y'') - \ell(h(x),y) \, \eta(dx\times dy)
        \notag
        \displaybreak[1]
    \\ &\notag  \leq \sup_{h\in H}
        \int \sup_{
        \substack{y'\in \pi\circ h(x) \\
        y''\in \pi(y) 
        }
        }\ell(y',y'')
        - \inf_{
        \substack{y'\in \pi\circ h(x) \\
        y''\in \pi(y) 
        }
        }
    \ell(y',y'') \, \eta(dx\times dy)
        \displaybreak[1]
    \\ & \leq \sup_{h\in H} \sup_{(x,y) \in X\times Y}\left[
        \sup_{
        \substack{y'\in \pi\circ h(x) \\
        y''\in \pi(y) 
        }
        }\ell(y',y'')
        - \inf_{
        \substack{y'\in \pi\circ h(x) \\
        y''\in \pi(y) 
        }
        }
        \ell(y',y'') \right]
    \notag
    \\ & \leq \max_{q,q'\in Q}\left[
        \sup_{
        \substack{y'\in q' \\
        y''\in q
        }
        }\ell(y',y'')
        - \inf_{
        \substack{y'\in q' \\
        y''\in q 
        }
        }
        \ell(y',y'') \right].
    \end{align}
    The bound (\ref{eq:finite-approx-bound}) holds for an arbitrary finite partition $Q$ of $Y$.

    Since $P$ is a compact problem, we can pick a metric $d_Y$ on $Y$ with respect to which $Y$ is compact. (Recall that we required our response spaces to be Polish, which makes $Y$ metrizable.) In particular, $Y\times Y$ is a compact metric space under the product metric, given by
    \[d_{Y\times Y}((y_1,y_1'),(y_2,y_2')) := \max\{d_Y(y_1,y_2),d_Y(y_1',y_2')\}.\]
    By compactness, $\ell$ is uniformly continuous with respect to $d_{Y\times Y}$. Hence we can pick $\delta>0$ such that, if $y_1,y_1',y_2,y_2'\in Y$ and
    \[d_{Y\times Y}((y_1,y_1'),(y_2,y_2'))<\delta,\]
    then
    \[\big|\ell(y_1,y_1') - \ell(y_2,y_2')\big|<\epsilon.\]
    Again using compactness, we choose a partition $Q$ of $Y$ such that $\diam_{d_Y}(q)<\delta$ for all $q\in Q$. Then for any $q,q'\in Q$, $y_1,y_2\in q$, $y_1',y_2'\in q'$, we have
    \begin{align*}
    d_{Y\times Y}((y_1,y_1'),(y_2,y_2')) 
    & =\max\{d_Y(y_1,y_2),d_Y(y_1',y_2')\}
    < \delta.
    \end{align*}
    Hence
    \begin{align*}
    \big|\ell(y_1,y_1') - \ell(y_2,y_2')\big|<\epsilon.
    \end{align*}
    Taking a supremum over all choices of $q,q',y_1,y_2,y_1',y_2'$, combined with inequality (\ref{eq:finite-approx-bound}), gives us the result.
\end{proof}

\begin{proof}[Theorem~\ref{thm:empirical-convergence}]
    We will proceed by approximating $P$ and $P_n$ by problems with finite input and output spaces. The process of approximation will proceed in steps depicted as follows:
    \begin{align*}
        P \xrightarrow{\text{Restrict to finite $H$}} P' \xrightarrow{\text{Replace $Y$ with partition $Q_Y$}} P'_{Q_Y} \xrightarrow{\text{Replace $X$ with partition $Q_X$}} P'_{Q_Y,Q_X}
    \end{align*}
    A similar sequence of approximations will be constructed for the empirical problem $P_n$ to produce $P_{Q_Y,Q_X,n}$. We then show that $P_{Q_Y,Q_X,n}$ converges a.s. to $P_{Q_Y,Q_X}$.

    To wit, let $\epsilon > 0$ and construct new problems as follows. Since $H$ is compact, we can find a finite $\epsilon$-net $H'\subseteq H$. Define
    \begin{align*}
        P' &:= (X,Y,\eta,\ell, H')  &  P'_n &:= (X,Y,\eta_n,\ell, H').
    \end{align*}
    By Theorem~\ref{thm:coarsen-approx}, we can select a partition $Q_Y$ of $Y$ such that the resulting coarsened problem $P'_{Q_Y}$ is $\epsilon$-close to $P'$. Furthermore, by Remark~\ref{remark:coarsening-distribution-independent}, we can choose $Q_Y$ in a way that does not depend on the joint law $\eta$. We can coarsen $P'_n$ with respect to the same partition $Q_Y$, giving us two new problems:
    \begin{align*}
        P'_{Q_Y} &:= (X,Q_Y,(\id_X\times \pi_{Q_Y})_\sharp\eta,\ell_{Q_Y}, H'_{Q_Y})  &
        P'_{Q_Y,n} &:= (X,Q_Y,(\id_X\times \pi_{Q_Y})_\sharp\eta_n,\ell_{Q_Y}, H'_{Q_Y}).
    \end{align*}
    Recall that $\pi_{Q_Y}:Y\to Q_Y$ is the canonical quotient map,  $\ell_{Q_Y}(q,q') := \sup_{y \in q, z\in q'}\ell(q,q')$, and $H'_{Q_Y} := \{\pi_{Q_Y}\circ h| h\in H'\}$.

    For the third and final round of approximation, notice that
    \(A := \{h^{-1}(q) | h\in H'_{Q_Y}, q\in Q_Y\}\)
    is a finite collection of subsets of $X$. Let $Q_X$ be the partition of $X$ generated by $A$. That is, $Q_X$ is the partition of $X$ induced by the equivalence relation ${\sim}$ where $x\sim x'$ means that $h(x)=h(x')$ for all $h\in H'_{Q_Y}$.  Then define
    \begin{align*}
        P'_{Q_Y,Q_X} &:= (Q_X,Q_Y,(\pi_{Q_X}\times \pi_{Q_Y})_\sharp\eta,\ell_{Q_Y}, \widetilde{H'_{Q_Y}})
        \\
        P'_{Q_Y,Q_X,n} &:= (Q_X,Q_Y,(\pi_{Q_X}\times \pi_{Q_Y})_\sharp\eta_n,\ell_{Q_Y}, \widetilde{H'_{Q_Y}}),
    \end{align*}
    where $\pi_{Q_X}:X\to Q_X$ is the canonical quotient map, and $\widetilde{H'_{Q_Y}}$ is the collection of functions sending $\pi_{Q_X}(x) \mapsto h(x)$ as $h$ ranges over $H'_{Q_Y}$.

    We apply the triangle inequality to write
    \begin{align*}
        \dexp(P,P_n) \leq & \dexp(P,P') \tag{1}
        \\
        + & \dexp(P',P'_{Q_Y}) \tag{2}
        \\
        + & \dexp(P'_{Q_Y},P'_{Q_Y,Q_X}) \tag{3}
        \\
        + & \dexp(P'_{Q_Y,Q_X},P'_{Q_Y,Q_X,n}) \tag{4}
        \\
        + & \dexp(P'_{Q_Y,Q_X,n},P'_{Q_Y,n}) \tag{5}
        \\
        + & \dexp(P'_{Q_Y,n},P'_n) \tag{6}
        \\
        + & \dexp(P'_n,P_n) \tag{7}.
    \end{align*}

    We consider each of these terms individually.

    \begin{enumerate}[(1)]
    \item Since $P$ and $P'$ have the same joint law, we can bound $\dexp(P,P')$ above by selecting the diagonal coupling. Additionally, since $H'$ is an $\epsilon$-net for $H$, we can select a correspondence $R\in \cC(H,H')$ such that $d_{\ell,\eta}(h,h')<\epsilon$ for all $(h,h')\in R$. This produces the bound
    \begin{align*}
        \dexp(P,P') 
        &
        \leq \sup_{(h,h')\in R} \int \big| \ell(h(x),y) - \ell(h'(x),y) \big|\, \eta(dx\times dy)
        = \sup_{(h,h')\in R} d_{\ell,\eta}(h,h') \leq \epsilon.
    \end{align*}

    \item We chose $Q_Y$ so that $P'$ and $P'_{Q_Y}$ would be $\epsilon$-close under $\dexp$, so this term is at most $\epsilon$.

    \item We bound $\dexp(P'_{Q_Y}, P'_{Q_Y,Q_X})$ by selecting the coupling induced by $\pi_{Q_X} \times \id_{Q_Y}: X\times Y \to Q_X\times Y$. We also select the correspondence $R$ between $H'_{Q_Y}$ and $\widetilde{H'_{Q_Y}}$ that pairs each $h\in H'_{Q_Y}$ with its counterpart $h'\in \widetilde{H'_{Q_Y}}$ which sends $\pi_{Q_X}(x)\mapsto h(x)$. Then we obtain the bound
    \begin{align*}
        \dexp(P'_{Q_Y}, P'_{Q_Y,Q_X})
        &
        \leq \sup_{(h,h')\in R} \int \big| \ell(h(x),q) - \ell(h'(\pi_{Q_X}(x)),q) \big|\, (\pi_{Q_X}\times \id_{Q_Y})_\sharp \eta (dx\times dq)\\
        &
        = \sup_{(h,h')\in R} \int \big| \ell(h(x),q) - \ell(h(x)),q) \big|\, (\pi_{Q_X}\times \id_{Q_Y})_\sharp \eta (dx\times dq) = 0.
    \end{align*}
    Alternatively, one could note that $P'_{Q_Y}$ is a simulation of $P'_{Q_Y,Q_X}$ via the maps $\pi_{Q_X}$ and $\id_{Q_Y}$ and apply Theorem~\ref{thm:characterization}.

    \item For brevity, define
    \begin{align*}
        \nu &:=(\pi_{Q_X}\times\pi_{Q_Y})_\sharp \eta
        \\
        \nu_n &:=(\pi_{Q_X}{\times}\pi_{Q_Y})_\sharp\eta_n
    \end{align*}
    We apply the total variation bound from Proposition~\ref{prop:H-wasserstein-bound} to get
    \[\dexp(P'_{Q_Y,Q_X},P'_{Q_Y,Q_X,n}) \leq \ell_{\mathsf{max}} \TV(\nu,\nu_n).\]
    
    Note that $\nu_n$ is the empirical measure for $\nu$, which is a measure supported on a finite set. Hence by the law of large numbers, the total variation distance between $\nu$ and $\nu_n$ converges to 0 a.s. as $n\to \infty$. Hence, with probability 1, for sufficiently large $n$, we have $\dexp(P'_{Q_Y,Q_X},P'_{Q_Y,Q_X,n}) < \epsilon$.

    \item The same calculation we applied to term (3) shows that $\dexp(P'_{Q_Y,Q_X,n},P'_{Q_Y,n}) = 0$.

    \item Since we chose $Q_Y$ in a way that did not depend on the joint law $\eta$, and since $P'$ and $P'_n$ only differ in their joint laws, we get $\dexp(P'_n,P'_{Q_Y,n})<\epsilon$.

    \item Consider the (random) function on $H\times H$ given by
    \[
        d_{\ell,\eta_n}(h,h') := \int \big|\ell_h(x,y) - \ell_{h'}(x,y)\big|\, \eta_n(dx{\times}dy).
    \]
    Let $R$ be the same as in our calculations for term (1). Then we can write
    \begin{align*}
        \dexp(P'_n,P_n) \leq \sup_{(h,h') \in R} \int \big|\ell_h(x,y) - \ell_{h'}(x,y)\, \big|\, \eta_n(dx\times dy) = \sup_{(h,h') \in R} d_{\ell,\eta_n}(h,h').
    \end{align*}
    Furthermore, since we assumed the integrands $|\ell_h - \ell_{h'}|$ are a Glivenko-Cantelli class with respect to $\eta$, we have that for sufficiently large $n$,
    \begin{align*}
        \sup_{(h,h') \in R} d_{\ell,\eta_n}(h,h') - \epsilon
        & \leq 
        \sup_{(h,h') \in R} d_{\ell,\eta_n}(h,h') - d_{\ell,\eta}(h,h')
        \\ & \leq
        \sup_{(h,h') \in R} \big| d_{\ell,\eta_n}(h,h') - d_{\ell,\eta}(h,h') \big|
        < \epsilon.
    \end{align*}
    Here the Glivenko-Cantelli assumption is used in the last inequality. Hence
    \begin{align*}
        \dexp(P'_n,P_n) \leq \sup_{(h,h') \in R} d_{\ell,\eta_n}(h,h') < 2\epsilon.
    \end{align*}
    \end{enumerate}

    Combining our bounds for all seven terms gives us that, with probability 1, for sufficiently large $n$ we have
    \begin{align*}
        \dexp(P,P_n) < \epsilon + \epsilon + 0 + \epsilon + 0 + \epsilon + 2\epsilon = 6\epsilon.
    \end{align*}
    Since $\epsilon$ was arbitrary, we conclude that $\dexp(P,P_n) \to 0$ with probability 1.
\end{proof}

\subsection{Proofs from Section \ref{sec:probabilistic}}
\begin{proof}[Proposition \ref{prop:conn-w-dGW}]
Let $(X,d_X,\mu_X)$ and $(Y,d_Y,\mu_Y)$ be two metric measure spaces. The map
\begin{align*}
    \Pi((\Delta_X)_\sharp\mu_X, (\Delta_Y)_\sharp \mu_Y) & \to \Pi(\mu_X, \mu_Y) \\
    \gamma & \mapsto (\pi_{2,4})_\sharp \gamma
\end{align*}
is a bijection. Here, $\pi_{2,4}:X\times X\times Y\times Y\to X\times Y$ is the map $(x,x',y,y')\mapsto (x',y')$. Using this fact, we can write
\begin{align*}
    &\dexp[,1](P_X,P_Y)
    \\= &
    \inf_{\substack{\gamma \in \Pi((\Delta_X)_\sharp\mu_X, (\Delta_Y)_\sharp \mu_Y)\\
    \rho \in  \Pi(\mu_X, \mu_Y)}}
    \int \int \big | d_X(h(x_0),x_1) - d_Y(h'(y_0),y_1) \big| \, \gamma(dx_0{\times}dx_1{\times}dy_0{\times}dy_1)\rho(dh{\times}dh')
    \\= &
    \inf_{\substack{\gamma \in \Pi((\Delta_X)_\sharp\mu_X, (\Delta_Y)_\sharp \mu_Y)\\
    \rho \in  \Pi(\mu_X, \mu_Y)}}
    \int \int \big | d_X(x_2,x_1) - d_Y(y_2,y_1) \big| \, \gamma(dx_0{\times}dx_1{\times}dy_0{\times}dy_1)\rho(dx_2{\times}dy_2)
    \\= &
    \inf_{\substack{\gamma \in \Pi((\Delta_X)_\sharp\mu_X, (\Delta_Y)_\sharp \mu_Y)\\
    \rho \in  \Pi(\mu_X, \mu_Y)}}
    \int \int \big | d_X(x_2,x_1) - d_Y(y_2,y_1) \big| \, (\pi_{2,4})_\sharp\gamma(dx_1{\times}dy_1)\rho(dx_2{\times}dy_2)
    \\= &
    \inf_{\substack{\gamma \in \Pi(\mu_X, \mu_Y)\\
    \rho \in  \Pi(\mu_X, \mu_Y)}}
    \int \int \big | d_X(x_2,x_1) - d_Y(y_2,y_1) \big| \, \gamma(dx_1{\times}dy_1)\rho(dx_2{\times}dy_2). 
\end{align*}
as desired.
\end{proof}

\begin{proof}[Theorem~\ref{thm:loss-control-weighted}]
    First, for any $h\in H$, $h'\in H'$, we can bound the Wasserstein distance between their loss profiles by
    \begin{align*}
        \dW((\ell_h)_\sharp\eta,(\ell'_{h'})_\sharp\eta')
        & =
        \inf_{\tau\in \Pi((\ell_h)_\sharp\eta,(\ell'_{h'})_\sharp\eta')} \int |a-b| \, \tau(da{\times}db)
        \\ &
        \leq \inf_{\gamma \in \Pi(\eta,\eta')} \int |a-b| \, (\ell_h\times \ell'_{h'})_\sharp (da{\times}db)
        \\ &
        = \inf_{\gamma \in \Pi(\eta,\eta')} \int |\ell_h(x,y)- \ell'_{h'}(x',y')| \, \gamma(dx{\times}dy{\times}dx'{\times}dy').
    \end{align*}
    Similarly, we can bound
    \begin{align*}
        \dWp{p}^{\dW}(\cL(P),\cL(P'))
        & =
        \inf_{\tau\in \Pi((F_P)\sharp \lambda, (F_P)\sharp \lambda')}\left(
            \int \dW^p(\mu,\nu)\, \tau(d\mu{\times}d\nu)
        \right)^{1/p}
        \\ &
        \displaybreak[1]
        \leq
        \inf_{\rho\in \Pi(\lambda, \lambda')}\left(
            \int \dW^p(\mu,\nu)\, (F_P\times F_{P'})_\sharp \rho(d\mu{\times}d\nu)
        \right)^{1/p}
        \\ &
        =
        \inf_{\rho\in \Pi(\lambda, \lambda')}\left(
            \int \dW^p((\ell_h)_\sharp \eta,(\ell'_{h'})_\sharp \eta') \, \rho(dh{\times}dh')
        \right)^{1/p}.
    \end{align*}
    
    Combining these bounds gives 
    \begin{align*}
        & \dWp{p}^{\dW}(\cL(P),\cL(P'))
        \\&\leq 
        \inf_{\rho\in \Pi(\lambda, \lambda')}\left(
            \int \left(
                \inf_{\gamma \in \Pi(\eta,\eta')} \int |\ell_h(x,y)- \ell'_{h'}(x',y')| \, \gamma(dx{\times}dy{\times}dx'{\times}dy')
            \right)^p \, \rho(dh{\times}dh')
        \right)^{1/p}
        \\ & \leq 
        \inf_{\rho, \gamma}\left(
            \int \left(
                \int |\ell_h(x,y)- \ell'_{h'}(x',y')| \, \gamma(dx{\times}dy{\times}dx'{\times}dy')
            \right)^p \, \rho(dh{\times}dh')
        \right)^{1/p}
        \\ & = \dexp[,p](P,P'), %\qedhere
    \end{align*}
    as desired.
\end{proof}

\begin{proof}[Theorem~\ref{thm:weighted-density}]
    The support of a probability measure on a Polish space is $\sigma$-compact, so without loss of generality, we can assume that $Y$ is $\sigma$-compact by replacing $Y$ with the support of the $Y$-marginal of $\eta$. Write $Y$ as an increasing sequence of compact subsets $(Y_n)_{n=1}^\infty$. For each $n$, define $\pi_n:Y\to Y_n\cup \{\bullet\}$ by
    \begin{align*}
      \pi_n(y) := \begin{cases}
        y & y\in Y_n\\
        \bullet & \text{else}
      \end{cases}.
    \end{align*}
    We then define a sequence of weighted problems $(P_n,\lambda_n)$. Define
    \[P_n := (X,Y_n\cup \{\bullet\},\eta_n,\ell_n,H_n)\]
    where
    \begin{itemize}
      \item $\eta_n := (\id_X\times \pi_n)_\sharp \eta$
      \item $\ell_n(y,y') := 1_{Y_n\times Y_n}(y,y')\ell(y,y')$
      \item $H_n := \set{\pi_n\circ h}{h\in H}$.
    \end{itemize}
    Furthermore, if we let $\pi_{n*}:H\to H_n$ be given by $\pi_{n*}(h) := \pi_n\circ h$, we can define $\lambda_n := (\pi_{n*})_\sharp \lambda$.

    We claim that $\dexp[,p](P,P_n)$ converges to 0. Since $\eta_n = (\id_X\times \pi_n)_\sharp \eta$, the function $\id_X\times \pi_n$ induces a coupling between $\eta$ and $\eta_n$. Similarly, since $\lambda_n = (\pi_{n*})_\sharp \lambda$, the function $\pi_{n*}$ induces a coupling between $\lambda$ and $\lambda_n$. By definition of $\dexp[,p]$, choosing these couplings produces a bound on $\dexp[,p](P,P_n)$.
    \begin{align*}
      \dexp[,p](P,P_n)
      & \leq \left(\int \left(\int \big|\ell(h(x),y) - \ell'(\pi \circ h(x),\pi(y))\big| \, \eta(dx{\times} dy) \right)^{p}\lambda(dh)\right)^{1/p}
      \\ &
      = \left(\int \left(\int 1_{(Y_n\times Y_n)^c}(h(x),y)\,\ell(h(x),y) \, \eta(dx\times dy) \right)^{p}\lambda(dh)\right)^{1/p}.
    \end{align*}
    Here the superscript $c$ represents set complement.
    Notice that the indicator $1_{(Y_n{\times} Y_n)^c}(h(x),y)$ converges pointwise to 0 as $n\to \infty$.
    We apply the dominated convergence theorem twice to finish the proof. Indeed, by assumption,
    \begin{align}
      \diam_p(P) =
      \left(
        \int \left(
          \int \ell(h(x),y) \, \eta(dx\times dy)
        \right)^p
      \lambda(dh)
      \right)^{1/p}
      < \infty. \label{eq:finite-diameter}
    \end{align}
    In particular, \eqref{eq:finite-diameter} tells us that $\ell(h(x),y)$ is integrable as a function of $(x,y)$ for $\lambda$-a.e. $h\in H$. Since $\ell(h(x),y)$ dominates $1_{(Y_n\times Y_n)^c}(h(x),y)\ell(h(x),y)$, the dominated convergence theorem tells us that
    \begin{align*}
      \int 1_{(Y_n\times Y_n)^c}(h(x),y)\ell(h(x),y)
      \, \eta(dx\times dy)
      \xrightarrow{n\to\infty} 0
    \end{align*}
    for $\lambda$-a.e. $h\in H$.
    Additionally, \eqref{eq:finite-diameter} tells us that
    \begin{align*}
        \left(
          \int \ell(h(x),y) \, \eta(dx\times dy)
        \right)^p
    \end{align*}
    is an integrable function of $h$. Using the dominated convergence theorem again,
    \begin{align*}
      \left(
        \int \left(
          \int 1_{(Y_n\times Y_n)^c}(h(x),y)\ell(h(x),y)
        \, \eta(dx\times dy) \right)^{p}
      \lambda(dh)\right)^{1/p}
      \xrightarrow{n\to\infty} 0
    \end{align*}
    as desired.
\end{proof}

\begin{proof}[Proposition~\ref{prop:finite-dense-in-compact-weighted}]
    Let $\epsilon > 0$. Let $Q$ be the partition of $Y$ provided by Theorem~\ref{thm:coarsen-approx}, and let
    \[P_Q = (X,Q,\eta_Q,\ell_Q,H_Q)\]
    be the coarsened problem. Let $R_\pi \in \cC(H,H_Q)$ be the correspondence induced by the surjection $\pi_*(h) = \pi\circ h$, and let $\gamma_{(\id_X \times \pi)} \in \Pi(\eta,\eta_Q)$ be the coupling induced by the map $\id_X\times \pi$. Recall the proof of Theorem~\ref{thm:coarsen-approx}, in which we demonstrated that $\dexp(P,P_Q)<\epsilon$ by showing
    \[\dis_{P,P_Q}(R_\pi, \gamma_{(\id_X\times \pi)})<\epsilon.\]
    Define a measure $\lambda_Q = (\pi_*)_\sharp \lambda$ on $H_Q$ and let $\rho_{\pi} \in \Pi(\lambda,\lambda_Q)$ be the coupling induced by $\pi_*$. Then, since $R_\pi$ is the support of $\rho_\pi$,
    \begin{align*}
        &\dexp[,p]((P,\lambda),(P_Q,\lambda_Q))
        \leq
        \dis_{P,P_Q,p}(\rho_\pi, \gamma_{\id_X\times \pi})
        \\ = &
        \left(\int \left(
            \int \big| \ell(h(x),y) - \ell_Q(h'(x'),y')\big|\, 
            \gamma_{\id_X\times \pi}(dx{\times}dy{\times}dx'{\times}dy')
        \right)^p \rho_{\pi}(dh{\times}dh')
        \right)^{1/p}
        \\ \leq &
        \sup_{(h,h') \in R_\pi}
            \int \big| \ell(h(x),y) - \ell_Q(h'(x'),y')\big|\, 
            \gamma_{\id_X\times \pi}(dx{\times}dy{\times}dx'{\times}dy')
        \\  = & \dis_{P,P_Q}(R_\pi,\gamma_{\id_X\times \pi}) < \epsilon,
    \end{align*}
    as desired.
\end{proof}

\begin{proof}[Theorem~\ref{thm:non-metric-bound}]
    Proposition~\ref{prop:H-wasserstein-bound-weighted} provides the bound $\dexp(P,P_N)\leq \dW^{s_{\ell,\lambda,p}}(\eta,\eta \cdot N)$.
    One can check that if $\tau$ is a coupling of the Markov kernels $\delta_{\id_{X\times Y}}$ and $N$, then $\eta\cdot \tau \in \Pi(\eta,\eta\cdot N)$. Therefore we can write
    \begin{align*}
        &\dW^{s_{\ell,\lambda,p}}(\eta,\eta \cdot N)
        \\ = & 
        \inf_{\gamma\in \Pi(\eta,\eta\cdot N)}
        \int_{(X\times Y)^2}
        s_{\ell,\lambda,p}((x,y),(x',y')) \,
        \gamma(dx{\times} dy {\times} dx' {\times} dy')
        \\ \leq & 
        \inf_{\tau\in \Pi(\delta_{\id_{X\times Y}},N)}
        \int_{(X\times Y)^2}
        s_{\ell,\lambda,p}((x,y),(x',y')) \,
        (\eta\cdot \tau)(dx{\times} dy {\times} dx' {\times} dy')
        \\ = &
        \inf_{\tau\in \Pi(\delta_{\id_{X\times Y}},N)}
        \int_{X\times Y}\int_{(X\times Y)^2}
        s_{\ell,\lambda,p}((x_1,y_1),(x_2,y_2)) \,
        \tau(x,y)(dx_1{\times} dy_1 {\times} dx_2 {\times} dy_2) \, \eta(dx{\times} dy)
        \\ = & \dWkern^{s_{\ell,\lambda,p}}(\delta_{\id_{X\times Y}},N),
    \end{align*}
    which is the desired bound.
\end{proof}

\subsection{Proofs from Section \ref{sec:topological}}
\begin{proof}[Lemma~\ref{lem:connected-gluing-lemma}]
    Let
    \[R := \set{(h_1,h_2,h_3)\in H_1{\times} H_2{\times} H_3}{(h_1,h_2) \in R_{1,2}, (h_2,h_3)\in R_{2,3}}.\]
    We have the following diagram.
    \[
    \begin{tikzcd}
    &&R\ar[dl,"r_1",swap] \ar[dr,"r_2"] \\
    &R_{1,2} \ar[dl,"p_1",swap] \ar[dr,"p_2"] && R_{2,3}\ar[dl,"q_1",swap] \ar[dr,"q_2"] \\
    H_1 && H_2 && H_3
    \end{tikzcd}  
    \]
    Every pictured map is a projection map. We need only check that they are inverse connected, meaning that the preimage of any connected set is connected. Let $A\subseteq H_1$ be a connected set. Then $p_1^{-1}(A)$ is connected, and so is $B:=  p_2(p_1^{-1}(A))$. Similarly, $C:=  q_2(q_1^{-1}(B))$ is connected. Noting finally that the preimage of $A$ in $R$ is
    \[r_1^{-1}(p_1^{-1})(A) = A\times B \times C,\]
    we see that it is a product of connected sets and hence connected. The proof for preimages of connected sets in $H_2$ and $H_3$ is similar.
\end{proof}

\begin{proof}[Theorem~\ref{thm:reeb-control}]
    For the first inequality, suppose $\widehat f,\widehat g:Z\to \bbR$ are continuous functions inducing the Reeb graphs $\reeb(P)$ and $\reeb(P')$ respectively. Then
    \[\|\widehat f - \widehat g\|_{\infty} \geq
    \left |\inf_{z\in Z} \widehat f(z) - \inf_{z\in Z} \widehat g(z)\right |
    = \left|\inf_{h\in H} \cR_P(h) - \inf_{h'\in H'} \cR_{P'}(h')\right| = |\cB(P) - \cB(P')|\]
    where the equality comes from the fact that the infimum of a function can be determined from the Reeb graph that it induces.
    
    Now suppose $\dexpcon(P,P')<a$. Then we can pick a correspondence $R\in \con(H,H')$ that witnesses this fact, giving us the following diagram.
    \[
    \begin{tikzcd}
    & R \ar[dl,"p_1",swap] \ar[dr,"p_2"]\\
    H \ar[d,"R_P"] & & H' \ar[d,"R_{P'}"]\\
    \mathbb{R} && \mathbb{R}
    \end{tikzcd}  
    \]
    Recall $\cR_P$ is the risk function associated with problem $P$. We can then induce two Reeb graphs from functions on $R$: one induced by $\cR_P\circ p_1$ and the other induced by $\cR_{P'} \circ p_2$. The former Reeb graph is isomorphic to $\reeb(P)$ since $p_1$ is a continuous surjection with connected fibers, and similarly the latter is isomorphic to $\reeb(P')$. Hence
    \begin{align*}
    d_{\mathrm{U}}(\reeb(P),\reeb(P')) & \leq \sup_{(h,h')\in R} |\cR_P\circ p_1(h,h') - \cR_P\circ p_1(h,h')|
    \\ & =
    \sup_{(h,h')\in R} |\cR_P(h) - \cR_P(h')|\\
    & \leq 
    \inf_{\gamma \in \Pi(\eta,\eta')} \sup_{(h,h')\in R} \int \big| \ell_h(x,y) - \ell_{h'}(x',y')\big | \,\gamma(dx{\times} dy{\times} dx'{\times} dy')<a.
    \end{align*}
    Hence $d_{\mathrm{U}}(\reeb(P),\reeb(P'))<a$ as well.
\end{proof}

\section{Weak Isomorphism with Classification Problems}\label{appendix:weak-iso}
    While the infimum in the definition of the Risk distance makes it difficult to characterize when two problems are distance zero apart in general, the closely related notion of weak isomorphism is simpler to deal with. For instance, we can characterize which problems are weakly isomorphic to classification problems.
\begin{definition}
    Call a problem $P$ \textit{rectangular} if there exists a finite partition $Y_1,\dots,Y_m$ of the response space $Y$ and, for any $h\in H$, a finite partition $X_1,\dots,X_n$ of the input space $X$ such that
    \[\ell_h(x,y) = \sum_{ij}a_{ij}1_{X_i^h}(x) 1_{Y_j}(y)\quad \eta{\text{-a.e.}}.\]
\end{definition}
A problem $P$ is rectangular, for instance, if the loss is a sum of indicators on rectangles. That is, if
\[\ell(y_1,y_2) = \sum_{ij} a_{ij} 1_{Y_i}(y_1) 1_{Y_j}(y_2)\]
for some partition $Y_1,\dots, Y_n \subseteq Y$, then $P$ is rectangular.
\begin{theorem}
    A problem is weakly isomorphic to a classification problem if and only if it is rectangular. 
\end{theorem}

The proof makes use of the notion of a \emph{generated partition}. Let $S$ be a set and let $T_1,\dots,T_n$ be a collection of subsets of $S$. Then the \emph{partition of $S$ generated by $T_1,\dots,T_n$} is the partition of $S$ induced by the equivalence relation ${\sim}$ where $s\sim s'$ means that $s\in T_i \iff s'\in T_i$ for all $1\leq i \leq n$.
Note that since there are finitely many $T_i$, the generated partition has finitely many blocks as well.
\begin{proof}
    Suppose $P$ is rectangular and let $Q:=\{Y_1,\dots,Y_m\}$, $X_1,\dots,X_n$, and $(a_{ij})_{ij}$ be as in the definition of a rectangular problem, and let $\pi_Q:Y\to Q$ be the quotient map. If $m\neq n$, pad the partitions appropriately with copies of the empty set and pad the matrix $(a_{ij})_{ij}$ with zeros to make $m=n$. Define a classification problem
    \[P':=(X,Q,(\id_X\times \pi_Q)_\sharp \eta, \ell', H')\]
    where
    \begin{itemize}
        \item $\ell'(Y_i,Y_j):=a_{ij}$
        \item $H'=\set{\phi(h)}{h\in H}$,  where $\phi(h):X \to Q$ sends $x$ to $Y_i$ if $x\in X^{h}_i$.
    \end{itemize}
    We claim $P'$ is a simulation of $P$ via the maps $\id_X$ and $\pi_Q$. Indeed, for any $h\in H$, we can calculate that $\ell'_{\phi(h)}(x,\pi_Q(y)) = a_{ij}$ if $x\in X^h_i$ and $y\in Y_j$. Hence $\ell'_{\phi(h)}(x,\pi_Q(y))=\ell_h(x,y)$ for any $h\in H$, showing that $P'$ is a simulation of $P$. In particular, the two are weakly isomorphic.

    For the right hand implication, we proceed in two parts. First suppose that $P$ is a simulation of a classification problem $P'$, whose response space is $Y' = \{y_1',\dots,y_m'\}$. Let the simulation be via the maps $f_1:X\to X'$ and $f_2:Y\to Y'$. We will show that $P$ is rectangular. By the definition of simulation, for any $h\in H$, we can select an $h'\in H'$ such that $\ell_h(x,y) = \ell'_{h'}(f_1(x),f_2(y))$ $\eta$-almost everywhere. Since the function $\ell'$ is on a finite set $Y'\times Y'$, we can write it with a finite sum of indicators to get that, $\eta$-almost everywhere,
    \begin{align*}
        \ell_h(x,y)
        = \ell'(h'\circ f_1(x),f_2(y))
        & =
        \sum_{ij}\ell'(y_i',y_j')1_{Y_i}(h'\circ f_1(x))1_{Y_j}(f_2(y))
        \\ & =
        \sum_{ij}\ell'(y_i',y_j')1_{h'\circ f_1^{-1}(Y_i)}(x)1_{f_2^{-1}(Y_j)}(y).
    \end{align*}
    Hence $P$ is rectangular.

    Next we prove that if $P$ is rectangular and simulates $P'$, then $P'$ must be rectangular as well. Let the simulation again be via maps $f_1:X\to X'$ and $f_2:Y\to Y'$, and let $\{X_i\}_{i=1}^m$, $\{Y_i\}_{i=1}^m$ and $(a_{ij})_{ij}$ be as in the definition of a rectangular problem. Consider the family
    \[\cY := \set{f_2[Y_i]}{1\leq i \leq m}\]
    of subsets of $Y'$. The sets in $\cY$ do not necessarily form a partition of $Y'$, so consider the partition $Q_{Y'}$ generated by $\cY$. Now let $h'\in H'$ be arbitrary and, using the definition of a simulation, pick $h\in H$ such that the diagram
    \[
    \begin{tikzcd}
        X\times Y \ar[dd,"f_1\times f_2",swap] \ar[dr,"\ell_h"]
        \\
        & \bbR
        \\
        X'\times Y' \ar[ur,"\ell'_{h'}",swap]
    \end{tikzcd}    
    \]
    commutes.
    Consider the family
    \[\cX^{h'} := \set{f_1[X^h_i]}{1\leq i \leq m}.\]
    Again, these sets do not necessarily partition $X'$, so let $Q_{X'}^{h'}$ be the partition of $X'$ generated by $\cX^{h'}$.

    We claim that, if $A\in Q_{X'}^{h'}$ and $B\in Q_{Y'}$, then $\ell'_{h'}$ is constant $\eta'$-almost everywhere on $A\times B$. Indeed, by construction of $Q_{X'}^{h'}$ and $Q_{Y'}$, we have that $f_1[X^h_{i}] \supseteq A$ and $f_2[Y_j]\supseteq B$ for some $i$ and $j$. Since $\ell_h$ is constant on $X^h_i \times Y_j$ $\eta$-almost everywhere, by commutativity of the above diagram we also have that $\ell'_{h'}$ is constant $\eta'$-almost everywhere on $(f_1\times f_2)[X^h_i\times Y_j] \supseteq A\times B$. Therefore, $\ell'_{h'}$ can be written as a linear combination of indicators on such rectangles:
    \[\ell'_{h'}(x',y')=\sum_{ij}a_{ij} 1_{A_i}(x) 1_{B_j}(y)\]
    for some $(a_{ij})_{ij}$. Here we have numbered the blocks in the partitions $Q_{X'}^{h'} = \{A_1,\dots, A_m\}$ and $Q_{Y'} = \{B_1,\dots, B_n\}$. Hence $P'$ is rectangular.

    Now, if $P$ is weakly isomorphic to a classification problem $P'$ via a common simulation $P''$, we have shown that $P''$, as a simulation of a classification problem $P'$, must be rectangular. Furthermore, since $P''$ is rectangular and simulates $P$, we have shown that $P$ must be rectangular as well.
\end{proof}

\bibliography{biblio.bib}

\end{document}